%% file: thesis.tex
\documentclass[a4paper,twoside,headsepline,bibliography=totoc,listof=totoc]{scrbook}
\pdfoutput=1

\usepackage[utf8]{inputenc}
\usepackage[T1]{fontenc}

\usepackage[a4paper,left=3.5cm,right=2.5cm,bottom=3.5cm,top=3cm]{geometry}
\usepackage[english]{babel}
\usepackage[titles]{tocloft}
\renewcommand{\cftchappagefont} 
             {\usefont{T1}{bch}{b}{n}\selectfont}

\usepackage{titlesec}
\titleformat{\chapter}[display]
{\vspace{-2em}\Large\sffamily}
{\filleft\MakeUppercase{\chaptertitlename} \Huge\thechapter}
{20pt}
{\titlerule \vspace{1ex} \filright \huge}
[\vspace{0.5ex} \titlerule \vspace{-10pt}]

\usepackage{amsmath}

\usepackage{url}
\usepackage{hyperref}
\usepackage{acronym}
\usepackage[noabbrev,nameinlink]{cleveref}

\usepackage{pgf,tikz}
\usetikzlibrary{arrows,automata,calc,shapes,positioning,shadows,decorations.pathmorphing,fadings,backgrounds}
\usepackage{pdfpages}

\usepackage{helvet}

\usepackage[charter,cal=cmcal]{mathdesign}
\usepackage{XCharter}

\usepackage{pifont}
\usepackage{marvosym}

\usepackage{lscape}
\usepackage[font=small,labelfont=bf]{caption}
\usepackage{subcaption}
\usepackage{enumerate}
\usepackage{paralist}

\usepackage{floatrow}
\usepackage{algorithm,algpseudocode}

\usepackage{amsthm}
\usepackage{thmtools}

\newtheoremstyle{default}
{}
{}
{\itshape}
{}
{\bfseries}
{}
{ }
{\thmname{#1}\thmnumber{ #2}\thmnote{ (#3)}}
\theoremstyle{default}

\newtheorem{theorem}{Theorem}
\newtheorem{definition}{Definition}
\newtheorem{lemma}{Lemma}

\theoremstyle{remark}

\usepackage[backend=bibtex,style=alphabetic,backref=true]{biblatex}
\usepackage[babel]{csquotes}
\usepackage[titletoc]{appendix}
\usepackage{multirow}
\usepackage{xspace}
\usepackage{calc}
\usepackage[babel]{csquotes}
\usepackage{csvsimple}
\usepackage{booktabs}
\usepackage{longtable}
\usepackage{tabu}
\usepackage{rotating}
\usepackage[colorinlistoftodos,prependcaption,textsize=tiny]{todonotes}
\presetkeys{todonotes}{inline}{}

\hypersetup{
	hidelinks,
	pdfauthor={Markus Theo Frohme},
	pdftitle={Active Automata Learning with Adaptive Distinguishing Sequences}
}

\makeatletter
\DeclareCaptionFormat{customBoxed}{%
	\colorbox[cmyk]{0.43, 0.35, 0.35, 0.01}{%
		\parbox{\linewidth - 2\fboxsep}{%
			\@hangfrom{\hspace{5pt}#1#2}{#3\par}%
		}%
	}%
}
\makeatother

\DeclareFloatVCode{bottomRule}{\vspace{1ex}\hrule}
\DeclareFloatStyle{customFrame}{
	postcode=bottomRule,
	captionskip=5pt,
	capposition=top,
	heightadjust=all,
}
\newcommand{\origttfamily}{}
\let\origttfamily=\ttfamily
\renewcommand{\ttfamily}{\origttfamily \hyphenchar\font=`\-}

\newcommand{\coffeeerror}{\text{\ding{88}}}
\newcommand{\coffeecup}{\text{\Coffeecup}}
\newcommand{\coffeeok}{\text{\ding{51}}}

\makeatletter
\@addtoreset{algorithm}{chapter}
\makeatother
\newcommand*\Let[2]{\State #1 $\gets$ #2}
\algnewcommand\Not{\textbf{not}\xspace}
\algnewcommand\Nil{\textbf{nil}\xspace}
\algnewcommand\True{\textbf{true}\xspace}
\algnewcommand\False{\textbf{false}\xspace}

\algdef{SE}[TRY]{Try}{EndTry}[0]{\textbf{try}}{\algorithmicend\ \textbf{try}}%
\algdef{SE}[CATCH]{Catch}{EndCatch}[1]{\textbf{catch} #1}{\algorithmicend\ \textbf{catch}}%

\tikzset{%
	symbol/.style={state, ellipse},
	reset/.style={draw, regular polygon, regular polygon sides=8, xscale=2, minimum size=0.9cm, label=center:reset},
	final/.style={draw,rectangle, inner sep=8pt}
}

\begin{document}
	\include{sections/titlepage}
	\cleardoublepage
	\pagenumbering{roman}
	\include{sections/changelog}

	\cleardoublepage
	\tableofcontents
	\cleardoublepage
	\pagenumbering{arabic}

	\include{sections/intro}

	\include{sections/preliminaries}

	\include{sections/learning}
	\include{sections/replacements}
	\include{sections/heuristics}

	\include{sections/ads}

	\include{sections/evaluation}

	\include{sections/future}

	\listoffigures


	\listofalgorithms
	\cleardoublepage

	\printbibliography 
	\cleardoublepage

	\include{sections/appendix}

\end{document}

%% file: sections/titlepage.tex
\begin{titlepage}

\vspace*{6\baselineskip}

{\parindent 0pt

\hrule height 1pt
\vspace{2ex}
{\Huge
Active Automata Learning with Adaptive Distinguishing Sequences\par
}
\vspace{2ex}
\hrule height 1pt

\vspace{2\baselineskip}

\textbf{Markus Theo Frohme}

\vspace{8\baselineskip}

This document is closely based on the Master thesis of Markus Theo Frohme, submitted at TU Dortmund University, Germany on September 21st, 2015 and may be used as the reference for the \enquote{ADT} algorithm.

\vspace{2\baselineskip}

\textbf{Abstract}

This document investigates the integration of adaptive distinguishing sequences into the process of active automata learning (AAL).
A novel AAL algorithm \enquote{ADT} (\emph{adaptive discrimination tree}) is developed and presented.
Since the submission of the original thesis, the presented algorithm has been integrated into LearnLib \cite{DBLP:conf/cav/IsbernerHS15} -- an open-source library for active automata learning -- and has been successfully used in related fields of research \cite{10.1007/978-3-319-77935-5_24}.

}

\end{titlepage}

%% file: sections/changelog.tex
\chapter*{Changelog, 2019-01-23}
\label{cha:changelog}

This chapter lists all changes (in the order in which they appear) between this version and the version submitted at TU Dortmund University, Germany on September 21st, 2015.

\begin{itemize}

\item Replaced the title page of the thesis with a (retrospective) introduction.

\item Added this \enquote{Changelog} chapter.

\item Fixed a small error in the example of \Cref{sec:aalex}.

\item Addressed a small inconsistency in the example of \Cref{sec:imreplex}.

\item Added addtional/updated references in the \enquote{Future Work} chapter.

\item Removed the \enquote{Statement in Lieu of an Oath} chapter.

\end{itemize}

%% file: sections/intro.tex
\chapter{Introduction}
\label{cha:intro}

The ever-growing complexity of today's soft- and hardware makes testing both an indispensable necessity and a challenging task.
At the given scale, manual testing is unfeasible, which raises the urge for automated approaches.
At the same time, the ongoing digitalization of security- and safety-centric applications requires exhaustive verification of key properties.
A field of research that tackles these problems and has yielded sophisticating results is that of model-based testing \cite{DBLP:conf/dagstuhl/2004test} and model checking \cite{Baier:2008:PMC:1373322}.

Formal verification methods, depending on the scenario, allow the automated generation of tests or the automated evaluation of test properties.
Being based on formal models, a successful verification is also able to provably guarantee certain properties of the system under testing.
Key to a successful application of these techniques is a formal specification of the target system.
This requirement, however, poses a problem to many real-world applications:
The lack of formal specifications for software or hardware hinders the employment of formal verification methods.

Creating formal specifications for soft- or hardware components is not only a tedious task but also prone to errors.
Not precisely specifying critical system behavior renders any formal verification methods redundant.
This problem is imminent in situations where e.g. third-party components, whose internal structure is often unknown, are integrated.
The question arises: How can one automatically extract a representative formal model from an unknown soft- or hardware component?

A potential answer to this question is given by the field of active automata learning.
Active automata learning describes the process of inferring a formal abstraction of an unknown black-box system based on its observable input/output behavior.
By actively interacting with the system, the learning algorithm (or simply learner) explores the structure of the system and ultimately yields an automaton -- a formal specification that is commonly used among formal verification methods -- that is behaviorally equivalent to the (abstracted) target system.

Initially, the effort that eventually led to the active learning paradigm was mainly motivated by a theoretical point of view.
As a consequence, certain specific characteristics have to be considered when employing active automata learning in real-world applications (see below).
Nonetheless, there exists several instances of successful applications in real-world environments \cite{Peled:2001:BBC:767345.767349, DBLP:conf/itc/HungarMS03, Raffelt:2008:HTW:1390832.1390833, conf/iceccs/IssarnySJBGKCITBS09,caseStudiesLBT}. 
The examples further show, that active automata learning is not only limited to the use case of formal verification methods.
The general possibility to extract formal specifications from black-box systems on a behavioral level is useful in numerous ways.
Improving tool support \cite{DBLP:conf/cav/VardhanV06,libalf,DBLP:conf/cav/IsbernerHS15} and the growing number of competitions \cite{zulu, DBLP:conf/isola/HowarSM10, stamina} focusing on the practical applicability of active learning and encouraging learning based verification techniques show the increasing interest in this field of research.

The initial active learning framework and learning algorithm $L^*$ were proposed by Angluin in \cite{Angluin:1987}.
Since then, not only the practical applicability of active automata learning has matured, but also algorithmic aspects have been subject to many extensions and improvements.
Yet, the core concepts of the initial approach can still be found in many of today's learning variants.

Conceptually, active automata learning is an iterative approach.
By continuously exploring the target system and verifying assumptions about it, the learner constructs evolving hypotheses about the target system that eventually converge against its true behavior.
The \enquote{protocol} of how a learning algorithm can interact with a system is formalized by two types of oracles: membership and equivalence oracles.
Abstracting from the concrete target system, the oracles allow two basic kinds of interaction with the system to learn:

\begin{description}

\item[Membership Queries] (MQs) form the basic instrument of communication between the learning algorithm and the system under learning (commonly abbreviated as SUL).
The learner can pose a membership query to the membership oracle containing a sequence of input stimuli.
The oracle applies these stimuli to the SUL and answers the membership query by returning the observed behavior.
The term \textit{membership} query originates from the fact that Angluins framework was initially designed for learning systems that resemble deterministic finite automata, where a membership query would answer if a word is a \textit{member} of the language induced by the SUL.

\item[Equivalence Queries] (EQs) are answered by the equivalence oracle and used to verify the assumptions of the learner about the SUL.
In most cases they are assumed to be \textit{strong} equivalence queries, meaning they not only indicate if the assumptions of the learner are correct or not but also provide an active counterexample in case of a failed verification.
The learner may then use the information of the counterexample to update its assumptions accordingly.

\end{description}

These concepts of communication are strongly connected to the procedure by which the learning algorithm explores the target system.
\Cref{fig:aal} depicts a sketch of the internal structure of this process, that many algorithms inherit.

\begin{figure}[h]
	\centering
	\input{figures/aal.tex}
	\caption{A sketch of the internal structure of active learning algorithms}
	\label{fig:aal}
\end{figure}
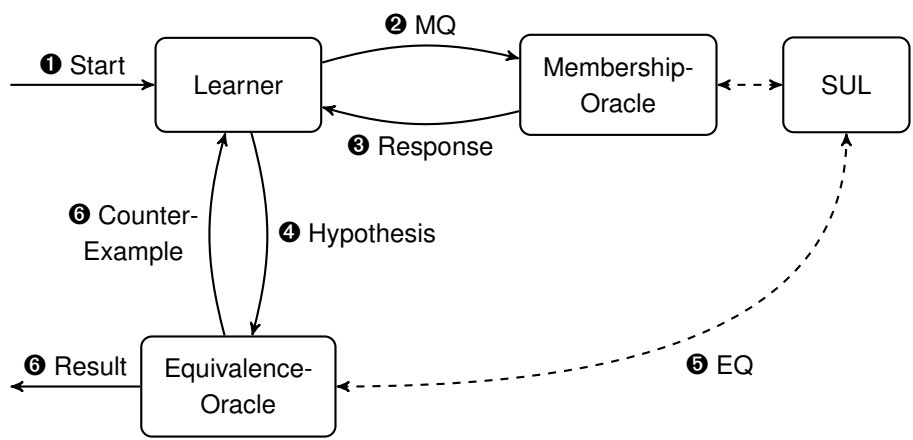

The learning process starts with the learner posing possibly multiple membership queries to the membership oracle.
By processing the answers of the membership oracle -- and therefore the answers of the SUL -- the learner constructs a hypothesis that, based on the observed behavior, resembles the target system.
Once the learner is certain, that its internal hypothesis is behaviorally equivalent to the SUL, it proposes its hypothesis to an equivalence oracle to check for true equivalence.
Depending on the result, further execution is determined.
If the equivalence query indicates true equivalence, the current hypothesis may be returned as the final result, meaning the learning process has successfully finished.
If, however, the equivalence query yields a counterexample, it is returned to the learner, which may use the counterexample and additional membership queries to refine its hypothesis and repeat the above procedure.

As stated earlier, the concepts of active automata learning were initially developed in a theoretical environment.
Transitioning these concepts to real applications hold additional challenges, as indicated by the dashed connections.

On the one hand, it is the membership oracle that connects the learner with the target system, as it translates the membership queries to concrete stimuli of the SUL.
This poses a problem, if the complexity of the learner's hypothesis and the complexity of the SUL differ.
If the hypothesis of the learner is not able to represent certain characteristics of the SUL, information will be lost when either applying the input stimuli or observing the reaction of the SUL.
Ultimately, the final model may not be able to capture the critical behavior.

Additionally, in its initial formalization, membership queries were assumed to be independent of each other.
In reality, this is often realized by \emph{resetting} the target system to a dedicated initial state.
Each membership query is then preceded by a reset, to ensure their independence.

On the other hand, it is in general not possible to construct true equivalence queries, because the equivalence problem for black-box systems is undecidable \cite{Moore56}. 
As a result, true correctness of the final model cannot be guaranteed for arbitrary black-box systems.
In practice, equivalence queries are usually approximated by a multitude of regular membership queries.
These might e.g. be randomly generated sequences of input stimuli or follow more sophisticated approaches, such as conformance testing \cite{DBLP:conf/dagstuhl/Gargantini04}.

Although these challenges seem critical at first, the mentioned use cases show that active automata learning is still a beneficial concept to domains such as model-based testing.

\section{Statement of the Problem}
\label{sec:problem}

Besides the inherent issues of applying the concepts of active automata learning to real-world systems, the approach faces another pragmatic challenge: runtime.
While being able to use a formal specification may open up possibilities for an improved workflow, the actual process of extracting such a specification introduces an additional unit of work.
In the case of active automata learning, the effective runtime depends on a variety of parameters.

The first criteria that may come to mind when analyzing the runtime of the learning process is the complexity (or efficiency) of the learning algorithm itself.
However, a fundamental part of active automata learning is the execution of membership and (approximated) equivalence queries on the target system.
Answering these queries slowly, will also affect the runtime of the learning process as a whole.
Reports \cite{Cho:2010:IAF:1866307.1866355,teachersCrowd,Choi:2013:GGT:2509136.2509552} show that under real-life conditions, it is in fact often the performance and complexity of the SUL that dominates the overall runtime.

While a learning algorithm generally has no control over these properties of the SUL, it has control over the content and the amount of membership queries it poses.
As a result, the active learning community often analyzes learning algorithms with regard to the number and the length of posed membership and equivalence queries.
This allows to abstract from technical details and focus on the query complexity of an algorithm rather than the performance of its implementation or the SUL.
While in general the rule \textit{the fewer, the better} is reasonable, this thesis takes a closer look at the performance of the learning process with special focus on resets.

As stated earlier, membership queries are assumed to be independent from each other, which is often realized by preceding each membership query with a special reset stimuli.
There are however application domains (cf. \cite{Choi:2013:GGT:2509136.2509552}) where resets form an expensive operation.
Especially software for embedded devices or smartphones is usually developed and tested in simulated environments which are easier to manage than the actual hardware device.
A straight-forward implementation of a reset could therefore be realized by restarting the simulator.
However, given today's complexity of hardware platforms, the simulator may take a significant amount of time until it has restarted and is able to process input stimuli again.
Regarding the total runtime, it would be beneficial to reduce the amount of resets even at the expense of possibly longer membership queries.

There has been research \cite{rivest1993inference,Freund:1993} on learning algorithms that follow a complete no-reset approach.
While originally motivated by the problem of a non-reliable or absent reset, such approaches may also pose improvements to the scenario described above.
On the contrary, to work correctly, such algorithms often require the target system to be strongly connected, which may drastically reduce their applicability to many real-world systems.

\section{Scope of This Thesis}

To tackle the aforementioned problem, this thesis elaborates an approach that aims at reducing the total amount of executed resets during the learning process and therefore improving the performance of active automata learning for applications with expensive resets.
This approach -- and hence the scope of this thesis -- will however be limited to reactive target systems, for which inferring a regular abstraction is possible.
Furthermore it is assumed that the target system has a reliable reset mechanism.

The core idea of the developed approach is to combine the concept of \emph{adaptive distinguishing sequences} -- a well-studied method from the field of model-based testing -- with active automata learning.
Both approaches face the same challenge: the problem of state identification.
In the case of active automata learning, the solution to the state identification problem is mainly driven by the iterative separation of (generally unknown) system states using multiple membership queries and therefore multiple resets.
Adaptive distinguishing sequences aim at identifying states with a single (adaptive) input sequence, however, requiring that all states are known beforehand.
Using the knowledge acquired throughout the learning process may allow to utilize their benefits in the active learning environment:
Adaptive distinguishing sequences do not require resets, which makes them a more favorable solution for the given state identification problem.
However, this beneficial property comes at the cost of their potential non-existence.

To successfully combine the two approaches, this thesis will present an active learning algorithm that allows the integration of the concepts of adaptive distinguishing sequences.
It will discuss the impact of using adaptive distinguishing sequences in the active learning process and propose a set of heuristics that aim at reducing the number of executed resets.
Furthermore an analysis of the developed heuristics and a comparison with competing state-of-the-art active learning algorithms will be presented.

\section{Outline}
\label{sec:outlook}

In detail, the chapters of this thesis cover the following topics:

\begin{description}

\item[\Cref{cha:preliminaries}] starts with the introduction of the basic notation and assumptions used throughout this thesis.
Based on that, it presents related work and their results, which are used as a foundation for the developed concepts of this thesis.

\item[\Cref{cha:aal}] continues to present the initial base algorithm.
While not yet using any concepts of adaptive distinguishing sequences specifically, it provides an environment that will allow their seamless integration.

\item[\Cref{cha:replacements}] presents the approach by which adaptive distinguishing sequence are integrated into the active learning process and discusses the effects of doing so.

\item[\Cref{cha:heuristics}] presents a set of heuristics that actively employ adaptive distinguishing sequences to reduce the amount of resets and therefore potentially improve the performance of the active learning process.

\item[\Cref{cha:ads}] briefly discusses the elaborated approaches to compute different kinds of adaptive distinguishing sequences used throughout the heuristics.

\item[\Cref{cha:eval}] inspects the proposed algorithm and techniques with regard to its theoretical complexity and empirical performance.
While synthetic benchmarks allow to expose certain characteristics, a set of real-life examples is used to show its applicability to real-life scenarios.

\item[\Cref{cha:future}] concludes the thesis with a final resume and an outlook on possible further research.

\end{description}

%% file: figures/aal.tex
\begin{tikzpicture}[thick,->,>=stealth',font=\sffamily]

	\node[draw,rectangle,rounded corners,inner sep=.5cm] (learner) at(0,0) {Learner};
	\node[draw,rectangle,rounded corners,inner sep=.3cm,align=center] (mq) at(5,0) {Membership-\\Oracle};
	\node[draw,rectangle,rounded corners,inner sep=.5cm] (sul) at(8,0) {SUL};
	\node[draw,rectangle,rounded corners,inner sep=.3cm,align=center] (eq) at(0,-4) {Equivalence-\\Oracle};

	\draw (learner) edge[bend left=15] node[anchor=south] {\ding{203} MQ} (mq); 
	\draw (mq) edge[bend left=15] node[anchor=north] {\ding{204} Response} (learner); 

	\draw[<->,dashed] (mq) -- (sul); 

	\draw (learner) edge[bend left=15] node[anchor=west] {\ding{205} Hypothesis} (eq); 
	\draw (eq) edge[bend left=15] node[anchor=east,align=center] {\ding{207} Counter-\\Example} (learner); 

	\draw (eq) -- node[anchor=south] {\ding{207} Result} ++(-3,0); 
	\draw ++(-3,0) -- node[anchor=south] {\ding{202} Start} (learner); 

	\draw[<->,dashed] (eq.0) to[out=0, in=270] node[anchor=north west] {\ding{206} EQ} (sul.270);
\end{tikzpicture}

%% file: sections/preliminaries.tex
\chapter{Preliminaries and Related Work}
\label{cha:preliminaries}

This chapter gives a preliminary overview of the key concepts that allow combining active automata learning and adaptive distinguishing sequences for learning reactive systems.
It introduces the basic notation used throughout this thesis and presents ideas and results of related fields of research.
However, most of the discussed concepts will only be sketched, as an in-depth analysis would exceed the scope of this chapter and is already covered in the referenced literature.

\section{Running Example}
\label{sec:runex}

For a better understanding, this and the following chapters will utilize a running example.
For explaining the concepts of active automata learning and adaptive distinguishing sequences and -- later -- visualizing the execution of the developed algorithms, the exemplary target system to learn will be represented by a real-life application: a coffee machine \cite{practical2011}.
The behavior of this coffee machine can be described as follows:

\begin{itemize}

\item The coffee machine has three components, which offer a direct way of interaction:

\begin{itemize}
\item $water$ describes the action of filling the water tank of the coffee machine with water,
\item $pod$ describes the action of putting a coffee pod the intended compartment and
\item $button$ describes the action of starting the coffee machine.
\end{itemize}

Additionally, the user may $clean$ the coffee machine by clearing the coffee pod and emptying the water tank.

\item Repeatedly filling the water tank, (re-)placing the coffee pod or cleaning the coffee machine, has no (observable) effect.

\item If the coffee machine is turned on, two possible situations may occur:

\begin{itemize}
\item If the water tank was filled with water and a coffee pod was added, the coffee machine will produce coffee.
\item If, however, any of the two requirements were not met, the coffee machine will break irreparably.
\end{itemize}

\item After successfully brewing a cup of coffee, it is necessary to clean the machine.
Any other interaction will break the machine again.

\end{itemize}

\section{Formal Definitions}
\label{sec:formaldef}

In order to formalize algorithms and prove certain properties, a formal way for representing the behavior of the target system is of special interest.
A formal model that has successfully been used to especially represent the behavior of reactive systems, is that of Mealy machines \cite{mealy1955method}.

\begin{definition}[Mealy machines]\label{def:mealy}
A Mealy machine $\mathcal{M}$ is a tuple $\langle S, s_0, I, O, \delta, \lambda \rangle$, where 
\begin{compactitem}
\item $S$ denotes a non-empty set of \emph{states},
\item $s_0 \in S$ denotes the \emph{initial state},
\item $I$ denotes a finite set of \emph{input symbols},
\item $O$ denotes a finite set of \emph{output symbols},
\item $\delta : S \times I \rightarrow S$ denotes a \emph{state transition function} and
\item $\lambda : S \times I \rightarrow O$ denotes an \emph{output function}.
\end{compactitem}
In certain situations, a component may be annotated with a subscript (e.g. $\delta_{\mathcal{M}}$) to refer to a component of a certain Mealy machine $\mathcal{M}$.
If the Mealy machine, which is referred to, is clear from the context, the annotation is omitted.
\end{definition}

Semantically, a Mealy machine starts in its initial state $s_0$.
Upon receiving an input symbol $i \in I$ it transitions from its current state $s_i \in S$ into a successor state $s_j \in S$ as defined by its state transition function $\delta(s_i, i) = s_j$, while emitting an output symbol $o \in O$ as defined by its output function $\lambda(s_i, i) = o$.
The behavior of such reactive systems is then defined by the sequence of observed output symbols after applying a sequence of input symbols.

Syntactically, a Mealy machine may be displayed in text form by enumerating the elements of the respective sets and providing $\delta$ and $\lambda$ as look-up tables.
However, a more convenient notation is the representation as a graph, where the states of a Mealy machine model the nodes of the graph and the transition and output function are represented by labeled edges.
The input and output alphabet of the displayed Mealy machine is then given implicitly by the union of all input/output symbols over all edges.
An example for this kind of visualization is given in \Cref{fig:coffeemachine}, which shows a potential Mealy machine abstraction of the coffee machine of \Cref{sec:runex}.

\begin{figure}[ht]
	\centering
	\resizebox{0.9\textwidth}{!}{
		\input{figures/coffeemachine.tex}
	}
	\caption{A (potential) Mealy abstraction of the coffee machine example}%
	\label{fig:coffeemachine}
\end{figure}
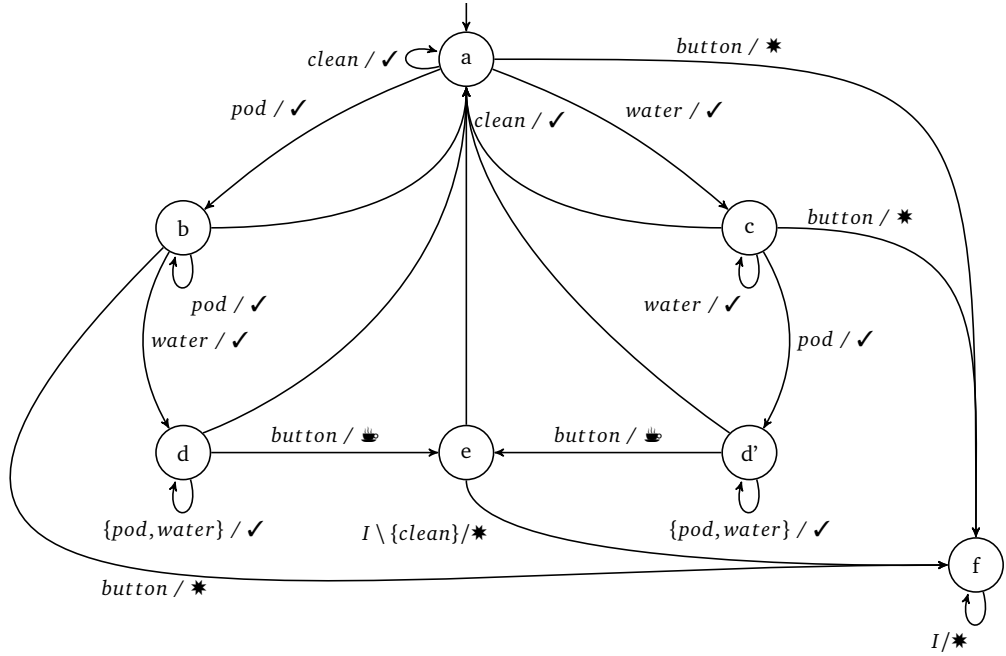

\Cref{def:mealy} gives a very general definition of Mealy machines, as it does not impose any constraints on the different components of a Mealy machine.
However, certain assumptions and restrictions prove useful to make the definition of algorithms easier and are in some cases even required for proofs of correctness and termination.
Therefore throughout this thesis, unless specified otherwise, Mealy machines -- and the abstractions of reactive systems they represent -- are assumed to have the following properties:

\begin{definition}[Finiteness]
A Mealy machine $\mathcal{M}$ is called finite iff its set of states is finite, i.e. $\vert S \vert < \infty$.
\end{definition}

Finiteness is property that restricts the complexity of a Mealy machine.
In formal language theory (cf. \cite{Hopcroft:2006:IAT:1196416}), the finiteness of a deterministic finite automaton ensures that the corresponding language it accepts only has a finite number of equivalence classes, which constitutes its regularity \cite{nerode58}.
The idea of regularity can be extended to the case of Mealy machines \cite{practical2011}.

Regularity does not directly impose restrictions on the SUL itself. 
It does, however, limit the complexity of the inferable abstraction that is intended to represent the behavior of the SUL.
For many use cases it is possible to find reasonable regular abstractions that cover the key behavioral aspects of interest.
However, there exist applications where the requirement of regularity may make capturing essential behavioral aspects harder, or worse, prevent it.
There has been research \cite{registerautomata,rmm} on extending the concepts of active automata learning to more complex models to tackle these issues.
The approaches presented in this thesis will however be limited to inferring a regular (Mealy) abstraction of the SUL.

Regarding the learning of such abstractions, most learning algorithms explore the states (or more precise, the equivalence classes) of the abstracted target system.
Limiting this property to a finite domain allows to prove the termination of the learning algorithm after a finite amount of time.

\begin{definition}[Determinism]
A Mealy machine $\mathcal{M}$ is called transition- (output-) deterministic iff its transition (output) function $\delta$ ($\lambda$) maps each input tuple $(s, i) \in S \times I$ to at most one successor state (output).
A Mealy machine $\mathcal{M}$ is called deterministic iff it is transition- and output-deterministic.
\end{definition}

Technically, \Cref{def:mealy} already enforces determinism, because in non-deterministic Mealy machines $\delta$ and $\lambda$ can map into the powerset $2^S$ and $2^O$ respectively.
While there exist non-deterministic Mealy machines for which no behaviorally equivalent deterministic Mealy machine can be found, determinism is often a question of abstraction when integrating the SUL into a learning environment (i.e. creating an interface for membership queries).
There exists research for learning non-deterministic Mealy machines \cite{DBLP:conf/icgi/KhaliliT14} and computing adaptive distinguishing sequences for non-deterministic Mealy machines \cite{kushik13}, but for the scope of this thesis, determinism is assumed.

\begin{definition}[Completeness]
A Mealy machine $\mathcal{M}$ is called transition- (output-) complete iff its transition (output) function $\delta$ ($\lambda$) is total, i.e. defined for every possible combination of state and input symbol.
A Mealy machine $\mathcal{M}$ is called complete iff it is transition- and output-complete.
\end{definition}

Completeness is a property that is useful for the definition of algorithms as there is no need for special treatment of undefined behavior.
It is a property that can be artificially added to the SUL when creating the Mealy abstraction, as can be seen in the coffee machine example from \Cref{sec:runex}: 
Once the coffee machine is in an erroneous state, it may not even react to certain input signals and come to a complete halt.
The Mealy abstraction however may still be able to process further input signals, as it has access to the real coffee machine and may return the \coffeeerror-symbol if the erroneous state is detected.

Within the abstraction, completeness can generally be realized by adding an additional sink-state that loops every input symbol to itself. 
Every undefined transition in the original system maps into the designated sink-state by optionally emitting a special \textit{error} or \textit{undefined} symbol.
Hence, the completeness requirement does not affect possible target systems.

\begin{definition}[Minimality]
A Mealy machine $\mathcal{M}$ is called minimal iff there exists no equivalent Mealy machine $\mathcal{M}'$ with fewer states than $\mathcal{M}$.
Two (deterministic) Mealy machines are called equivalent iff they produce the same output sequence for every possible input sequence.
\end{definition}

The concept of minimality is rather a theoretical motivation than a restriction on the target system, as can be seen by the coffee machine example from \Cref{fig:coffeemachine} again.
The depicted Mealy machine is not minimal, because the states $d$ and $d'$ are equivalent.
While this structure may correspond to the true implementation of the target system, the two states cannot be distinguished by any observable behavior.
As a consequence, a learning algorithm will only be able to distinguish six distinct states, opposed to the original seven.
The inferred model will still be equivalent (as defined above), though not necessarily isomorphic to the target abstraction.

However, the minimality of the target abstraction is a necessary property to prove the exact (up to isomorphism) model inference of learning algorithms.
Furthermore, is it a requirement for the application and computation of adaptive distinguishing sequences (cf. \Cref{sec:ads}).
In situations where necessary, the abstraction of the target system can be minimized in polynomial time \cite{Hopcroft:1971:NLN:891883}.

Besides semantic properties, Mealy machines may also be extended syntactically, mainly by overloading the transition- and output functions, which allows for more convenient notations in certain situations.
The main concept of this extension is to define the behavior on input-\emph{words}, the concatenation of multiple input symbols.
The established syntax is, however, also applicable to output-words.

\begin{definition}[Words over an alphabet]
Let $\Sigma$ be an input (or output) alphabet of a Mealy machine.
A word $w$ of length $n \in \mathbb{N}_0$ is defined as the concatenation of $n$ input (or output) symbols $i_j \in \Sigma, 1 \leq j \leq n$.
For concatenation, the following syntax is used:
\begin{align*}
w &= w_1 \cdot ... \cdot w_n = w_1 ... w_n & \forall w \in \Sigma^n, n \in \mathbb{N}_0
\end{align*}
The special case $n = 0$ denotes the empty word $\varepsilon$.
A similar syntax is used to denote the concatenation of words:
\begin{align*}
uv &= u \cdot v = u_1 \cdot ... \cdot u_n \cdot v_1 \cdot ... \cdot v_m & \forall u \in \Sigma^n, v \in \Sigma^m 
\end{align*}

In certain situations it is useful to extract certain subsequences of words, for which the following syntax is used:
Let $w \in \Sigma^n$, then 
\begin{align*}
w_{i:j} &= w_i \cdot ... \cdot w_j & \forall i \leq j \in \{1, ..., n\} 
\end{align*}
denotes the syntax for the sub-word starting at index $i$ and ending at index $j$.
Note that $w_{i:j} \in \Sigma^{j-i+1}$.
For $i > j$, $w_{i:j}$ denotes the empty word $\varepsilon$.

For retrieving the length of a word, the following syntax is used:
\begin{align*}
\vert w \vert &= n & \forall w \in \Sigma^n, n \in \mathbb{N}_0
\end{align*}
\end{definition}

With the given syntax, the transition- and output function may then be extended to operate on words.

\begin{definition}[Extension of transition- and output functions of Mealy machines]
Let $\mathcal{M} = \langle S, s_0, I, O, \delta, \lambda \rangle$ denote a Mealy machine and $s \in S, w \in I^*$ a state and an input word of arbitrary length.
The extension of transition- and output-function $\delta$ and $\lambda$ to the domain $S \times I^*$ is defined as follows:
\begin{align*}
\delta(s, w) &= 
	\begin{cases} 
		\delta(\delta(s, w_1), w_{2:\vert w \vert}) & \textrm{ if } \vert w \vert > 0\\
		s & \textrm{ if } \vert w \vert = 0
	\end{cases}\\
\lambda(s, w) &=
	\begin{cases}
		\lambda(s, w_1) \cdot \lambda(\delta(s, w_1), w_{2:\vert w \vert}) & \textrm{ if } \vert w \vert > 0\\
		\varepsilon & \textrm{ if } \vert w \vert = 0
	\end{cases}
\end{align*}

For simply tracing an input sequence $w \in I^*$, the following syntax is used:
\begin{align*}
\delta(w) = \delta(s_0, w)\\
\lambda(w) = \lambda(s_0, w)
\end{align*}

\end{definition}

Throughout the learning process, a learning algorithm typically interacts with the SUL by posing membership queries.
These queries are answered by the membership oracle, which provides an abstracted and the only interface to the application.
Therefore, if in the following the term \emph{target system} or \emph{SUL} is used, it is usually referred to the (regular) Mealy abstraction of the actual application.

For representing the acquired knowledge about the behavior of the target system, many learning algorithms use temporary Mealy machines.
To emphasize the distinction between temporary models and the target system, interaction with the real system (via membership queries) will be formalized by the membership query function.

\begin{definition}[Membership Query Function]
Let $I$ denote an input alphabet and $O$ an output alphabet.
For an input word $w \in I^n$ of length $n \in \mathbb{N}_0$, $mq: I^n \rightarrow O^n$ denotes the membership query function that poses the given input word to the target system and returns the observed behavior.
Additionally, for input words $w \in I^n, w' \in I^m$, let $mq: I^n \times I^m \rightarrow O^m$ denote the overloaded function, that allows to specify a prefix input word, whose output is ignored.
Given the traditional membership query function, this can be defined as:
\begin{align*}
mq(u, v) = mq(u \cdot v)_{\vert u \vert + 1: \vert u \cdot v \vert}
\end{align*}
\end{definition}

\section{Active Automata Learning}
\label{sec:aal}

With the above definitions, the goal of active automata learning for the given scenario can be described as \textit{inferring an unknown, finite, deterministic, complete and minimal Mealy machine based on its observable behavior.}
An approach that many active learning algorithms pursue is given by the idea of partition refinements:

The design of Mealy machines suggests that its behavior is inherently defined by means of its output traces, i.e. sequences of input symbols and the emitted sequences of output symbols.
However, it can be shown, that Mealy machines can be completely characterized by a functional $P: I^* \rightarrow O$, that only returns the last observation after applying an input sequence.
With the introduced syntax, the functional can be defined as $P(w) = \lambda(w)_{\vert w \vert}$.

For a functional $P$ of a Mealy machine, a relation $\equiv_P$ on its input arguments can be defined, that transfers the concepts of the Nerode relation \cite{nerode58} of formal language theory to Mealy machines:

\begin{definition}[Equivalence of words with respect to P \cite{practical2011}]
Two words $u, v \in I^*$ are equivalent with respect to $\equiv_P$, iff for all continuations $w \in I^*$ the concatenated words $uw$ and $vw$ are mapped to the same output by $P$:

$$u \equiv_P v \Leftrightarrow \forall w \in I^*: P(uw) = P(vw)$$
\end{definition}

It is easy to see, that $\equiv_P$ resembles an equivalence relation, as the equality of the returned output symbols is reflexive, symmetric and transitive.
For further notation, the equivalence class of an input word $w$ with respect to $\equiv_P$ will be denoted as $[w]_{\equiv_P}$.

As a consequence of employing techniques of the well-studied field of formal language theory, one can also transfer its results (Myhill-Nerode theorem).
By definition, the target system under learning (i.e. its abstraction) is assumed to be finite.
This means, the index of $\equiv_P$ (i.e. the number of its equivalence classes) is finite as well.
Furthermore does the minimality of the target system allow to conclude, that each equivalence class directly corresponds to a distinct state of the target system.
An equivalence class therefore consists of all possible input sequences that lead to the state, the equivalence class represents.

Given full information about the functional $P$ and its induced equivalence relation $\equiv_P$, one can therefore construct an equivalent automaton, that is isomorphic to the target system:

\begin{definition}[Construction of the canonical automaton]\label{def:canonical}
Given a functional $P: I^* \rightarrow O$ and the induced equivalence relation $\equiv_P$, the canonical automaton $\mathcal{M}^c = \langle S^c, s_0^c, I^c, O^c, \delta^c, \lambda^c \rangle$ can be constructed as follows:

\begin{compactitem}
\item $S^c = $ the set of equivalence classes of $\equiv_P$,
\item $s^c_0 = [\varepsilon]_{\equiv_P}$,
\item $I^c = I$,
\item $O^c = O$,
\item $\delta^c([w]_{\equiv_P}, i) = [wi]_{\equiv_P} ~\forall [w]_{\equiv_P} \in S^c, i \in I^c$ and
\item $\lambda^c([w]_{\equiv_P}, i) = P(wi) ~\forall [w]_{\equiv_P} \in S^c, i \in I^c$
\end{compactitem}
\end{definition}

Inferring the canonical automaton therefore meets the requirements of the active learning process.
However, while an active learning algorithm generally has access to the functional $P$ by means of the membership oracle (i.e. $P(w) = mq(w)_{\vert w \vert}$), it lacks information about the equivalence classes of $\equiv_P$.
At this point, the idea of partition refinements materializes.

The learning process starts with the assumption of a single equivalence class and constructs a local hypothesis based on this assumption.
In most cases however, unless the target system in fact consists of a one-state automaton, a single equivalence class is too coarse as it unifies all true equivalence classes of the target system.
By posing an equivalence query, the learner receives information about input sequences for which the output of the local hypothesis and the true target system differ -- a clear indication, that an equivalence class of the local hypothesis is too coarse and needs to be refined.
By alternating membership queries -- to construct local canonical automata -- and equivalence queries -- to possibly refine the local hypothesis -- the active learning algorithm successively refines discovered partitions until eventually convergence against the (true) canonical automaton is achieved.

To further formalize this process and to present the conceptual base algorithm used for the approaches developed in this thesis, the following section will present a variation of the Discrimination Tree learning algorithm.

\subsection{Discrimination Tree Algorithm}
\label{sec:dt}

The Discrimination Tree algorithm is an active learning algorithm proposed by Kearns and Vazirani \cite{Kearns:1994:ICL:200548}.
Key to its design are two core data structures:

\begin{compactitem}

\item a tentative hypothesis that stores information about the discovered equivalence classes and represents the learners assumptions about the structure of the target system and 
\item the discrimination tree, a tree data structure that stores input sequences, that allow to distinguish equivalence classes of the target system.

\end{compactitem}

\noindent
For a better distinction between the local hypothesis and the target system, henceforth the local hypothesis will be labeled with $\mathcal{H}$, whereas the target system will be labeled with $\mathcal{M}$.

As described in the previous section, the algorithm starts with the initial assumption of a single equivalence class.
The situation for the coffee machine example is depicted in \Cref{fig:dt_1}:

\begin{figure}[ht]
	\centering
	\input{figures/dt_1.tex}
	\caption{Initial hypothesis and discrimination tree}%
	\label{fig:dt_1}
\end{figure}
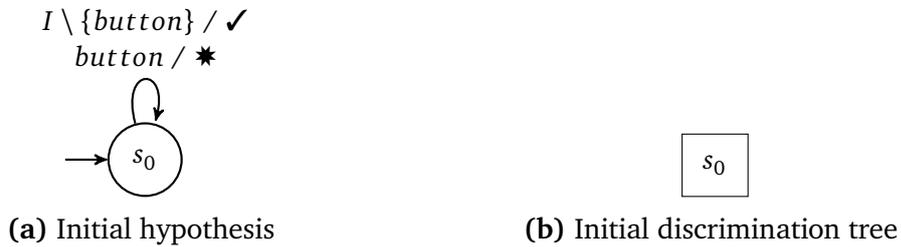

The initial hypothesis consists of a single-state automaton.
Its access sequence ($\varepsilon$) is stored as the representative of the corresponding equivalence class $[\varepsilon]_{\mathcal{M}}$ of the target system.
(Recall, that an equivalence class consists of all access sequences of its represented state).
At the same time, let $[s_0]_{\mathcal{H}}$ denote the access sequence that $s_0$ represents.

The outputs for the transitions are determined according the construction of the canonical automaton.
For example, the output for the $pod$-transition can be determined as follows:

$$\lambda_{\mathcal{H}}(s_0, pod) = P([s_0]_{\mathcal{H}} \cdot pod) = P(\varepsilon \cdot pod) = \lambda_{\mathcal{M}}(\varepsilon \cdot pod)_1 = \coffeeok$$

For determining the successors of a transition, the discrimination tree is consulted.
However, under the initial assumptions, there is no equivalence class other than $[\varepsilon]_{\mathcal{M}}$ represented.
Hence,

$$\delta_{\mathcal{H}}(s_0, pod) = [\varepsilon]_{\mathcal{M}} = s_0$$

After finishing the construction of the remaining hypothesis, the learning algorithm returns its local hypothesis as the, what is assumed, final result. 
Following the active learning process (cf. \Cref{fig:aal}), the hypothesis is then presented to the equivalence oracle, to check for the equivalence with the true target system.
It is easy to see, that the two models are not equivalent yet.
As a result, the equivalence oracle may return the counterexample $\hat c = button \cdot water$ for which the tentative hypothesis outputs $\coffeeerror \cdot \coffeeok$, whereas the target system outputs $\coffeeerror \cdot \coffeeerror$.

Rivest and Shapire have shown \cite{rivest1993inference} (for hypotheses constructed as below), that each counterexample $\hat c$ can be decomposed into a triple $\hat c = uav$ with $\langle u, a, v\rangle \in I^* \times I \times I^+$ such that

$$P([\delta_{\mathcal{H}}(u)]_{\mathcal{H}}av) \neq P([\delta_{\mathcal{H}}(ua)]_{\mathcal{H}}v)$$

This means, the state $\delta_{\mathcal{H}}(ua)$ represents too many access sequences (namely $[\delta_{\mathcal{H}}(u)]_{\mathcal{H}} \cdot a$) because there exists a distinguishing suffix $v$, that proves $[[\delta_{\mathcal{H}}(u)]_{\mathcal{H}}a]_{\mathcal{M}} \neq [[\delta_{\mathcal{H}}(ua)]_{\mathcal{H}}]_{\mathcal{M}}$.
For the given counterexample $\hat c$, such a decomposition is given by $u = \varepsilon, a = button, v = water$.
As a consequence, the current tentative hypothesis needs to be refined.

Effectively, the $a$-successor of the state $\delta_{\mathcal{H}}(u)$ needs to represent the newly discovered equivalence class $[[\delta_{\mathcal{H}}(u)]_{\mathcal{H}}a]_{\mathcal{M}}$.
This can be achieved by adding a new state $n$ to the local hypothesis and defining $\delta_{\mathcal{H}}(\delta_{\mathcal{H}}(u), a) = n$.
Furthermore, the learning algorithm now needs to distinguish between the old (too coarse) equivalence class and the newly discovered one.
To do so, the learner can use the obtained discriminator $v$ to split the leaf of the discrimination tree referencing the hypothesis state that represented the old equivalence class and insert a \enquote{decision point} that distinguishes between the old and new equivalence class.

The updated structures for the given counterexample are shown in \Cref{fig:dt_2}.

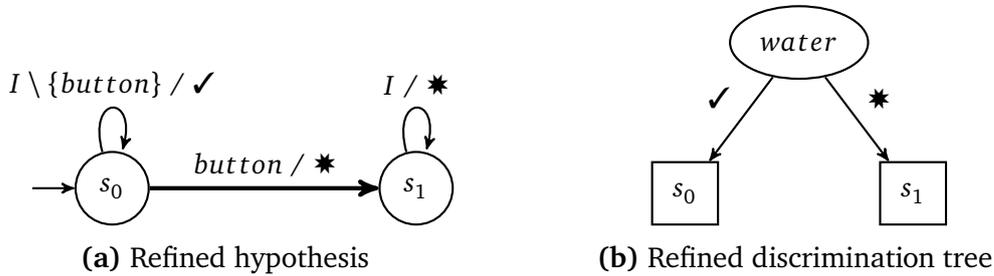
\begin{figure}[ht]
	\centering
	\input{figures/dt_2.tex}
	\caption{Refined hypothesis and discrimination tree after the first counterexample}%
	\label{fig:dt_2}
\end{figure}

A new state -- $s_1$ -- representing $[[\delta_{\mathcal{H}}(u)]_{\mathcal{H}}a]_{\mathcal{M}} = [button]_{\mathcal{M}}$ has been added to the hypothesis.
Furthermore, was the old leaf $s_0$ of the discrimination tree split and replaced by the discriminator $water$, which distinguishes between $s_0 (\hat = [\varepsilon]_{\mathcal{M}})$ and $s_1 (\hat = [button]_{\mathcal{M}})$ according to the behavior of the target system.
The labels of the edges were obtained by evaluating $P([s_0]_{\mathcal{H}} \cdot water)$ and $P([s_1]_{\mathcal{H}} \cdot water)$ respectively.

With the updated information, the new canonical hypothesis can be constructed.
While the definition of the output function $\lambda_{\mathcal{H}}$ remains similar to the scenario described above, the definition of $\delta_{\mathcal{H}}$ now needs to consider multiple possible target states.
The construction of a hypothesis that adheres to the behavior of the target system can be realized by the concept of sifting access sequences through the discrimination tree.

By definition, the value of e.g. $\delta_{\mathcal{H}}(s_1, pod)$ is defined by $[[s_1]_{\mathcal{H}} \cdot pod]_{\mathcal{M}}$.
Given the current knowledge, this equivalence class either coincides with $[\varepsilon]_{\mathcal{M}}$ or $[button]_{\mathcal{M}}$, whose elements can be distinguished by the input sequence $water$.
Therefore, by evaluating $P(button \cdot pod \cdot water)$ and choosing the corresponding child node in the discrimination tree, the proper representative can be determined.
Applying this concept to the remaining transitions (i.e. sifting access sequences of states through the discrimination tree) the hypothesis shown in \Cref{fig:dt_2} can be constructed.
It may then be proposed to an equivalence oracle again and further refinement steps may be triggered.
Repeating this process until eventually all equivalence classes of the true target system are discovered, allows to infer the true canonical automaton.

Special focus should be denoted to the highlighted transition $\langle s_0, button \rangle$.
During the learning process, most of the transitions are defined by sifting the corresponding access sequences through the discrimination tree.
This means the successor states are solely determined by the output behavior of the target system.
However, this \enquote{knowledge} is not certain until all equivalence classes are discovered, because future refinement steps may split hypothesis states and alter transitions.

An exception to that are the so called spanning-tree transitions.
When refining the hypothesis, the counterexample decomposition yields a representative $[\delta_{\mathcal{H}}(u)]_{\mathcal{H}}$, whose one-letter extension $[\delta_{\mathcal{H}}(u)]_{\mathcal{H}} \cdot a$ corresponds to a new equivalence class.
This stepwise construction of representatives makes the set of all representatives prefix-closed.
Correspondingly, the transitions representing these one-letter extensions (i.e. the transitions leading into newly added states) form a spanning tree of the hypothesis.

A useful property of these spanning-tree transitions is, that their induced behavior coincides with the behavior of the target system.
The prefix-closed set of representatives resembles a spanning tree of the target system as well and the individual representatives resemble access sequences to the states of the target system.
Hence, tracing the representatives of the discovered equivalence classes in the hypothesis -- which iterates over the spanning-tree transitions -- is sufficient to obtain the true behavior of the target system.

Since its proposal, the Discrimination Tree algorithm has been subject to further improvements and extensions \cite{ttt}.
Its core data structure -- the discrimination tree -- offers a flexible concept that allows to integrate techniques from different fields of research, such as model-based testing.
An example is given by the developed base algorithm (cf. \Cref{cha:aal}), which will use the core principles of the Discrimination Tree algorithm and extend its ideas to an adaptive scenario, which will allow the utilization of the second conceptual influence: adaptive distinguishing sequences.

\section{Adaptive Distinguishing Sequences}
\label{sec:ads}

Model-based testing has in recent years, similar to active automata learning, gained more and more attention from real-world applications.
Beneficial to their success is the fact, that the two fields share a lot of concepts and ideas \cite{correspondence}, which means both fields contribute to each other's success.
Of particular interest for the topic of this thesis is the common problem of \emph{state identification}.

The classic question for the state identification problem is as follows:
Given a system that is in an unknown state and the possibility to interact with said system.
After applying a sequence of inputs and observing a sequence of outputs: In which state was the target system initially (i.e. before applying the sequence of input symbols)?

In active automata learning, this question is usually answered by posing a series of membership queries.
Recall from \Cref{sec:dt} the approach to determine the successor of a transition:
For determining the value of $\delta_{\mathcal{H}}(s_1, pod)$ the target system was first transitioned into the state of interest by applying the input sequence $button \cdot water$.
Then the input sequence of a discriminator was applied to decide, based on the observed output, which initial states can be disregarded.
The process is repeated, until the set of possible initial states is narrowed down to a single state.

In conformance testing, the classic approach involves \emph{distinguishing sequences}, which come either in a preset or adaptive form.
A preset distinguishing sequence (PDS) is a single input sequence that, once applied, produces a unique output sequence for each state of the target system.
This means, by only observing the single output sequence, one can determine the initial state.
The term \emph{preset} comes from the fact that the complete input sequence is determined beforehand and applied as a whole, so it does not change during application.
In contrast to that, an \emph{adaptive} distinguishing sequence (ADS) is applied symbol-wise.
After each input symbol the reaction of the target system is observed and depending on the output, the next input symbol to query is selected.
This gives, albeit being called a sequence, most ADSs the form of a decision tree.

Thus, where active automata learning needs a multitude of membership queries -- each requiring a reset to ensure their independence -- the same problem can be (potentially) solved by a single distinguishing sequence.
This raises the question, if distinguishing sequences may replace certain sets of membership queries and therefore improve the performance of the learning process with regard to the number of required resets.
A question, this thesis will investigate.

Regarding practicability, ADSs are more preferable than PDSs, because every PDS can be transformed into its adaptive counter part \cite{gill1961state}.
On the contrary however, there exist systems that only have an ADS but no PDS \cite{krichenStateIdentification}.
Yet, even the existence of an ADS cannot be guaranteed for every system, as there are examples of systems that neither have a PDS nor an ADS \cite{gill1961state,krichenStateIdentification}.
In the context of active learning however, the potential absence of a distinguishing sequence does not pose a problem, because one can always resort to the initial membership query based approach to solve the state identification problem.

Regarding performance, ADSs again yield better results.
Lee and Yannakakis have shown in \cite{lee1994testing} that the computation of a PDS is $PSPACE$-complete and that there exist systems, whose PDS is of exponential length\footnote{Exponential in the number of states of the automaton.}.
For ADSs, Rystsov \cite{rystsov76} proved a (tight) quadratic upper bound for the length of an ADS (i.e. depth of the decision tree) while again Lee and Yannakakis proposed a polynomial (quadratic) time algorithm to compute an (quadratically bound) ADS \cite{lee1994testing}.

Aside from the \enquote{normal} use case, computing an ADS is hard.
The computation of an optimal ADS with regard to certain measures is $NP$-complete \cite{optiADS}.
Furthermore, the general consensus when talking about distinguishing sequences is, to distinguish between all states of the target system.
Within this thesis there will often occur the situation, where an ADSs is only needed for a subset of states of the target system.
While for two states, the problem breaks down to the state equivalence problem, which can be solved in polynomial time and with a linear bound ADS \cite{Moore56}, the situation for $2 < m < \vert S \vert$ states is worse.
Lee and Yannakakis have shown that for an arbitrary set of states (of size $m$), the computation of an ADS is $PSPACE$-complete.
Regarding the length of these ADSs, only exponential bounds \cite{sokolovskii,zbMATH03423955,gobershtein} are known to the author.

%% file: figures/coffeemachine.tex
\begin{tikzpicture}[thick, ->, >=stealth']

	\clip (-8.2,1) rectangle (9.6,-10.6);
	\node[state, initial above, initial text=] (a) {a};
	\node[state] (b) at(-5,-3) {b};
	\node[state] (c) at( 5,-3) {c};
	\node[state] (d) at(-5,-7) {d};
	\node[state] (d') at( 5,-7) {d'};
	\node[state] (e) at( 0,-7) {e};
	\node[state] (f) at( 9,-9) {f};

	\draw (a) edge[bend right=10] node[anchor=south east] {$pod$ / \coffeeok} (b);
	\draw (a) edge[bend left=10] node[anchor=south west] {$water$ / \coffeeok} (c);
	\draw (a) edge[out=0,in=90, looseness=1.8] node[pos=0.2,anchor=south] {$button$ / \coffeeerror} (f);
	\draw (a) edge[loop left] node[anchor=east] {$clean$ / \coffeeok} (a);

	\draw (b) edge[loop below] node[anchor=north west] {$pod$ / \coffeeok} (b);
	\draw (b) edge[bend right] node[anchor=west] {$water$ / \coffeeok} (d);
	\draw (b) edge[out=225, in=180, looseness=1.9] node[anchor=north] {$button$ / \coffeeerror} (f);
	\draw (b) edge[out=0,in=270,looseness=1] (a);

	\draw (c) edge[bend left] node[anchor=west] {$pod$ / \coffeeok} (d');
	\draw (c) edge[loop below] node[anchor=north east] {$water$ / \coffeeok} (c);
	\draw (c) edge[out=0,in=90,looseness=1.5] node[anchor=south,pos=0.15] {$button$ / \coffeeerror} (f);
	\draw (c) edge[out=180,in=270,looseness=1] (a);

	\draw (d) edge[loop below] node[anchor=north] {$\{pod, water\}$ / \coffeeok} (d);
	\draw (d) edge node[anchor=south] {$button$ / \coffeecup} (e);
	\draw (d) edge[out=45,in=270,in looseness=1.5,out looseness=0] (a);

	\draw (d') edge[loop below] node[anchor=north] {$\{pod, water\}$ / \coffeeok} (d');
	\draw (d') edge node[anchor=south] {$button$ / \coffeecup} (e);
	\draw (d') edge[out=135,in=270,out looseness=0,in looseness=1] (a);

	\draw (e) edge node[anchor=west,pos=0.9] {$clean$ / \coffeeok} (a);
	\draw (e) edge[out=270,in=180,out looseness=0.5,in looseness=0] node[anchor=north east,pos=0.15] {$I \setminus \{clean\} / $\coffeeerror} (f);

	\draw (f) edge[loop below] node[anchor=north east] {$I / $\coffeeerror} (f);
\end{tikzpicture}

%% file: figures/dt_1.tex
\centering
\subcaptionbox
{Initial hypothesis}
[0.5\textwidth]
{
	\begin{tikzpicture}[->,thick,>=stealth']
		\node[state,initial, initial text={}] (q0) at(0,0) {$s_0$};
		\draw (q0) edge[loop above] node[anchor=south,align=center] {$I \setminus \{button\}$ / \coffeeok\\$button$ / \coffeeerror} (q0);
	\end{tikzpicture}
}%
\subcaptionbox
{Initial discrimination tree}
[0.5\textwidth]
{
	\begin{tikzpicture}
		\node[final] (q1) {$s_0$};
	\end{tikzpicture}
}

%% file: figures/dt_2.tex
\centering
\subcaptionbox
{Refined hypothesis}
[0.5\textwidth]
{
	\begin{tikzpicture}[->,thick,>=stealth']
		\node[state,initial, initial text={}] (q0) at(0,0) {$s_0$};
		\draw (q0) edge[loop above] node[anchor=south] {$I \setminus \{button\}$ / \coffeeok} (q0);

		\node[state] (q1) at(4,0) {$s_1$};
		\draw (q0) edge[ultra thick] node[anchor=south] {$button$ / \coffeeerror} (q1);
		\draw (q1) edge[loop above] node[anchor=south] {$I$ / \coffeeerror} (q1);
	\end{tikzpicture}
}%
\subcaptionbox
{Refined discrimination tree}
[0.5\textwidth]
{
	\begin{tikzpicture}[->,thick,>=stealth']
		\node[symbol] (q0) at(0,0) {$water$};
		\node[final] (q1) at(-1.5,-2) {$s_0$};
		\node[final] (q2) at(1.5,-2) {$s_1$};

		\draw (q0) edge node[anchor=south east] {\coffeeok} (q1);
		\draw (q0) edge node[anchor=south west] {\coffeeerror} (q2);
	\end{tikzpicture}
}

%% file: sections/learning.tex
\chapter{Active Automata Learning in an Adaptive Environment}
\label{cha:aal}

This chapter introduces the \emph{ADTLearner}, the adaptive extension of the Discrimination Tree algorithm that allows the integration of adaptive distinguishing sequences.
To allow its specification, \Cref{sec:aalpre} presents preliminary concepts that are required by the developed approach.
\Cref{sec:base} then formalizes the algorithm, shows an exemplary execution and proves that the presented base version alone, represents a fully functional learning algorithm for inferring regular Mealy machines.

\section{Preliminaries}
\label{sec:aalpre}

Key to the developed approach is a refined method of interacting with the target, which will allow the utilization of adaptive queries, as required by adaptive distinguishing sequences.

\subsection{Symbol Query Oracle}

As pointed out in \Cref{sec:ads}, an adaptive distinguishing sequence is essentially a decision tree with input symbols labeling inner nodes.
The adaptiveness comes from the possibility to dynamically decide which symbol to query next, based on former observations.
This however poses a problem to the classic communication mechanism provided by a membership oracle.
\Cref{fig:sqo} shows such a problematic case:

\begin{figure}[h]
	\centering
	\input{figures/adaptivequery.tex}
	\caption{An example of an adaptive query}
	\label{fig:sqo}
\end{figure}
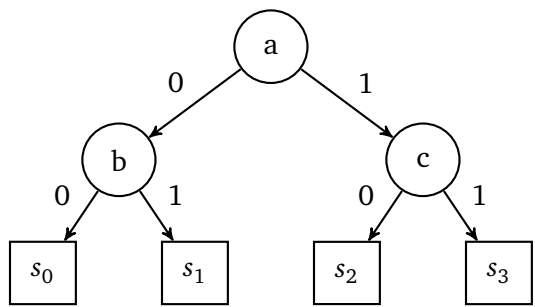

Since classic membership queries are preset, the query of \Cref{fig:sqo} cannot be answered by a single membership query, because one would have to query two distinct symbols simultaneously.
One could answer the query by posing two separate membership queries -- namely $mq(ab)$ and $mq(ac)$ -- and evaluate the complete answer to decide for the possible state.
From a performance point of view however, this approach is not acceptable due to its significant overhead.
In the worst case, every pair of leaves would require a separate membership query.
Hence, the developed approach introduces a new kind of oracle: the \emph{symbol query oracle}.

A symbol query oracle allows two kinds of possible queries: reset and symbol queries.

\begin{itemize}

\item A reset query is a query that resets the SUL into its initial state and has no return value.
As classic membership oracles also require a certain reset mechanism to ensure the independence of membership queries, a reset query does not pose new requirements to a SUL.

\item A symbol query executes a single input symbol on the SUL and returns the observed output symbol of the SUL.
In contrast to a classic membership queries, no reset query precedes a symbol query, which makes a series of symbol queries generally dependent on each other.
Since the SUL (or rather its abstraction) is assumed to be a reactive system (i.e. representable by a Mealy machine), this only changes the way of communication and does not impose further restrictions.

\end{itemize}

As for interoperability, a membership oracle can always be simulated by a symbol query oracle.
A membership query can always be answered by executing a reset query followed by subsequent symbol queries processing the input symbols of the membership query.
Depending on the use case, either the complete response or simply the last symbol of the answer may be extracted.
Therefore, if in the following the term membership query is used, it usually refers to the simulated version.

\subsection{Adaptive Discrimination Tree}
\label{sec:adt}

The core data structure of the developed approach to model the progress of the learning process is the \emph{adaptive discrimination tree}, or short ADT.
An ADT is a mixture-model of an adaptive distinguishing sequence and a discrimination tree.
One can think of it as an adaptive distinguishing sequence enhanced with the possibility to reset during a sifting operation or an discrimination tree with the ability to dynamically change the query of an inner node.

The structure of an adaptive discrimination tree is subject to a certain set of constraints:

\begin{definition}[Adaptive Discrimination Tree]
Let $I$ denote an input- and $O$ denote an output alphabet.
An adaptive discrimination tree is a rooted tree, that can contain three kinds of nodes: symbol, reset and final nodes.

\begin{compactitem}

\item An ADT consists either of a single or multiple nodes.

\begin{compactitem}
\item If the tree consists of a single node, this node must be a final node.
\item If the tree consists of multiple nodes, the root node must be a symbol node.
\end{compactitem}

\item Each symbol node references exactly one input symbol $i \in I$.
A symbol node must have at least one but may have multiple children, which can be of any kind.
A symbol node and its children are connected by a labeled edge, with label $o \in O$.

\item Reset nodes must have exactly one child.
A reset node and its child are connected by an unlabeled edge.

\item Final nodes must not have any children.
A final node references exactly one state of the tentative hypothesis. 
Every leaf of an ADT must be a final node.

\end{compactitem}
\end{definition}

During the learning process an ADT may be subject to changes.
It may be expanded due to new observations or certain subtrees might be updated to utilize adaptive distinguishing sequences (cf. \Cref{cha:replacements}).
In order to allow proofs for correctness, the property of a \emph{verified} ADT is introduced.
For the ease of notation, let us first introduce the idea of \emph{paths} and \emph{traces}.

\begin{definition}[Path of a node]\label{def:paths}
Let $\mathcal{ADT}$ denote an adaptive discrimination tree and $n$ a node of $\mathcal{ADT}$.
Then $path_{\mathcal{ADT}}(n)$ is defined as the sequence $\langle e_1, v_1 \rangle, ..., \langle e_k, v_k \rangle$ with following properties:
$e_{i+1}$ is the edge connecting the node $v_i$ with its parent node $v_{i+1}$.
In particular, $v_1$ is the parent node of $n$ and $v_k$ is the root node of $\mathcal{ADT}$.
If $n$ is the root node of $\mathcal{ADT}$, its path is defined as the empty sequence.
\end{definition}

\begin{definition}[Traces of a node]\label{def:traces}
Let $\mathcal{ADT}$ denote an adaptive discrimination tree, $n$ a node of $\mathcal{ADT}$, $r$ the root node of $\mathcal{ADT}$ and $path_{\mathcal{ADT}}(n) = \langle e_1, v_1 \rangle, ..., \langle e_k, v_k \rangle$ the path of $n$.
Furthermore let $s()$ denote a function that extracts the label (symbol) of an edge (node), and $rn()$ an indicator function that returns true, if a node is a reset node.
Then the traces of $n$ are defined as follows:

\begin{align*}
trace_{\mathcal{ADT}}(n) &=
\begin{cases}
\langle \varepsilon, \varepsilon \rangle & \text{if } n = r\\
\langle \varepsilon, \varepsilon \rangle & \text{if } rn(v_1)\\
\langle s(v_k) \cdots s(v_1), s(e_k) \cdots s(e_1) \rangle & \text{if } \neg rn(v_i), 1 \leq i \leq k\\
\langle s(v_{m-1}) \cdots s(v_1), s(e_{m-1}) \cdots s(e_1) \rangle, m = \min\limits_i rn(v_i) & \text{otherwise}
\end{cases}\\
traces_{\mathcal{ADT}}(n) &=
\begin{cases}
\emptyset & \text{if } n = r\\
traces_{\mathcal{ADT}}(v_1) & \text{if } rn(v_1)\\
\{trace_{\mathcal{ADT}}(n)\} \cup traces_{\mathcal{ADT}}(v_m), m = \min\limits_i v_i = r \lor rn(v_i) & \text{otherwise}
\end{cases}
\end{align*}

\end{definition}

\Cref{fig:adt_traces} shows an example of an adaptive discrimination tree as well as evaluations of the trace functions for its final nodes.

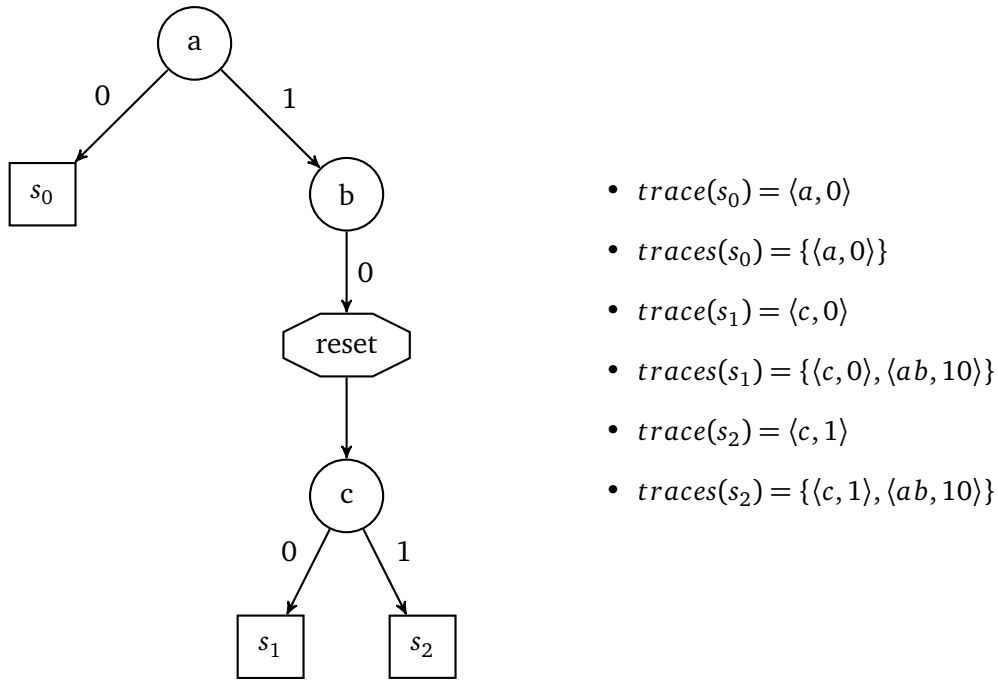
\begin{figure}[ht]
	\input{figures/traces.tex}
	\caption{An example of an ADT and its traces}
	\label{fig:adt_traces}
\end{figure}

Traces allow to extract the behavioral information that are stored in an adaptive discrimination tree.
With these information, the concept of a verified ADT can be defined.

\begin{definition}[Verified Adaptive Discrimination Tree]
Let $\mathcal{ADT}$ denote an adaptive discrimination tree and $rep_f$ the input sequence $[s_f]_{\mathcal{H}}$ for the referenced hypothesis state $s_f$ in the final node $f$.
$\mathcal{ADT}$ is called verified iff for all its final nodes $f$, we have:
$$\forall \langle i, o \rangle \in traces(f): mq(rep_f, i ) = o$$
\end{definition}

Intuitively a verified ADT only describes true behavior of the SUL as it has been verified by membership queries.
For verified ADTs one can now show that the referenced hypothesis states in the leaves of the ADT truly represent distinct equivalence classes of the SUL.

\begin{theorem}[Correctness of ADT]\label{thm:adt}
Given a verified adaptive discrimination tree $\mathcal{ADT}$, every final node of $\mathcal{ADT}$ represents a distinct (set of) equivalence class(es) of the SUL.
\end{theorem}

\begin{proof}
Assume for contradiction, that two distinct final nodes $f_1, f_2$ represent the same equivalence class of the SUL.
Given that $\mathcal{ADT}$ follows a tree structure, $f_1$ and $f_2$ must have a lowest common ancestor $n_{lca}$.
By definition, $n_{lca}$ must be a symbol node, because only symbol nodes are allowed to have multiple children.
Let $\langle i_{lca}, o_{lca} \rangle = trace_{\mathcal{ADT}}(n_{lca})$ denote the behavioral information of $n_{lca}$, $i = i_{lca} \cdot s(n_{lca})$ denote a discriminating input sequence and $rep_{f_i}$ denote the representative $[s_{f_i}]_{\mathcal{H}}$ of the hypothesis state $s_{f_i}$ referenced in the final node $f_i$.
Since $\mathcal{ADT}$ is a verified adaptive discrimination tree, the following holds:

$$mq(rep_{f_1} \cdot i) \neq mq(rep_{f_2} \cdot i)$$

This contradicts the assumption, that $f_1$ and $f_2$ represent the same equivalence class, because there exists an input sequence that results in different output behavior.
\end{proof}

To utilize the knowledge about the (distinct) equivalence classes an adaptive discrimination tree represents, it provides access to a \textsc{sift} operation, which takes an input word as a parameter and returns a final node of the ADT.
Starting at the root node of the ADT, the sift operation iterates over a sequence of nodes.
If the current node is a symbol node, a symbol query with the referenced symbol is executed.
Depending on the observed output, the child with the correspondingly labeled edge is selected as the next node.
If during the iteration a reset node is encountered, a reset query followed by a sequence of symbol queries representing the initial input parameters is executed.
Once a final node is reached, it will be returned.

During the sift operation it might occur, that for a certain observed output symbol, the current symbol node has no defined successor.
In this case a new final node will be added to the ADT and set as the previously missing successor.

Regarding nomenclature, the term \enquote{subtree} may denote any tree rooted in a specific node of the adaptive discrimination tree.
However, in most situations, the distinguishing property of an incorporated discriminator is needed. 
Therefore, the term subtree (or sub-ADT) will usually refer to subtrees rooted in a symbol node that succeeds a reset node.
For example, in \Cref{fig:adt_traces}, the symbol node $c$ may be referred to as a sub-ADT.

Furthermore does an adaptive discrimination tree generalize the concepts of an adaptive distinguishing sequence.
However, the two terms will be used to emphasize the presence of reset nodes:
When referring to an adaptive distinguishing sequence, the absence of reset nodes is assumed, while adaptive discrimination trees usually contain reset nodes.

\section{Base Algorithm}
\label{sec:base}

This section presents the adaptive base algorithm.
While not yet including any specific techniques to include adaptive distinguishing sequences, the base version will serve as a sound basis that allows proofs for termination and correctness and introduces certain components and concepts that will be referenced in the subsequent chapters.
Adjusting to the workflow of the active learning loop (cf. \Cref{fig:aal}) the algorithm will be specified by means of an \emph{initialization} procedure, that will be executed once at the beginning of the learning process and a \emph{refinement} procedure, that will receive counterexamples returned by the equivalence oracle.

During its execution, the learner instance needs to access shared data structures in order to retrieve information from previous refinement steps.
The following enumeration lists the shared variables that are expected to be available at a global scope.

\begin{description}

\item[hypothesis]
The internal hypothesis $\mathcal{H} = \langle S, s_0, I, O, \delta, \lambda \rangle$ that represents the current approximation of the system to learn.
For the ease of notation, the states $S$ are assumed to be integers to allow using them as indexes in array-like structures.
In most occasions, the specific components of the hypothesis (e.g. $\delta$, ...) will be used directly.

\item[accessSequences]
An array-like structure which stores the representatives $[s]_{\mathcal{H}} \in I^*$ for a given hypothesis state $s \in S$.

\item[sqo]
The symbol query oracle that allows the learning algorithm to post parameterized symbol- and reset-queries to retrieve information about the target system.
For convenience, the query function may also receive input \emph{words}, which results in subsequent symbol queries.

\item[adt]
The adaptive discrimination tree, as described in \Cref{sec:adt}.

\item[openTransitions]
A queue-like structure that holds descriptors of the hypothesis transitions.
A transition will be described by a tuple $\langle source, input, output, target \rangle$, describing a transition that originates in state $source$ and transitions into state $target$ on input $input$ while emitting $output$.
In instances where the parameters $output$ and $target$ are irrelevant, the shorthand notation $\langle source, input \rangle$ is used.

\item[openCounterExamples]
A queue-like structure that holds potential counterexamples.
A counterexample will be described by a tuple $\langle in, out \rangle$, that contains an input sequence $in$ leading to the output sequence $out$ in the system under learning.

\end{description}

For interacting with the global variables, an object-oriented visualization will be used, meaning methods will be invoked \emph{on} objects.
This is for example the case for the queue-like structures \textbf{openTransitions} and \textbf{openCounterExamples}, for which the following methods are assumed to be available:

\begin{description}

\item[add]
Adds a new element to the end of the queue.

\item[pop]
Retrieves and removes the first element of the queue.

\item[isEmpty]
returns \texttt{true} if the queue is empty and \texttt{false} otherwise.

\end{description}

The initialization step of the ADTLearner is described in \Cref{alg:adtinit}.

\begin{algorithm}[p]
	\input{alg/adt_init.tex}
	\caption{ADTLearner: Initialization}
	\label{alg:adtinit}
\end{algorithm}

The initial state of the hypothesis as well as its (trivial) representative is set.
The \textsc{initializeADT} call initializes the adaptive discrimination tree with a single final node, referencing the initial hypothesis state $s_0$.
The outgoing transitions of the initial state are then added to the \textbf{openTransitions} queue, for which the \textsc{closeTransitions} procedure (cf. \Cref{alg:adtclose}) will determine output and successor values.
Since this procedure will be used by the refinement step as well, let us first describe the refinement procedure to explain the closing of transitions in the context of both use cases.
The refinement step is separated in two functions, displayed in \Cref{alg:adtref}.

\begin{algorithm}[p]
	\input{alg/adt_refine.tex}
	\caption{ADTLearner: Handling of counterexamples}
	\label{alg:adtref}
\end{algorithm}

The \enquote{official} refinement procedure is merely a wrapper for managing counterexamples and ensuring consistency with the current adaptive discrimination tree.
While each true counterexample discovers a new equivalence class and therefore alters the hypothesis, it may happen that the refined hypothesis is still not consistent with past observations.
For example, even \emph{after} a refinement triggered by a counterexample $ce = \langle in, out\rangle$, it may still hold that $\lambda_{\mathcal{H}}(in) \neq out$.
Therefore, at first, a counterexample is reevaluated until it is no longer a valid counterexample.
The motivation to wrap these operations in an additional while loop is given by future extensions (cf. \Cref{sec:immediateRepl}) which may detect additional counterexamples during internal refinement steps.

A similar effect can occur for the discriminators stored in the current adaptive discrimination tree.
Therefore, after the internal refinement steps finished, the \textsc{ensureADTConsistency} procedure checks for every final node $f$ of the current adaptive discrimination tree if the hypothesis exhibits the behavior described in $traces_{adt}(f)$.
If any diverging behavior is observed, a counterexample is added to the \textbf{openCounterExamples} queue and the outer while-loop refines the hypothesis as needed.
The main refinement step that updates the hypothesis, is given by the \textsc{refineHypothesisInternal} procedure.

At first it is checked, if the current hypothesis already outputs the expected output of the current counterexample.
If so, the function returns \texttt{false} to indicate that no additional information can be extracted from the given observations.
If, however, the counterexample is still valid, the hypothesis needs refinement.

Therefore the counterexample is decomposed into a triplet $\langle u, a, v\rangle$ as described in \Cref{sec:dt}.
A new state is added to the hypothesis in order to reflect the partition refinement and the additional data structures are updated.
The access sequence of the new state is set to the one symbol-extension of the access sequence represented by state $\delta(u)$ and the transition upon receiving the input symbol $a$ is set to the new state in order to represent the newly discovered equivalence class $[accessSequences[\delta_{\mathcal{H}}(u)] \cdot a]_{\mathcal{M}}$.

In the following, the adaptive discrimination tree needs to be updated to ensure the correct distinction between the (newly) discovered equivalence classes.
In the presented base version of the learner, this may simply be accomplished by replacing the leaf referencing the node to be split ($\delta(ua)$) with a reset node.
The reset node is then followed by a sequence of symbol nodes representing the distinguishing suffix $v$.
The labels for the intermediate edges are determined by the outputs of the target system, that are given by $mq(accessSequences[\delta(ua)], v)$ for the old, too coarse representative and by \linebreak $mq(accessSequences[n], v)$ for the newly added representative.
Since $v$ is a true discriminator for the two equivalence classes, the outputs of the two membership queries will differ at one point.
This very symbol node may then be succeeded by two final nodes (with correspondingly labeled edges), referencing the old, refined hypothesis state $\delta(ua)$ and the new state $n$.
Using this kind of replacement maintains the property of a verified adaptive discrimination tree, which therefore completes the update procedure of the ADT.

In order to reflect the required changes in the hypothesis and maintain canonicity, certain transitions need to updated.
It is easy to verify, that it is sufficient to only update the incoming transitions of $\delta(ua)$ and the outgoing transitions of $n$.
For all other transitions, the updated ADT would determine the same target states as before.

This closing of a transition is displayed in \Cref{alg:adtclose}.

\begin{algorithm}[t]
	\input{alg/adt_close.tex}
	\caption{ADTLearner: Closing the open transitions}
	\label{alg:adtclose}
\end{algorithm}

For each transition, the system under learning is first brought to a state representing the source equivalence class by resetting the system and then applying the access sequence of the source state.
Afterwards the input symbol of the transition is queried to determine its output.
After querying said input symbol, the system under learning is in a state that is relevant for determining the value of $\delta(s, i)$.
In order to reconstruct this situation, a temporary access sequence is stored in the local variable $lp$, which is then used in the sifting method of the ADT to determine the local hypothesis state.

As stated in \Cref{sec:adt}, the sift operation may return a (newly constructed) final node which does not reference any hypothesis state if undefined behavior is observed during the process of sifting.
If this scenario is encountered (cf. line~\ref{lst:adt_newequiv}) a new state is added to the hypothesis and its outgoing transitions are added to the \textbf{openTransitions} queue.
Note that in this case, the adaptive discrimination tree is still verified and it is sufficient to only add the outgoing transitions of the new state, since no other state could have possibly referenced this state before.
Otherwise the reference of the returned leaf is used to determine the target state for the current transition.

\subsection{Example}
\label{sec:aalex}

For a better understanding, this section will visualize the procedures and internal states of the proposed base algorithm by presenting the first iterations for the given running example.
Besides the summarizing aspect of this description, a special focus should be attributed to the state of the hypothesis and the state of the adaptive distinguishing tree during refinement steps.
With regard to \Cref{cha:heuristics}, the interaction of these two components will determine which heuristics can be applied at which point of time during the learning process.

The first interaction takes place during the initialization phase, where the initial hypothesis and the initial adaptive discrimination tree are constructed.
However, since the initial discrimination tree does not distinguish between any equivalence classes -- there does not exist any evidence for multiple equivalence classes --, the initial hypothesis will always result in the single-state automaton and therefore resemble the situation depicted in \Cref{fig:dt_1}.
As a result, let us continue with the encounter of the first counterexample.

The situation after decomposing $\hat c = button \cdot water$, and updating the adaptive discrimination tree is depicted in \Cref{fig:learning_ex2}.

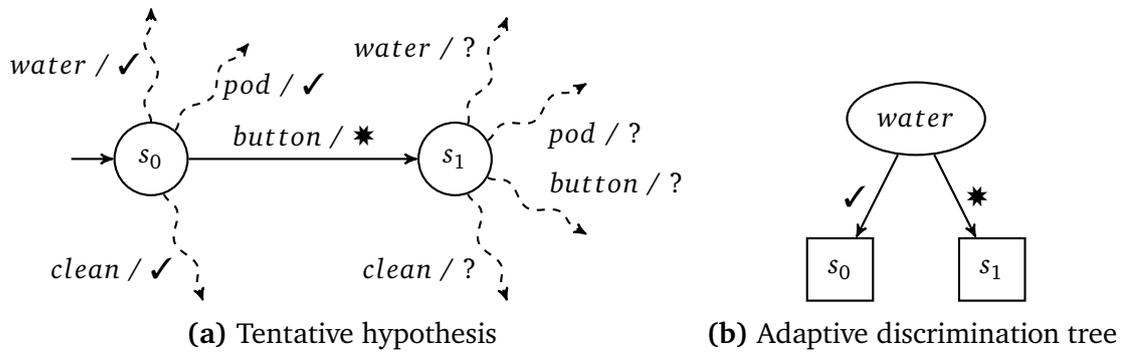
\begin{figure}[h]
	\input{figures/adt_ex2.tex}
	\caption{Tentative hypothesis and adaptive discrimination tree after the first counterexample analysis}
	\label{fig:learning_ex2}
\end{figure}

The only certain knowledge the learner has, is that the transition $\langle s_0, button \rangle$ transitions the target system into a state truly different from $s_0$.
However, for all remaining transitions, the successor is unknown.
Only by sifting the corresponding access sequences of their target states, this knowledge can be obtained.
The closing of the remaining (dashed) transitions does not result in any further discoveries of new equivalence classes.
After terminating the \textsc{closeTransitions} procedure, the hypothesis returned by the learning process therefore resembles the hypothesis already shown in \Cref{fig:dt_2}.

When comparing the new hypothesis with the true target system (cf. \Cref{fig:coffeemachine}) it is apparent that the two models are still not equivalent yet.
An equivalence query might therefore yield the counterexample $\hat c = pod \cdot water \cdot pod \cdot water \cdot button$, for which the hypothesis outputs $\coffeeok \cdot \coffeeok \cdot \coffeeok \cdot \coffeeok \cdot \coffeeerror$, whereas the target system outputs $\coffeeok \cdot \coffeeok \cdot \coffeeok \cdot \coffeeok \cdot \coffeecup$. 
Therefore $\lambda_{\mathcal{H}}(\hat c) \neq \lambda_{\mathcal{M}}(\hat c)$ and the \textsc{refineHypothesisInternal} procedure will continue with a true refinement step.

This time, the counterexample decomposes into the triplet $\langle \varepsilon, pod, water \cdot button \rangle$ meaning, the action $pod$ transitions the hypothesis and the target system from the state reached by $\varepsilon$ ($s_0$) to different successor states, as indicated by the diverging output upon applying $water \cdot button$.
Therefore the state $s_0$ represents two access sequences ($\varepsilon$ and $pod$) that belong to provably different equivalence classes of the target system.
To reflect this information, a new state -- the $pod$-successor of $s_0$ -- is added to the hypothesis, which represents the equivalence class of which $pod$ is a member.
Similar to the previous refinement steps, the corresponding data structures are updated.
Furthermore, by posing the membership queries $mq(\varepsilon, water \cdot button)$ and $mq(pod, water \cdot button)$ it is possible to split the final node $s_0$ of the current adaptive discrimination tree and correctly distinguish between $s_0$ and the new state $s_2$.
But what impact has this refinement step on the transitions?

First of all, all outgoing edges of the new state $s_2$ are undefined:
Since $pod$ (i.e. $\varepsilon \cdot pod$) has only been a one-symbol extension of an access sequence so far, it was never resolved in which target state the access sequences $pod \cdot water, pod \cdot pod,$ etc. led.
Additionally, all incoming transitions of the too coarse state $s_0$ need to be refined:
Previous sifting operations only revealed, that e.g. the input symbol $water$ leads to a state that emits a \coffeeok-symbol when receiving another $water$ symbol query.
However, this is now true for both the states $s_0$ and $s_2$.
In order to correctly single out the specific target state, the transitions, i.e. the access sequence of their source states concatenated with their corresponding input symbol, need to be sifted through the subtree representing the new discriminator $water \cdot button$.

The situation of the tentative hypothesis and the adaptive discrimination tree after splitting the too coarse state $s_0$ are displayed in \Cref{fig:learning_ex3}.

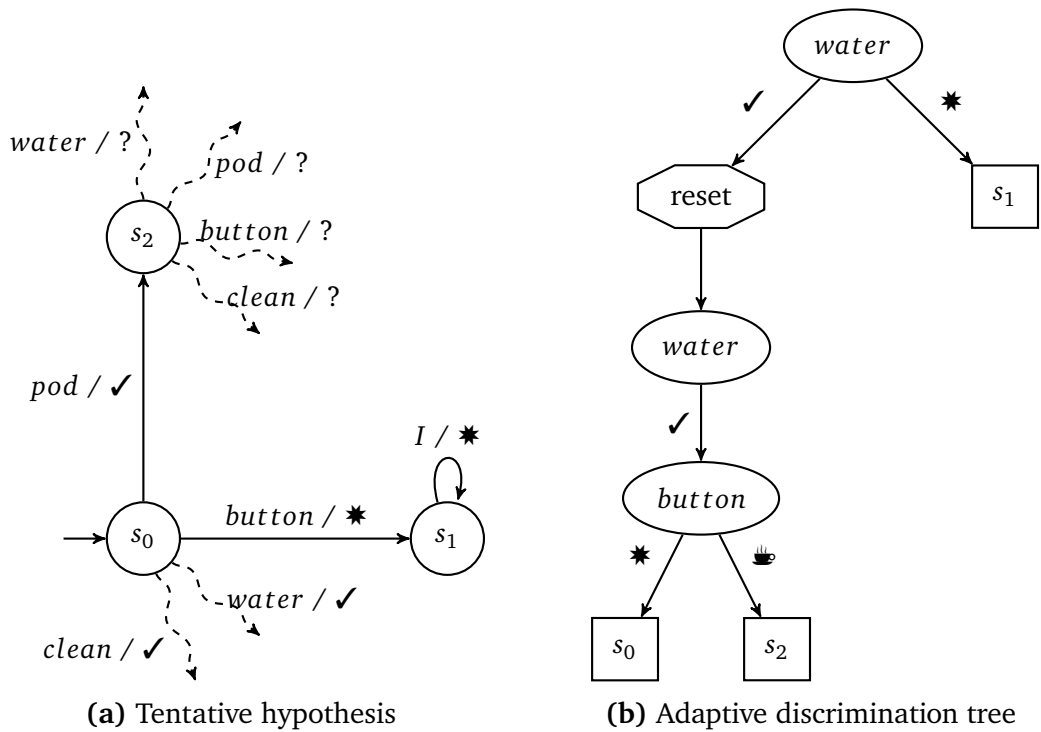
\begin{figure}[ht]
	\input{figures/adt_ex3.tex}
	\caption{Tentative hypothesis and adaptive discrimination tree after the second counterexample analysis}
	\label{fig:learning_ex3}
\end{figure}

As defined in \Cref{alg:adtref} the refinement step is finished by a call to the \textsc{closeTransitions} procedure.
By checking the posed symbol queries during the sifting operations, one can see that the \textsc{closeTransitions} procedure will not discover any new equivalence classes.
Therefore, after the internal refinement procedure finishes, the tentative hypothesis will be a three state hypothesis.
However the current counterexample $\hat c = pod \cdot water \cdot pod \cdot water \cdot button$ still poses a valid counterexample, as even the refined hypothesis will output \coffeeok $\cdot$ \coffeeok $\cdot$ \coffeeok $\cdot$ \coffeeok $\cdot$ \coffeeerror.
Hence a second refinement step will be triggered.

In the second iteration, the counterexample will decompose into the triplet \linebreak $\langle pod, water, button \rangle$, resulting in the splitting of state $s_2$ and the creation of state $s_3$.
The remaining refinement steps follow the patterns described above.
The next section will show, that this approach will eventually terminate and return the correct behavioral model of the target system.

\subsection{Termination \& Correctness}

Given the base algorithm, termination and correctness can be proved for the learning approach with adaptive discrimination trees.
A useful property for the two proofs is the canonicity of the intermediate hypotheses, formalized by the following lemma.

\begin{lemma}[Canonicity of intermediate hypotheses]\label{lem:canonicity}
After each phase of the algorithm (i.e. after the initialization and each refinement step) the tentative hypothesis $\mathcal{H}$ is canonical.
\end{lemma}

\begin{proof}
\textit{Initialization}:
The hypothesis is initialized with a single state, which is the only valid approximation possible, when no behavior is observed.
The empty word is selected as the representative for this equivalence class and before termination, all possible transitions are closed.

\textit{Refinement}:
Recall, that the decomposition yields input-words/-symbols such that \linebreak $mq(accessSequences[\delta(u)] \cdot a, v) \neq mq(accessSequences[\delta(ua)], v)$.
This means, that the input sequence $accessSequences[\delta(u)] \cdot a$ is wrongly attributed to the equivalence class represented by $\delta(ua)$ and hence $\delta(ua)$ is too coarse and needs refinement.
Therefore a new state -- representing $[accessSequences[\delta(u)] \cdot a]_{\mathcal{M}}$ -- is added to the hypothesis and $accessSequences[\delta(u)] \cdot a$ is set as its correct representative.
The adaptive discrimination tree is updated to incorporate the new discriminator ($v$) and the outputs and successors of all affected transitions (outgoing transitions of the new state and incoming transitions of the split one) are determined according to the definition of the canonical automaton.

\textit{Closing}:
For each pending transition $\langle s, i \rangle$, the output is determined by bringing the target system into state $s$ (by means of its access sequence) and executing the symbol query $i$.
The successor is determined by sifting the access sequence of the successor ($accessSequences[s] \cdot a$) which returns the state for the corresponding equivalence class (cf. \Cref{thm:adt}).
It may however be possible that the sifting operation yields a new final node $f$, in which case a new hypothesis state is added.
This does not break the canonicity but rather serves as an implicit counterexample, because $accessSequences[s] \cdot a$ yields unexpected output behavior.
Moreover, the adaptive discrimination tree is still verified, because the sifting operation is performed by means of reset and symbol queries, so it reflects the behavior of the true target systems.
One can see, that the input components of $traces_{adt}(f)$ resemble input sequences that truly distinguish between the new state referenced in $f$ and the (remaining) final nodes for the adaptive discrimination tree.
\end{proof}

The canonicity of intermediate hypotheses can then be used, to prove the termination of the learning algorithm, as shown in \Cref{thm:termination}.

\begin{theorem}[Termination of the base algorithm]\label{thm:termination}
The base algorithm terminates after at most $n-1$ equivalence queries, where $n$ denotes the size (i.e. the number of states) of the (minimized) target system.
\end{theorem}

\begin{proof}
The proof is similar to termination proofs of other active learning algorithms and is based on the principles of invariance and progress.

\textit{Invariant}:
The number of states of the intermediate hypothesis $\mathcal{H}$ never exceeds the number of states of the target system.
This is due to the fact that the tentative hypothesis is canonical (i.e. minimal with respect to the discovered equivalence classes, cf. \Cref{lem:canonicity}) and states are only added when a distinguishable behavior is observed.

\textit{Progress}:
Each equivalence query yields a counterexample that refines the existing partition by adding an additional state to the hypothesis. 
Given the invariant, this can only happen at most $n-1$ times.
\end{proof}

Applying \Cref{lem:canonicity} to the hypothesis after termination, allows to deduce the correctness of the final hypothesis.

\begin{theorem}[Correctness of the base algorithm]\label{thm:correctnes}
Upon termination, the base algorithm returns a hypothesis that is equivalent (up to isomorphism) to the target automaton.
\end{theorem}

\begin{proof}
The termination of the algorithm guarantees, that all true equivalence classes have been discovered.
By construction, building the canonical automaton (cf. \Cref{lem:canonicity}) from these information yields an equivalent (up to isomorphism) automaton with regard to the target system.
\end{proof}

An important aspect to note is that throughout the proofs, no assumptions about the structure of the adaptive discrimination tree were made (except being verified).
This opens the way to integrate heuristics that may improve performance without losing the properties of termination and correctness.

%% file: figures/adaptivequery.tex
\begin{tikzpicture}[thick, ->, >=stealth', level/.style={sibling distance=40mm/#1}]

	\node [state] {a}
	child {
		node[state] {b}
		child {
			node[draw,rectangle, inner sep=8pt] {$s_0$}
			edge from parent
			node[anchor=south east] {0} 
		}
		child {
			node[draw,rectangle, inner sep=8pt] {$s_1$}
			edge from parent
			node[anchor=south west] {1} 
		}
		edge from parent
		node[anchor=south east] {0}
	}
	child {
		node[state] {c}
		child {
			node[draw,rectangle, inner sep=8pt] {$s_2$}
			edge from parent
			node[anchor=south east] {0} 
		}
		child {
			node[draw,rectangle, inner sep=8pt] {$s_3$}
			edge from parent
			node[anchor=south west] {1} 
		}
		edge from parent
		node[anchor=south west] {1}
	};

\end{tikzpicture}

%% file: figures/traces.tex
\begin{minipage}{0.58\textwidth}
	\centering
	\begin{tikzpicture}[thick, ->, >=stealth', level/.style={sibling distance=40mm/#1}]

		\node[symbol] (a) at(0,0) {a};

		\node[final] (s1) at(-2, -2) {$s_0$};
		\node[symbol] (b) at(2, -2) {b};

		\node[reset] (r) at(2,-4) {};

		\node[symbol] (c) at (2,-6) {c};

		\node[final] (s2) at (1, -8) {$s_1$};
		\node[final] (s3) at (3, -8) {$s_2$};

		\draw (a) -- node[anchor=south east] {0} (s1);
		\draw (a) -- node[anchor=south west] {1} (b);

		\draw (b) -- node[anchor=west] {0} (r);

		\draw (r) -- (c);

		\draw (c) -- node[anchor=south east] {0} (s2);
		\draw (c) -- node[anchor=south west] {1} (s3);
	\end{tikzpicture}
\end{minipage}%
\begin{minipage}{0.48\textwidth}
	\centering
	\begin{itemize}
		\item $trace(s_0) = \langle a, 0 \rangle$
		\item $traces(s_0) = \{ \langle a, 0 \rangle \}$
		\item $trace(s_1) = \langle c, 0 \rangle$
		\item $traces(s_1) = \{ \langle c, 0 \rangle, \langle ab, 10 \rangle \}$
		\item $trace(s_2) = \langle c, 1 \rangle$
		\item $traces(s_2) = \{ \langle c, 1 \rangle, \langle ab, 10 \rangle \}$
	\end{itemize}
\end{minipage}

%% file: alg/adt_init.tex
\begin{algorithmic}[1]
	\Function{initialize}{}
	\Let{$s_0$}{$0$} 
	\Let{$accessSequences[s_0]$}{$\varepsilon$}
	\State{$\Call{intializeADT}{s_0}$}
	\ForAll{$i \in I$}
	\State{$openTransitions.add(\langle s_0, i\rangle)$}
	\EndFor
	\State{\Call{closeTransitions}{}}
	\EndFunction
\end{algorithmic}

%% file: alg/adt_refine.tex
\begin{algorithmic}[1]
	\Function{refineHypothesis}{$ce = \langle in, out \rangle$}
	\State{$openCounterExamples.add(ce)$}
	\While{$\Not~openCounterExamples.isEmpty()$}
	\While{$\Not~openCounterExamples.isEmpty()$}
	\Let{$currentCE$}{$openCounterExamples.pop()$}
	\While{$\Call{refineHypothesisInternal}{currentCE}$}
	\EndWhile
	\EndWhile
	\State{\Call{ensureADTConsistency}{}}
	\EndWhile
	\EndFunction
	\Statex
	\Function{refineHypothesisInternal}{$ce = \langle in, out \rangle$}
	\If{$\lambda(in) = out$}
	\State \Return{\False}
	\EndIf
	\Let{$\langle u, a, v \rangle$}{\Call{decomposeCounterExample}{$in$}}
	\Let{$n$}{$\vert S \vert$} \Comment{The new state}
	\Let{$S$}{$S \cup \{n\}$}
	\Let{$accessSequences[n]$}{$accessSequences[\delta(u)] \cdot a$}
	\Let{$\delta(\delta(u), a)$}{$n$}
	\State{$\Call{splitLeaf}{adt, \delta(ua), n, v}$} 
	\Let{$openTransitions$}{$\{\langle n, i\rangle \vert i \in I\} \cup \{\langle s, i\rangle \vert s \in S, i \in I, \delta(s, i) = \delta(ua)\}$}
	\State \Call{closeTransitions}{}
	\State \Return{\True}
	\EndFunction
\end{algorithmic}

%% file: alg/adt_close.tex
\begin{algorithmic}[1]
	\Function{closeTransitions}{}
	\While{\Not $openTransitions.isEmpty()$}
	\Let{$t$}{$openTransitions.pop()$}
	\State{$\Call{closeTransition}{t}$}
	\EndWhile
	\EndFunction
	\Statex
	\Function{closeTransition}{$t = \langle s, i \rangle$}
	\Let{$as$}{$accessSequences[s]$}
	\State{$sqo.reset()$}
	\For{$i = 1 ~\textbf{to}~ \vert as \vert$}
	\State{$sqo.query(as_i)$}
	\EndFor
	\Let{$\lambda(s, i)$}{$sqo.query(i)$} 
	\Let{$lp$}{$as \cdot i$}
	\Let{$leaf$}{$adt.sift(lp)$}
	\If{$leaf.reference = \Nil$} \label{lst:adt_newequiv} \Comment{New equiv. class}
	\Let{$n$}{$\vert S \vert$}
	\Let{$S$}{$S \cup \{n\}$}
	\Let{$openTransitions$}{$openTransitions \cup ~ \{\langle n, i\rangle \vert i \in I\}$}
	\Let{$leaf.reference$}{$n$}
	\Let{$accessSequences[n]$}{$lp$}
	\Let{$\delta(s, i)$}{$n$}
	\Else
	\Let{$\delta(s, i)$}{$leaf.reference$}
	\EndIf
	\EndFunction
\end{algorithmic}

%% file: figures/adt_ex2.tex
\subcaptionbox
{Tentative hypothesis}
[0.6\textwidth]
{
	\begin{tikzpicture}[thick, ->, >=stealth']
		\tikzstyle{edge} = [dashed, decorate, decoration={snake, segment length=8mm}]
		\node[initial, initial text={}, state] (s0) at(0,0) {$s_0$};
		\node[state] (s1) at(4,0) {$s_1$};

		\draw[edge] (s0) -- node[anchor=east] {$water$ / \coffeeok} ++(90:2cm);
		\draw[edge] (s0) -- node[anchor=west] {$pod$ / \coffeeok} ++(50:2cm);
		\draw (s0) -- node[anchor=south] {$button$ / \coffeeerror} (s1);
		\draw[edge] (s0) -- node[anchor=north east] {$clean$ / \coffeeok} ++(-70:2cm);

		\draw[edge] (s1) -- node[anchor=south east] {$water$ / ?} ++(70:2cm);
		\draw[edge] (s1) -- node[anchor=north west] {$pod$ / ?} ++(30:2cm);
		\draw[edge] (s1) -- node[anchor=south west] {$button$ / ?} ++(-30:2cm);
		\draw[edge] (s1) -- node[anchor=north east] {$clean$ / ?} ++(-70:2cm);
	\end{tikzpicture}
}%
\subcaptionbox
{Adaptive discrimination tree}
[0.4\textwidth]
{
	\begin{tikzpicture}[thick, ->, >=stealth']
		\tikzstyle{symbol} = [state, ellipse]
		\tikzstyle{final} = [draw,rectangle, inner sep=8pt]

		\node[symbol] (w) at(0,0) {$water$};
		\node[final] (s0) at(-1, -2) {$s_0$};
		\node[final] (s1) at(1, -2) {$s_1$};

		\draw (w) -- node[anchor=east] {\coffeeok} (s0);
		\draw (w) -- node[anchor=west] {\coffeeerror} (s1);
	\end{tikzpicture}
}

%% file: figures/adt_ex3.tex
\subcaptionbox
{Tentative hypothesis}
[0.5\textwidth]
{
	\begin{tikzpicture}[thick, ->, >=stealth']
		\tikzstyle{edge} = [dashed, decorate, decoration={snake, segment length=8mm}]
		\node[initial, initial text={}, state] (s0) at(0,0) {$s_0$};
		\node[state] (s1) at(4,0) {$s_1$};
		\node[state] (s2) at(0,4) {$s_2$};

		\draw (s0) -- node[anchor=east] {$pod$ / \coffeeok} (s2);
		\draw (s0) -- node[anchor=south] {$button$ / \coffeeerror} (s1);

		\draw[edge] (s0) -- node[anchor=west] {$water$ / \coffeeok} ++(-40:2cm);
		\draw[edge] (s0) -- node[anchor=north east] {$clean$ / \coffeeok} ++(-70:2cm);

		\draw (s1) edge[loop above] node[anchor=south] {$I$ / \coffeeerror} (s1);

		\draw[edge] (s2) -- node[anchor=east] {$water$ / ?} ++(90:2cm);
		\draw[edge] (s2) -- node[anchor=west] {$pod$ / ?} ++(50:2cm);
		\draw[edge] (s2) -- node[anchor=south, xshift=1em] {$button$ / ?} ++(-10:2cm);
		\draw[edge] (s2) -- node[anchor=west] {$clean$ / ?} ++(-40:2cm);

	\end{tikzpicture}
}%
\subcaptionbox
{Adaptive discrimination tree}
[0.5\textwidth]
{
	\begin{tikzpicture}[thick, ->, >=stealth']

		\node[symbol] (w1) at(0,0) {$water$};

		\node[final] (s1) at(2, -2) {$s_1$};
		\node[reset] (r) at(-2, -2) {};

		\node[symbol] (w2) at(-2,-4) {$water$};

		\node[symbol] (b) at (-2,-6) {$button$};

		\node[final] (s2) at (-1, -8) {$s_2$};
		\node[final] (s0) at (-3, -8) {$s_0$};

		\draw (w1) -- node[anchor=south west] {\coffeeerror} (s1);
		\draw (w1) -- node[anchor=south east] {\coffeeok} (r);

		\draw (r) -- (w2);

		\draw (w2) -- node[anchor=east] {\coffeeok} (b);

		\draw (b) -- node[anchor=south west] {\coffeecup} (s2);
		\draw (b) -- node[anchor=south east] {\coffeeerror} (s0);
	\end{tikzpicture}
}

%% file: sections/replacements.tex
\chapter{Embedding Adaptive Distinguishing Sequences}
\label{cha:replacements}

The previous chapter presented the base algorithm, which introduced the core concepts of active automata learning in an adaptive environment.
However, the base algorithm does not utilize adaptive distinguishing sequences: 
The extracted discriminators are maintained in their original form and their behavioral information is arranged by reset nodes.
Therefore, this chapter presents the main approach by which adaptive distinguishing sequences will be integrated into the learning process.
It describes what influence the usage of ADSs has on the learning process and shows what steps are necessary to successfully benefit from their potential.

\section{Subtree Replacements}

The core concept to integrate adaptive distinguishing sequences is by replacing nodes -- specifically subtrees -- of the adaptive discrimination tree.
During the default execution of the base algorithm, the adaptive discrimination tree may aggregate a considerable amount of reset nodes.
Replacing subtrees with many reset nodes or even the complete adaptive discrimination tree with a single, reset-free adaptive distinguishing sequence, reduces the number of reset queries of succeeding refinement steps, as the number of encountered reset nodes during the sifting operations is reduced.

\subsection{Replacement Validation}

Different replacements may be considered at different stages of the learning algorithm (cf. \Cref{cha:heuristics}), but all replacements have a common characteristic:
For the construction of the adaptive distinguishing sequence, the current tentative hypothesis is used as a reference point for the behavior of the target system.
However, until termination, the hypothesis only approximates the behavior of the target system, meaning that any extracted information may not be valid.
To ensure that the adaptive discrimination tree is still verified after the replacement occurred and therefore guarantee the correctness and termination of the learning process, the proposed replacements need to be validated.
The main validation process is depicted in \Cref{alg:repl_validate}.

\begin{algorithm}[!t]
	\caption{ADS Replacements: Replacement validation \label{alg:repl_validate}}
	\input{alg/repl_validation.tex}
\end{algorithm}

The \textsc{validate} function receives three input parameters: $ntr$, the node of the current adaptive discrimination tree to be replaced; $repl$, the (start of the) adaptive distinguishing sequence, that is proposed to replace $ntr$; and $cutout$, a set of states, that may not be covered by the replacement, but are referenced in the subtree of $ntr$.
For the parameters, basic sanity properties are assumed, for example that all hypothesis states referenced in the subtree of $ntr$ are covered by the union of $repl$'s leaves and $cutout$.

Intuitively, in order to verify the replacement, one has to assure that the predicted input/output behavior defined by the replacement matches the real input/output behavior of the target system.
Therefore, the main loop of the function iterates over every referenced state of the replacement and verifies the suggested input/output trace.
The \textsc{collectLeaves} function collects every final node of the replacement, whereas \textit{trace} works as defined in \Cref{def:traces}.
The function continues to transition the target system into the state whose output behavior should be validated, by applying the access sequence of the current state and the potential input trace of the parent node.

Note, that there were no restrictions on the parameter $ntr$:
Certain replacements may aim at replacing reset nodes of the current adaptive discrimination tree.
In this case, the proposed replacement is essentially a continuation of an existing discriminator.
As a result, the complete trace is required to transition the target system in its correct state.
If a replacement seeks to replace a complete subtree (i.e. the parent of $ntr$ is a reset node) the parent trace will be $\varepsilon$.

The next step is to verify the behavior of the proposed replacement.
The input trace is sequentially applied to the target system and the observed output is compared with the expected output given by the adaptive distinguishing sequence.
If the outputs differ, a counterexample is encountered, because the replacement was computed based on the behavior of the tentative hypothesis, which in the case of a mismatch is provably wrong.

However, a potential mismatch between expected and real behavior does not necessarily result in a failure of the validation process.
As long as the observed output behavior distinguishes all states uniquely it is still possible to construct a correctly classifying adaptive distinguishing sequence.
Therefore, independent of the verification result, the function continues to construct a single input/output trace from the observed behavior by calling the \textsc{buildADS} subroutine.
If a result from previous loop iterations exists, the function tries to merge the current trace with the existing distinguishing sequence by a call to the \textsc{mergeADS} subroutine.
This subroutine simultaneously traverses the existing distinguishing sequence ($result$) and the trace to merge ($trace$) by means of the input sequence of $trace$.
If at one point -- under the maintenance of a shared input sequence -- diverging output behavior is observed, the remaining trace of $trace$ can be appended to the corresponding node of the existing distinguishing sequence $result$.

The attempt to merge the two traces may however fail, in which case the subroutine returns $false$.
An example for this situation is given by the two traces $\langle 12, ab \rangle$ and $\langle 1, a \rangle$:
The leaf of the second trace coincides with the symbol node $2$ of the first trace.
For the behavioral information available at this point, the two nodes are not distinguishable by a single adaptive distinguishing sequence.
In these situations, the incurred ambiguities can be resolved by consulting the current adaptive discrimination tree, as realized by the \textsc{resolveAmbiguities} routine displayed in \Cref{alg:repl_ambiguities}.

Aside from the validation of the proposed replacement, the states left out also need to be covered by a valid substitution.
Since for these states, no behavioral information is given by the proposed replacement, only the existing adaptive discrimination tree can be used to distinguish them.
If all validation steps succeed, the final result can be used to substitute the $ntr$ node in the adaptive discrimination tree, while maintaining its property of being verified.

\begin{algorithm}[t]
	\caption{ADS Replacements: Resolving Ambiguities \label{alg:repl_ambiguities}}
	\input{alg/repl_ambiguities.tex}
\end{algorithm}

For resolving the encountered ambiguities, recall the current situation and the meaning of the parameters:
The ultimate goal is to find a verified replacement for the node to replace ($ntr$).
A mandatory adaptive discrimination tree ($repl$) already exists\footnote{In the first iteration, this parameter references a single adaptive distinguishing sequence.
However, subsequent executions may add additional reset nodes.}, yet the obtained information for the hypothesis state $s$ did not suffice to distinguish it from the states referenced in the leaves of $repl$.

Crucial to resolving the ambiguities is the node to replace, because it decides whether just the state $s$ or the state \textit{$s$ after applying a certain input sequence} needs to be distinguished from the other states.
As a result, the target system first needs to be transitioned into the state accessed by the access sequence of $s$ concatenated with the potential parent input trace.

The function may then continue to sift the respective hypothesis state into the provided adaptive discrimination tree.
The subroutine may either consult the current adaptive discrimination tree or use the symbol query oracle to perform the sift-operation.
Note, that for the sifting operation the information about $s$ suffice.
The target system was transitioned to the correct state before the call to the sift function.
If during the sift operation any reset nodes are encountered, the information about the original parent trace $pi$ become irrelevant and the target system only needs to be transitioned into the state accessed by the access sequence of $s$.

If this operation discovers unexpected behavior, the provided ADS $repl$ may be extended by a corresponding leaf referencing $s$.
If, however, the state is sifted through the complete tree (and the subroutine therefore returns a conflicting final node), the conflicting hypothesis states can be distinguished by the input trace and the symbol of their lowest common ancestor in the current adaptive discrimination tree.
The update operation to split the leaf node follows the same semantics as in the hypothesis refinement step.

A consequence of the extensive error handling done by the \textsc{resolveAmbiguities} function is, that the verified replacement may \enquote{degenerate} to an adaptive discrimination tree.
In certain situations, the actual replacement of the intended node may therefore not necessarily improve the structure of the ADT and the replacement should be reconsidered.
However, this effect may raise the question, if one can directly propose a replacement in form of an adaptive discrimination tree.
One may think of scenarios, where replacing a (possibly larger) subtree with an improved adaptive distinguishing (sub-)tree still reduces the number of total reset nodes.

The concept of replacing subtrees of the current adaptive distinguishing tree may be generalized to arbitrary (adaptive) replacements.
To allow this generalization, the validation procedure would have to validate \emph{all} traces of each final node of the replacement.
This increases the impact of validation errors, because not only the ambiguity of single final nodes but potentially whole subtrees may need to get resolved.
At the same time, the search space for potential replacements grows.
While certain replacement computations may improve because of this -- a reset may be used as a shortcut for distinguishing states -- other computations may experience an increase in complexity -- minimal-size ADS versus minimal-size ADT.

As for the scope of this thesis, this and the remaining chapters will focus on its central theme, the embedding of adaptive distinguishing \emph{sequences}.

\subsection{Hypothesis Update}

If the validation of one or possibly multiple replacements succeeded, they may replace their targeted nodes in the adaptive discrimination tree.
However, there is still the need for a post-processing step:
Transitions that led into states referenced in the leaves of the replaced nodes need to be re-sifted through the new replacement.
This procedure is depicted in \Cref{alg:repl_resift}.

\begin{algorithm}[t]
	\caption{ADS Replacements: Re-sifting transitions \label{alg:repl_resift}}
	\input{alg/repl_resift.tex}
\end{algorithm}

The successor of a hypothesis state was determined by the behavior of the target system under certain discriminators.
If these discriminators change, the target system may exhibit different behavior as well, especially if the final hypothesis is not yet reached.
Because of that, a different successor may be selected and hence all incoming transitions need to be reevaluated against the new discriminators.
If this step is omitted, there exists a discrepancy between the output behavior of the current hypothesis states and future hypothesis states, whose successors will be determined by the new discriminators.
It may happen that the counterexample decomposition issues a refinement of a state that is already represented in the hypothesis, which results in the final hypothesis not being minimal.

At first, it might seem sufficient to update the transitions on a per-state granularity.
If a transition is found, that previously led into state $s_1$ and now -- under the new discriminators -- leads into state $s_2$, all remaining transitions that led into state $s_1$ may also be updated to target state $s_2$.
This is however not sufficient, because the learner has no information about which of the states of the tentative hypothesis already represent singleton partitions of the true equivalence classes and which do not.
This approach may only be valid, if the tentative hypothesis is already isomorphic to the target system, which however renders any subtree replacements redundant as the final hypothesis is already computed.
It may in fact happen, that two transitions that previously led into the same state, may lead into two different states under the new discriminators, which renders the individual investigation of every single transition necessary.

Since the closing of transitions may discover new hypothesis states, all replacements are scheduled before the call to \textsc{closeTransitions}.
Otherwise, a replacement may discard hypothesis states, that have been discovered in previous re-sift operations.

\subsection{Counterexample Reactivation}

Besides the additional efforts for validating a replacement and updating the tentative hypothesis, replacing discriminators has another effect on the learning process.
A special focus should be denoted to the semantics of the discriminators:
From an algorithmic point of view, the nodes of the adaptive discrimination tree -- and therefore the discriminators -- serve as decision points when sifting a word through the tree and determining the final node.
However, from a semantic point of view, one should recall that discriminators also define the output behavior of a state after applying the input sequence represented by the discriminator.

In the base algorithm, once a counterexample is decomposed, the extracted distinguishing suffix $v$ is integrated into the adaptive discrimination tree and each hypothesis state referenced in the leaves shows a certain behavior corresponding to the subtree it is contained in.
Since only final nodes are split in a refinement step, the information (or classification) obtained by $v$ is preserved throughout the learning process. 
However, by replacing discriminators (e.g. $v$), the information about $v$ and the expected output behavior upon receiving $v$ as an input sequence is discarded.
Applying the counterexample that yielded the discriminator $v$ to the updated, tentative hypothesis (cf. \Cref{alg:repl_resift}) may result in diverging output behavior again.
This means, counterexamples that have already been used to refine the hypothesis, may become valid counterexamples again.
While this effect does not affect the execution of the learning algorithm, it breaks the consistency of the hypothesis with previous observations from counterexamples.

To handle this problem and integrate the general support for discriminator replacements, the base algorithm is extended to cache every counterexample in a global set of counterexamples.
The \textsc{refineHypothesis} function, after retrieving the current counterexample from the queue of open counterexamples is altered to add the current counterexample to the global cache.
Additionally, after the initial refinement succeeded, every entry of the global cache is reevaluated again.
If thereby a valid counterexample is encountered, a further refinement step is issued.

%% file: alg/repl_validation.tex
\begin{algorithmic}[1]
	\Function{validate}{$ntr, repl, cutout$}
	\Let{$\langle pi, po \rangle$}{$trace_{adt}(ntr)$}
	\Let{$result$}{\Nil}
	\ForAll{$f \in \Call{collectLeaves}{repl}$}
	\Let{$\langle i, o \rangle$}{$trace_{repl}(f)$}
	\Let{$as$}{$accessSequences[f.reference]$}
	\State{$sqo.reset()$}
	\State{$sqo.query(as \cdot pi)$}
	\Let{$equal$}{\True}
	\Let{$k$}{$1$}
	\Let{$output$}{$\varepsilon$}
	\While{$equal \And k \leq \vert i \vert$}
	\Let{$output$}{$output \cdot sqo.query(i_k)$}
	\If{$output_k \neq o_k$}
	\Let{$equal$}{\False}
	\Else
	\Let{$k$}{$k + 1$}
	\EndIf
	\EndWhile
	\If{$\Not ~equal$}
	\State{$openCounterExamples.add(\langle as \cdot pi \cdot i_{1:k}, \lambda(as) \cdot po \cdot output \rangle)$}
	\EndIf
	\Let{$trace$}{$\Call{buildADS}{i_{1:k}, output, f.reference}$}
	\If{$result = \Nil$}
	\Let{$result$}{$trace$}
	\Else
	\If{$\Not~ \Call{mergeADS}{result, trace}$}
	\State{$\Call{resolveAmbiguities}{ntr, result, f.reference}$}
	\EndIf
	\EndIf
	\EndFor
	\ForAll{$c \in cutout$}
	\State{$\Call{resolveAmbiguities}{ntr, result, c}$}
	\EndFor
	\State \Return $result$
	\EndFunction
\end{algorithmic}

%% file: alg/repl_ambiguities.tex
\begin{algorithmic}[1]
	\Function{resolveAmbiguities}{$ntr, repl, s$}
	\Let{$\langle pi, po \rangle$}{$trace_{adt}(ntr)$}
	\State{$sqo.reset()$}
	\State{$sqo.query(accessSequences[s] \cdot pi)$}
	\Let{$leaf$}{$\Call{siftAndReturnConflict}{repl, s}$}
	\If{$leaf = \Nil$}
	\State \Return
	\EndIf
	\Let{$lca$}{$\Call{findLCA}{adt, leaf.reference, s}$}
	\Let{$\langle lcai, lcao \rangle$}{$trace_{adt}(lca)$}
	\State{$\Call{splitLeaf}{repl, leaf.reference, s, lcai \cdot lca.symbol}$}
	\State \Return
	\EndFunction
\end{algorithmic}

%% file: alg/repl_resift.tex
\begin{algorithmic}[1]
	\Function{resift}{$replacements$}
	\ForAll{$repl \in replacements$}
	\ForAll{$f \in \Call{collectLeaves}{repl}$}
	\Let{$openTransitions$}{$openTransitions ~\cup$}
	\Statex \hskip5\dimexpr\algorithmicindent $\{ \langle s, i \rangle \vert s \in S, i \in I, \delta(s, i) = f.reference\}$
	\EndFor
	\EndFor
	\State \Call{closeTransitions}{}
	\EndFunction
\end{algorithmic}

%% file: sections/heuristics.tex
\chapter{Replacement Heuristics}
\label{cha:heuristics}

With the previous chapter introducing the necessary mechanics to successfully integrate adaptive distinguishing sequences into the learning process, this chapters continues the idea by presenting several replacement heuristics.
This will cover heuristics, that are solely possible due to the adaptive environment in which the learning process takes place (cf. \Cref{sec:subext}); heuristics that follow the \enquote{classic} approach of replacing subtrees of the adaptive discrimination tree (cf. \Cref{sec:subtreeRepl,sec:immediateRepl}) and data structures that gain new utility because of the altered learning setup (cf. \Cref{sec:ot}).

\section{Subtree Extensions}
\label{sec:subext}

The idea of extending subtrees is a heuristic that does not directly involve the computation of an adaptive distinguishing sequence.
It does, however, make use of the fact that the adaptive environment in which the learning process takes place, provides access to complete behavioral traces.
Recall from \Cref{alg:adtclose} that during the sift operation, unexpected behavior may occur, which leads to the discovery of a new equivalence class.
From a discriminating point of view, a new discriminator is found, that is a prefix of an existing discriminator.
As a result, a minimum of three -- at least two final nodes of the original subtree and the one discovered during the sift operation -- equivalence classes can be distinguished with a single, reset-free discriminating sequence.

The idea of subtree extensions applies the same effect in reverse.
Whenever a final node $f$ of the current adaptive discrimination tree is split in a refinement step, it may happen that the newly obtained discriminator $v$ is a continuation of the input sequence of $trace_{adt}(f)$.
So instead of adding a reset-node succeeded by the complete trace of $v$ to the adaptive discrimination tree, it is sufficient to append the corresponding suffix of $v$ to the ADT, which allows to save an unnecessary reset-node.
An example of the heuristic is shown in \Cref{fig:heu_extension}, which compares the two resulting adaptive distinguishing trees after the second refinement step discussed in \Cref{sec:aalex}.

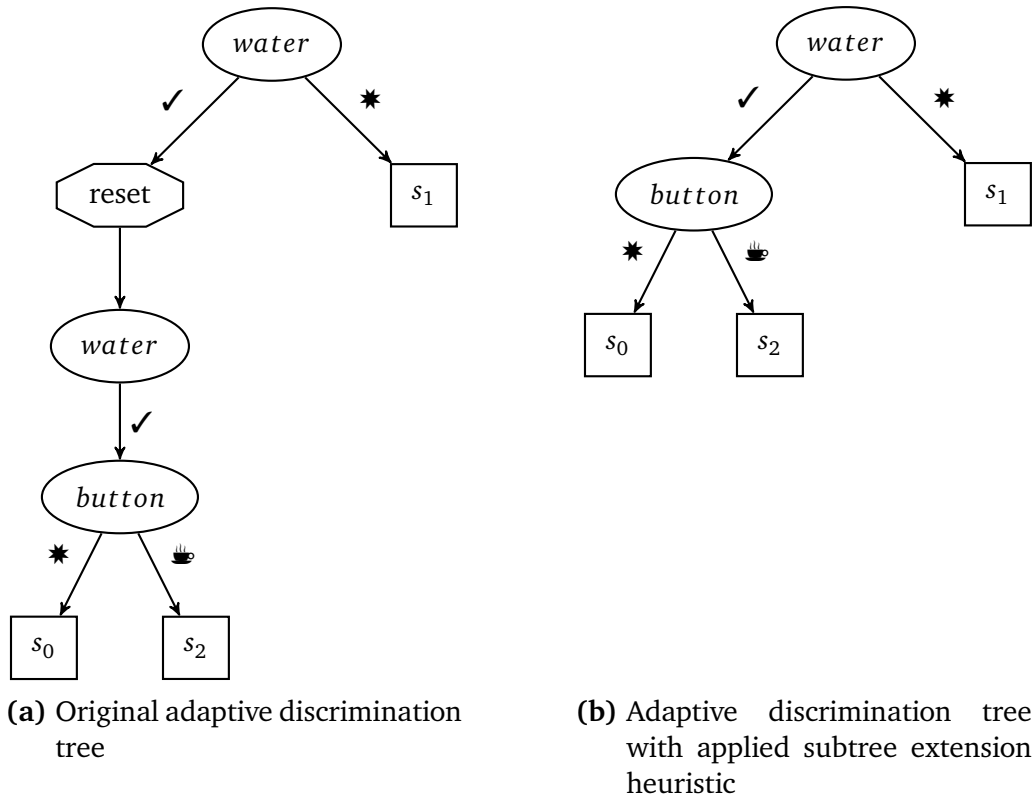
\begin{figure}[ht]
	\centering
	\input{figures/repl_ext.tex}
	\caption{Subtree Extension Heuristic}
	\label{fig:heu_extension}
\end{figure}

Although the original subtree containing the complete discriminator of the second counterexample ($water \cdot button$) is replaced by the corresponding suffix ($button$), these kind of replacements differ from to the ones discussed in \Cref{cha:replacements}.
It is clear to see, that this heuristic does not need any additional steps to maintain the properties of termination and correctness of the learning algorithm.
The behavioral information stored in both adaptive discrimination trees is identical, as the second one is merely a more compact representation of the discovered knowledge.
While this heuristic is therefore optimal with regard to the required verification costs, its applicability highly depends on the decomposed discriminator -- a factor the learning algorithm usually has no control over.
A more self-determined way for saving reset nodes is pursued by the initial idea of active subtree replacements, as discussed in the next section.

\section{Subtree Replacements}
\label{sec:subtreeRepl}

This section presents heuristics that actively seek to replace nodes of the current adaptive discrimination tree by computing adaptive distinguishing sequences based on the current tentative hypothesis.
However, a preliminary question to ask is: \enquote{When do these replacements take place?}

In the developed approach, subtree replacements are preformed prior to internal refinement steps.
That is, if the \textsc{refineHypothesis} procedure (cf. \Cref{alg:adtref}) receives a valid counterexample, one of the heuristics of the following subsections is applied. 
Afterwards, the main loop of iterating over the queue of open counterexample is executed.
While this decision is motivated by the intent to reduce the number of reset nodes for succeeding refinement steps, one may question the state of the current hypothesis.

Given that the encountered counterexample is valid, the behavior exposed by the tentative hypothesis is provably wrong.
Therefore it might be questionable to consult the current hypothesis to compute adaptive distinguishing sequences, because the utilized behavior may not hold in the real target system.
However, a similar situation is encountered, when performing replacements after the hypothesis update.
If the updated hypothesis is assumed to be correct, subtree replacements can be considered redundant, as the final hypothesis is reached and no further refinement steps occur.
If one performs subtree replacements for \emph{future} refinement steps, it is indirectly admitted, that even the updated hypothesis does not exhibit correct behavior either.

Being a heuristic after all, the presented replacement strategies aim utilizing the structural information at hand.
In combination with the gracious validation process, that allows to recover from potentially failed validations, the computed replacements may therefore still pose improvements to the active learning process.

Regarding the validation process, it was mentioned that in the case of a failed validation, the validated replacement may contain reset nodes and therefore not necessarily reduce the number of reset nodes in a subtree.
The presented heuristics therefore calculate an \emph{effective reset count} for both the subtree to replace and the replacement.
That is, the accumulated sum of all reset nodes on $path_{\mathcal{T}}(f)$ over all leaves $f$ of a tree $\mathcal{T}$. 
If the effective reset count of a replacement is higher or equal to the value of the original subtree, the replacement is discarded.

The following three subsections continue to present the elaborated replacement heuristics.

\subsection{Leveled Replacements}
\label{sec:levelrepl}

The leveled replacement heuristic is a greedy replacement strategy that seeks to replace subtrees with reset nodes whenever possible.
Its core approach can be summarized as a breadth-first traversal of the current adaptive discrimination tree and the attempt to replace any encountered reset node.
The approach is formalized in \Cref{alg:leveledreplace}.

\begin{algorithm}[t]
	\caption{Replacement Heuristic: Leveled Replacement Heuristic \label{alg:leveledreplace}}
	\input{alg/heu_leveled.tex}
\end{algorithm}

The function starts with initializing the set of proposed replacements with the empty set and the queue, that is used for the breadth-first iteration, with the root node of the current adaptive discrimination tree.
The main loop then picks a node from the queue and tries to find an extension of a potential parent trace.
The computation of such an extension is formalized in \Cref{alg:extCalc}.
If such an extension is found, it is scheduled as a replacement for the parent (reset-)node and the loop continues to process its remaining elements.
Note, that this heuristic always computes replacements covering all affected nodes -- the potential cutout is therefore always defined by the empty set.
If no such extension is found, it is first checked, if the current sub-tree still contains reset nodes that could potentially be saved by a replacement.
If this is not the case, the loop continues with the investigation of the remaining nodes.
Otherwise, the calculation of an adaptive distinguishing sequence for the nodes covered by the current subtree is issued.
If such a sequence is found, it is proposed as a replacement for the current node and added to the global set of proposed replacements.
If, however, neither an extension nor a replacement could be found, the child-ADTs (i.e. the first child-nodes that succeed a reset node) are added to the queue and investigated in the following iterations of the loop.
Ultimately, the set of all proposed replacements is returned.

Regarding the decision to stop further investigations, it is noteworthy, that subtrees without reset nodes model the stop criterion.
Situations may occur, where the computation of an adaptive distinguishing sequence may still improve the learning process, e.g. if the current subtree has great length and the hypothesis may yield an ADS that can distinguish the states by means of a shorter sequence.
However, with the mindset of saving resets, these replacements would only introduce additional verification costs without improving the structure of the adaptive discrimination tree reset-wise.
Hence, no further computations are issued.

Regarding the result of the heuristic, it should be noted that a set and therefore potentially multiple replacements are returned.
Given that the original data-structure, i.e. the current adaptive discrimination tree, follows a tree structure and its nodes are traversed from top to bottom, it is clear that only distinct sets of nodes are investigated, which ensures that there exist no collisions between two replacements.

\begin{algorithm}[t]
	\caption{Replacement Heuristic: Extension Calculation \label{alg:extCalc}}
	\input{alg/heu_extend.tex}
\end{algorithm}

For the computation of potential extensions, it is first checked, if a parent node exists.
This is not the case only if the root node of the current adaptive discrimination tree is passed as parameter.
Hence, unless the current adaptive discrimination tree is already free of reset nodes, \Cref{alg:leveledreplace} always tries to compute an adaptive distinguishing sequence for the complete hypothesis first.

If a parent node -- and therefore an input sequence that can be extended -- exists, its trace in the current ADT is extracted.
Recall that for extensions, one must not consider the target states directly, but the successor states after applying the input sequence of the parent trace.
One may however, similar to the computation of adaptive distinguishing sequences, encounter the problem of converging states:
If for an input symbol, two states emit the same output symbol and transition into the same target state, they cannot be distinguished anymore.

Note, that the convergence test does not explicitly check the output of the hypothesis.
The definition of the \textsc{refineHypothesis} procedure (cf. \Cref{alg:adtref}) ensures that after the termination of a refinement step, the hypothesis is consistent with the behavioral information stored in the current adaptive discrimination tree.
Hence the outputs of all target states correspond to the outputs defined in the parent (output-) trace.
If converging states are detected, the computation of an extension aborts.
Otherwise a mapping is created, which stores the original states of the new current states.

Using the new set of current states, it is then attempted to compute an adaptive distinguishing sequence.
If such a sequence exists, its leaves however, reference the states of the new current set.
As a result, an additional post-processing step is necessary, that updates all referenced hypothesis states to their original initial states.
Afterwards, the computed extension is returned for replacing the parent (reset-)node.

\subsection{Exhaustive Replacements}
\label{sec:exhausrepl}

A potential drawback of the leveled replacement approach is the passive behavior if an adaptive distinguishing sequence does not exist for a certain subtree.
Especially in situations, where discriminators have a big fan-out with many subsequent reset nodes, the absence of an adaptive distinguishing sequence may retain these reset nodes, as the heuristic continues to search for replacements for each child subtree separately.
Therefore, in many situations, the capabilities of an (partial) adaptive distinguishing sequence may not be used to its full potential.
An approach to tackle this situation is described by the exhaustive replacement heuristic, which is formalized in \Cref{alg:repleager}.

\begin{algorithm}[t]
	\caption{Replacement Heuristic: Exhaustive Replacement Heuristic}
	\label{alg:repleager}
	\input{alg/heu_exhaus.tex}

\end{algorithm}

The main approach of this heuristic is to always compute a replacement for the root node of the current adaptive discrimination tree.
As a result, to not compute a replacement if the current adaptive discrimination tree is already optimal with respect to the amount of reset nodes, a redundancy check is modeled explicitly.
Core to the execution of this heuristic is the realization of the \textsc{computeCutout} method, which returns a sequence of cutouts -- a sequence of sets of nodes to remove from the set of all hypothesis states.

For instance, this method may be realized by enumerating over the elements of the powerset $2^S$, in ascending order of their size.
This would allow to find an adaptive distinguishing sequence that would cover the maximum number of nodes possible.
However, it is easy to see that the exponential nature of the powerset would dominate the runtime even for remotely sized systems.

A more sophisticated approach may be achieved by the following idea:
For the initial computation of an adaptive distinguishing sequence for the complete hypothesis, the algorithm of Lee and Yannakakis is used.
The algorithm can be slightly modified, so that instead of returning \textit{nil}, the set of indistinguishable states is returned.
This \enquote{problem-oriented} set may then be used to propose reasonable cutouts.
However, during the development of this heuristic, several tests showed that it is often already the initial partition of states for which finding a splitting input fails.
Therefore this approach would often degenerate to the classic powerset scenario.

Ultimately, this method was realized by proposing cutouts based on the structure of the current adaptive discrimination tree.
For all possible sub-ADTs, the set of hypothesis states referenced in their leaves are constructed.
These sets are then (ascending in their size) proposed as cutouts.

\subsection{Single Replacements}

The previous two heuristics aim at replacing subtrees with reset nodes whenever possible.
While their greedy approach contributes to the reduction of reset nodes, it neglects the fact that each proposed replacement results in validation costs and -- in case of a successful validation -- requires an update of the tentative hypothesis.
A more passive approach is pursued by the single replacement heuristic, which -- as the name suggest -- only proposes a single replacement.
The heuristic is formalized  in \Cref{alg:replsingle}.

\begin{algorithm}[t]
	\caption{Replacement Heuristic: Single Replacement Heuristic}
	\label{alg:replsingle}
	\input{alg/heu_single.tex}
\end{algorithm}

The heuristic starts by collecting all (sub-) ADTs of the root node, i.e. all nodes, that succeed a reset node.
Recall, if the current adaptive discrimination tree contains no reset nodes, this method returns the empty set and therefore no replacements are proposed.
It then sorts the subtrees according to the \emph{reset/final score} in descending order.
The $rf$ score is defined as follows:

$$rf(n) = \dfrac{1 + \vert \textsc{collectResetNodes}(n) \vert}{\vert \textsc{collectLeaves}(n) \vert}$$

The motivation behind this score is, that it values subtrees with a high amount of reset nodes and a low amount of final nodes.
Replacements for these trees improve the structure of the resulting adaptive discrimination tree while coming at relatively low validation costs.
The additional $+1$ comes from the fact, that the collection of (sub-) reset nodes misses the additional reset node given by the parent of the current subtree's root.

Iterating over the sorted subtrees, the search for a replacement follows the structure of the leveled replacement heuristic.
At first, it is attempted to compute an extension for the parent trace of the currently inspected node.
If this attempt is successful, a singleton set containing the replacement is returned.
Otherwise, if the subtree still contains reset nodes, it is checked if the hypothesis states referenced in the current subtree can be distinguished by means of a single adaptive distinguishing sequence.
Again, if such a sequence exists, it is returned as the single proposed replacement.
Otherwise, the subtree with the next lower $rf$ score is investigated.
If no valuable distinguishing sequences have been found, the heuristic proposes no replacement.

\section{Immediate Replacements}
\label{sec:immediateRepl}

A trait the previously presented replacement heuristics have in common, is the clear separation from the internal procedures of the learning algorithm.
The replacements and validations take place at a distinct point of time, leaving procedures such as \textsc{refineHypothesisInternal} (cf. \Cref{alg:adtref}) essentially atomic.
However the approach of utilizing adaptive distinguishing sequences may also be applied in a more fine-grained manner as realized by the immediate replacement heuristic.

The heuristic intervenes the internal refinement procedure (cf. \Cref{alg:adtref}) after the states of the hypothesis and the final nodes of the adaptive discrimination tree have been refined, but before the open transitions are closed.
The key idea of the heuristic revolves around the distinction between temporary and finalized discriminators:
In the base algorithm, the counterexample decomposition yields a distinguishing suffix $v$ which is directly integrated into the adaptive discrimination tree by replacing the final node referencing the too coarse equivalence class with a reset node followed by the sequence of symbol nodes resembling $v$.
The immediate replacement heuristic however seeks for an extension to the previous trace\footnote{Note, that this heuristic is not applicable for the first discriminator, as no previous trace exists yet.} (i.e. $trace_{adt}(f')$ for the too coarse final node $f'$ in the former adaptive discrimination tree), that distinguishes the newly discovered equivalence classes.
It is similar to the approach described in \Cref{sec:subext}, although it does not rely on the previous trace being a prefix of $v$, since the heuristic actively computes a potential extension based on the hypothesis.
To determine such an extension, it will use the original discriminator $v$ as a temporary discriminator for only a small amount of transitions and proposes, in case of success, a finalized discriminator that does not use an additional reset node.
The approach is formalized in \Cref{alg:replimmed}.

\begin{algorithm}[tp]
	\caption{Replacement Heuristic: Immediate Replacements}
	\label{alg:replimmed}
	\input{alg/heu_immed.tex}
\end{algorithm}

The heuristic receives the subtree that has been added to the adaptive discrimination tree in the preceding refinement step and resembles the discriminator $v$ as an input parameter.
The work of the heuristic is embedded in a potential infinite loop, whose explanation will follow shortly.
Within said loop, the heuristic starts with initializing a set of auxiliary variables:
First, the final nodes of the temporary discriminator are collected.
In the first loop iteration, the set will contain the two final nodes referencing the recently added hypothesis state as well as the hypothesis state to be refined.
For future reference, these two nodes may be referred to as $f$ and $f'$.
Second, the trace that leads into the temporary discriminator is computed by invoking the trace function on the reset node that precedes the given subtree.
Third, a mapping of hypothesis states is defined that allows to keep track of the initial and current set of states by storing array entries in the form of $mapping[current] = initial$.
It is initialized with the identical mapping.

The heuristic continues with reapplying the input sequence of the parent trace to the current set of nodes.
Therefore it iterates over all affected states and first checks if the transition for the current state and the current input symbol is well-defined.
This check is required, because for certain states -- one may think of the most recently added state -- the outgoing transitions have not yet been defined or -- in the case of transitions that led into the refined state -- need to be refined.
If necessary, the transition in context will be closed by a call to the known \textsc{closeTransition} procedure (cf. \Cref{alg:adtclose}).

The resulting sift operations (of the \textsc{closeTransition} invocation) may at one point come to a situation, where the target state of the passed transition is either $f$ or $f'$.
Although the current execution of the replacement heuristic is still seeking for an input sequence that distinguishes the two states, there already exists a distinguishing sequence: the current (temporary) discriminator $v$.
The motivation behind calling $v$ a \emph{temporary} discriminator is, that $v$ will only be used to distinguish between $f$ and $f'$ in these specific cases, unlike previous scenarios, where $v$ would have been used to close all transitions in general.
Only if the heuristic does not manage to provide a valid extension for the previous trace, $v$ will also be retained as a finalized discriminator.

Additionally, the sifting operation may hold another exceptional behavior.
During a sift operation using the current temporary discriminator, another equivalence class may be discovered due to previously unobserved behavior.
The situation that ensues is that the current iteration of the replacement heuristic computes an extension for $m$ final nodes, although the temporary discriminator, which should be replaced, distinguishes $m+1$ states.
To not discard this information, the \textsc{closeTransition} procedure may signal a modification exception that interrupts the further execution of the current replacement computation.
Encountering this situation may however be handled by simply restarting the computation, which explains the infinite loop wrapping the heuristic.
It can however be assured, that the computation does not end in a true infinite loop, since only a finite amount of modification exceptions can be raised, as the target system is assumed to be finite.

Continuing the computation of the replacement, the heuristic first checks if the output of the currently iterated state matches the expected output of the parent trace.
Since the hypothesis has just been refined, the consistency with the behavioral information of the adaptive discrimination tree cannot be guaranteed.
If the outputs differ, a counterexample is logged and the computation of an immediate replacement is aborted by returning the temporary discriminator as the final one.
Otherwise, it is checked if the successor of the current state is already reached by another state.
This tackles the same problem as mentioned in \Cref{sec:levelrepl}:
If two distinct states produce the same output and transition into the same target state, no further input sequence is able to distinguish the two states.
If all validation checks pass, the mapping from the current to initial states is updated and the next input symbol is investigated.

Once the set of current nodes is determined, the computation of an adaptive distinguishing sequence is attempted.
Note, due to the scenario described above, a slightly modified version of the ADS-computation procedure is invoked, that respects potentially undefined transitions.
At first, undefined transitions are skipped, since determining their output/successor requires an additional sifting operation.
If, however, for the current hypothesis no adaptive distinguishing sequence can be found, the undefined transitions are selectively closed and the attempt to find an adaptive distinguishing sequence is repeated.
A more detailed explanation of the defensive calculation approach is presented in \Cref{cha:ads}.

If an adaptive distinguishing sequence is found, it is still a final post-processing step required.
The adaptive distinguishing sequence was computed to distinguish between the states reached after applying the input sequence of the parent trace.
In order to return a valid discriminator for the initial states, the referenced hypothesis state of every final node of the computed adaptive distinguishing sequence is updated using the previously computed mapping array.
The updated ADS is then returned as the finalized discriminator.

After the termination of the \textsc{computeFinalDiscriminator} procedure, the returned (finalized) discriminator $fd$ is investigated.
If the result is equal to the temporary discriminator (i.e. the finalization procedure was not able to compute an extension), no additional steps are necessary, as the temporary discriminator is already part of the adaptive discrimination tree.
The refinement step may finish by closing the remaining open transitions.
Otherwise, the replacement $\langle tempDiscr.parent, fd, \emptyset \rangle$ is issued.
Similar to the subtree replacement heuristic, the validated (finalized) discriminator is discarded, if its effective reset count does not improve the effective reset count of the temporary discriminator.
Afterwards, for closing the remaining open transitions, the updated adaptive discrimination tree may be utilized.

\subsection{Example}
\label{sec:imreplex}

To summarize the functionality of the immediate replacement heuristic and give an example of the potential savings the heuristic may offer, this paragraph shows an exemplary execution for the coffee machine use case.
Therefore, recall the situation depicted in \Cref{fig:learning_ex3}:
The decomposed counterexample resulted in the creation of state $s_2$ as a refinement of state $s_0$ and yielded the discriminator $water \cdot button$ to distinguish between the two states.
In order to complete the refinement step, the remaining open transitions need to be closed.

Contrary to the base algorithm, which uses the adaptive discrimination tree as-is, the immediate replacement heuristic seeks for an adaptive distinguishing sequence that extends the previous trace to distinguish between $s_2$ and $s_0$.
To do so, the heuristic first applies the previous (input) trace $water$ to both states $s_2$ and $s_0$.
However, due to the refinement step, both the $\langle s_0, water \rangle$ and $\langle s_2, water \rangle$ transitions are open.
This problem can be circumvented by using the obtained discriminator $water \cdot button$ as a temporary discriminator.
Similar to the base algorithm, the queries $mq(pod \cdot water, water)$ and $mq(pod \cdot water, water \cdot button)$ for the state $s_2$ and $mq(\varepsilon, water \cdot water \cdot button)$ for the state $s_0$ are used to close the required transitions.
The situation is displayed in \Cref{fig:immedrepl}:

\begin{figure}[ht]
	\centering
	\input{figures/heu_immed.tex}
	\caption{Tentative hypothesis and adaptive discrimination tree after decomposing the counterexample}
	\label{fig:immedrepl}
\end{figure}
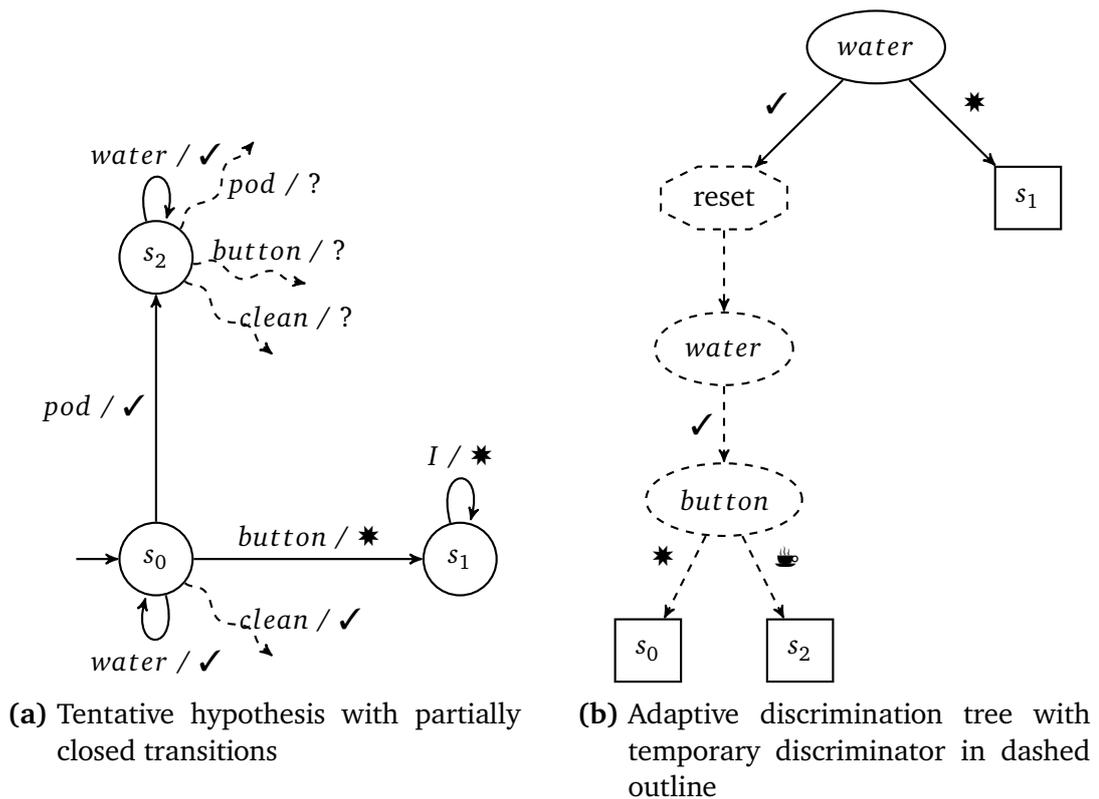

After applying $water$ to both states $s_0$ and $s_2$, the hypothesis remains in the same states, as both transitions are self-loops.
The heuristic continues to defensively compute an adaptive distinguishing sequence for the two states.
The first iteration of this computation however fails:
Upon receiving the input $water$ both states emit the same output symbol (\coffeeok) and remain in the same states, so that $water$ cannot be an adaptive distinguishing sequence.
Moreover all remaining outgoing transitions are undefined for $s_2$.

The defensive ADS calculation continues with closing open transitions using the temporary discriminator.
In the given example it may decide for the $\langle s_2, button \rangle$ transition, because the $button$-successor for $s_0$ is already defined and closed.
For the sake of this example, let us assume that the transition $\langle s_2, button \rangle$ outputs $\coffeecup$ (in reality, the determined system output is $\coffeeerror$ and consequently no adaptive distinguishing sequence based on the hypothesis is found).
Restarting the computation of an adaptive distinguishing sequence, the defensive ADS calculation may now return the input sequence $button$, as the outputs for $s_0$ (\coffeeerror) and $s_2$ (\coffeecup) differ.
Therefore the symbol node referencing the input symbol $button$ is proposed as a replacement for the subtree starting at the reset node.

One can see, that the verification will succeed and ultimately an adaptive discrimination tree similar to the one of \Cref{fig:heu_extension_b} will be constructed.
To finalize the replacement, transitions leading into $s_0$ or $s_2$ need to be re-sifted, as the classification may change under the new discrimination tree.
It is noteworthy, that the remaining transitions, namely $\langle s_2, pod \rangle$ $\langle s_2, clean \rangle$ and $\langle s_0, clean \rangle$ will be closed using the updated adaptive discrimination tree.
For all three transitions, the successor will either be $s_0$ or $s_2$, meaning the replacement computed by the heuristic saved a reset query for each transition.

\section{Observation Tree}
\label{sec:ot}

The previous chapters and sections introduced the concept of subtree replacements and discussed the impact they have on the learning process.
Yet, there are further areas whose potential to improve the learning process increases when exposed to replacements of discriminators.
The component discussed in this section is the \emph{observation tree}.

The observation tree is a secondary automaton that is linked with the symbol query oracle of the learning algorithm and tracks every posed sequence of symbols including the target system's response.
In its core, the observation tree resembles a tree cache for queries.
However, for classic learning algorithms this cache often only aids the learning process, if the currently posed query is a prefix of a previously posed query, i.e. the current query can be cached.
In many situations, the behavioral information the observation tree holds is already represented in the primary data structure (e.g. the discrimination tree) of the learning algorithm.

When replacing discriminators however, the observation tree maintains the behavioral information of the target system that is discarded in the main algorithm.
From the learners perspective, the observation tree gains two beneficial properties:
First, it holds structural information about unrepresented behavior.
This allows to find diverging behavioral information at low costs, because it is present in the form of an automaton and does not require reset or symbol queries.
Second, the information stored in the observation tree is verified, as only traces executed on the target system are stored.
This means, any information extracted from the observation tree, does not need additional verification steps.

One possible point of execution, where the information of the observation tree can be used, is after the decomposition of the counterexample into the tuple $\langle u, a, v \rangle$.
In most occasions presented so far, key aspect of the counterexample decomposition was the extraction of a discriminator $v$.
This changes for the observation tree, where determining the state to split $s_{sp} = \delta_{\mathcal{H}}(ua)$ and the new state $s_n$ is of key interest.
Their corresponding access sequences ($as_x$ for $s_x$) may be used to transition the observation tree automaton into states $ot_{sp} = \delta_{OT}(as_{sp})$ and $ot_n = \delta_{OT}(as_n)$.
For these two states, one can now compute separating words:

In a first variation, one may continue to apply the input sequence of the old parent trace, i.e. compute a separating word for the states $\delta_{OT}(ot_{sp}, i)$ and $\delta_{OT}(ot_n, i)$ where $i$ represents the input component of $trace_{adt}(f)$ for the final node $f$ referencing $s_{sp}$ before incorporating the new discriminator $v$.
A separating word may not always exists for these states due to either undefined transitions in the observation tree automaton or simply the absence of a separating word. 
In case of success however, the computed separating word represents a distinguishing extension to the parent adaptive distinguishing sequence.
In essence, the observation tree provided a result similar to the immediate replacement heuristic (cf. \Cref{sec:immediateRepl}) without the need to use $v$ as a temporary discriminator or the need to validate a proposed replacement.
This allows to finish the refinement step with a low amount of reset queries, as only the outgoing transitions of $s_{n}$ and the incoming transitions of $s_{sp}$ need to be closed.

If such a separating word does not exist, one may still continue to try to compute a separating word for the states $ot_{sp}$ and $ot_n$.
While again, its existence is not guaranteed, any input sequence that separates $ot_{sp}$ and $ot_n$ is also able to separate $\delta_{\mathcal{M}}(as_{sp})$ and $\delta_{\mathcal{M}}(as_n)$ in the true target system $\mathcal{M}$.
Depending on the counterexample, the computed separating word may be significantly shorter than the extracted discriminator $v$.
In this case, replacing the initial discriminator $v$ with the computed separating word, the refinement process may continue as presented, but use less symbol queries compared to its original execution.

If, after all, the observation tree does not find any separating words, the learning process may continue as presented.
This case does not worsen the performance of the learning process, since all computations revolving around the observation tree work on cached values and do not pose additional reset or symbol queries.

Regarding the potential impact on the learning process, the observation tree poses a counter part to the subtree replacement costs.
Since each subtree replacement introduces query overhead for validating traces and re-sifting transitions, a high amount of replacements results in a high amount of additional queries.
However, the more replacements take place during the learning process, the more alternative knowledge will be stored in the observation tree, potentially offering improvements more often.
The influence of these two properties on each other and the performance of the learning process is depicted in the collected data in \autoref{cha:appendix}.

With regard to comparing the performance of the ADTLearner to other learning algorithms, it is noteworthy, that consulting the observation tree gives the presented algorithm an unfair advantage, because it grants access to information, that is not currently part of the main data structures, at no costs.
To create a common ground between the presented and competing algorithms, the costs of accessing old information may be nullified by a query cache.
For the developed approaches, the observation tree will also be used as a query cache, whereas the membership oracles of competing algorithm will be wrapped in a tree cache.

%% file: figures/repl_ext.tex
\centering
\subcaptionbox
{Original adaptive discrimination tree}
[0.4\textwidth]
{\input{figures/adt_ref}}%
\hspace{0.1\textwidth}%
\subcaptionbox
{Adaptive discrimination tree with applied subtree extension heuristic \label{fig:heu_extension_b}}
[0.4\textwidth]
{
	\begin{tikzpicture}[thick, ->, >=stealth']

		\tikzstyle{symbol} = [state, ellipse]
		\tikzstyle{final} = [draw,rectangle, inner sep=8pt]

		\node[symbol] (w) at(0,0) {$water$};

		\node[final] (qb) at(2, -2) {$s_1$};
		\node[symbol] (b) at(-2, -2) {$button$};

		\node[final] (qp) at (-1, -4) {$s_2$};
		\node[final] (qe) at (-3, -4) {$s_0$};

		\draw (w) -- node[anchor=south west] {\coffeeerror} (qb);
		\draw (w) -- node[anchor=south east] {\coffeeok} (b);

		\draw (b) -- node[anchor=south west] {\coffeecup} (qp);
		\draw (b) -- node[anchor=south east] {\coffeeerror} (qe);

		\node[final, draw=none] at (1, -8) {\vphantom{$q_p$}};
	\end{tikzpicture}
}

%% file: figures/adt_ref.tex
\begin{tikzpicture}[thick, ->, >=stealth']

	\node[symbol] (w1) at(0,0) {$water$};

	\node[final] (qb) at(2, -2) {$s_1$};
	\node[reset] (r) at(-2, -2) {};

	\node[symbol] (w2) at(-2,-4) {$water$};

	\node[symbol] (b) at (-2,-6) {$button$};

	\node[final] (qp) at (-1, -8) {$s_2$};
	\node[final] (qe) at (-3, -8) {$s_0$};

	\draw (w1) -- node[anchor=south west] {\coffeeerror} (qb);
	\draw (w1) -- node[anchor=south east] {\coffeeok} (r);

	\draw (r) -- (w2);

	\draw (w2) -- node[anchor=west] {\coffeeok} (b);

	\draw (b) -- node[anchor=south west] {\coffeecup} (qp);
	\draw (b) -- node[anchor=south east] {\coffeeerror} (qe);
\end{tikzpicture}

%% file: alg/heu_leveled.tex
\begin{algorithmic}[1]
	\Function{computeLeveledReplacements}{}
	\Let{$result$}{$\emptyset$}
	\Let{$queue$}{$\{adt.root\}$}
	\State{queueLoop:}
	\While{\Not $queue.isEmpty()$}
	\Let{$node$}{$queue.pop()$}
	\Let{$extension$}{$\Call{computeADTExtension}{node}$}
	\If{$extension \neq \Nil$}
	\Let{$result$}{$result \cup \{ \langle node.parent, extension, \emptyset \rangle \}$}
	\State{\textbf{continue} queueLoop}
	\EndIf
	\If{$\Call{collectResetNodes}{node} = \emptyset$}
	\State{\textbf{continue} queueLoop}
	\EndIf
	\Let{$targets$}{$\{ t.reference \vert t \in \Call{collectLeaves}{node}\}$}
	\Let{$replacement$}{$\Call{computeADS}{hypothesis, targets}$}
	\If{$replacement \neq \Nil$}
	\Let{$result$}{$result \cup \{\langle node, replacement, \emptyset \rangle\}$}
	\Else
	\ForAll{$f \in \Call{collectChildADTNodes}{node}$}
	\State{$queue.add(f)$}
	\EndFor
	\EndIf
	\EndWhile
	\State \Return $result$
	\EndFunction
\end{algorithmic}

%% file: alg/heu_extend.tex
\begin{algorithmic}[1]
	\Function{computeADTExtension}{$node$}
	\If{$node.parent = \Nil$} \Comment{true if node $=$ adt.root}
	\State \Return \Nil
	\EndIf
	\Let{$targets$}{$\{ t.reference \vert t \in \Call{collectLeaves}{node}\}$}
	\Let{$reset$}{$node.parent$}
	\Let{$\langle i, o\rangle$}{$trace_{adt}(reset)$}
	\If{$\exists s_1, s_2 \in targets, s_1 \neq s_2 : \exists k, 1 \leq k \leq \vert i \vert: \delta(s_1, i_{1:k}) = \delta(s_2, i_{1:k})$}
	\State \Return \Nil \Comment{converging states}
	\EndIf
	\Let{$mapping[S]$}{\Nil} \Comment{initialize array for all states}
	\ForAll{$s \in targets$}
	\Let{$mapping[\delta(s, i)]$}{$s$}
	\EndFor
	\Let{$extension$}{$\Call{computeADS}{hypothesis, \{ s \vert s \in S, mapping[s] \neq \Nil\}}$}
	\If{$extension = \Nil$}
	\State \Return \Nil
	\EndIf
	\ForAll{$l \in \Call{collectLeaves}{extension}$}
	\Let{$l.reference$}{$mapping[l.reference]$}
	\EndFor
	\State \Return $extension$
	\EndFunction
\end{algorithmic}

%% file: alg/heu_exhaus.tex
\begin{algorithmic}[1]
	\Function{computeEagerReplacement}{}
	\If{$\Call{collectResetNodes}{adt.root} = \emptyset$}
	\State \Return $\emptyset$
	\EndIf
	\ForAll{$c \in \Call{computeCutouts}{S}$}
	\Let{$result$}{$\Call{computeADS}{hypothesis, S \setminus c}$}
	\If{$result \neq \Nil$}
	\State \Return $\{\langle adt.root, result, c \rangle\}$
	\EndIf
	\EndFor
	\State \Return{$\emptyset$}
	\EndFunction
\end{algorithmic}

%% file: alg/heu_single.tex
\begin{algorithmic}[1]
	\Function{computeSingleReplacement}{}
	\Let{$subtrees$}{$\Call{collectAllSubADTs}{adt.root}$}
	\State{stLoop:}
	\ForAll{$st \in \Call{sortByRFScoreDesc}{subtrees}$}
	\Let{$extension$}{$\Call{computeADTExtension}{st}$}
	\If{$extension \neq \Nil$}
	\State \Return $\{\langle st.parent, extension, \emptyset \rangle\}$
	\EndIf
	\If{$\Call{collectResetNodes}{st} = \emptyset$}
	\State \textbf{continue} stLoop 
	\EndIf
	\Let{$targets$}{$\{t.reference \vert t \in \Call{collectLeaves}{st}$\}}
	\Let{$replacement$}{$\Call{computeADS}{hypothesis, targets}$}
	\If{$replacement \neq \Nil$}
	\State \Return $\{\langle st, replacement, \emptyset \rangle\}$
	\EndIf
	\EndFor
	\State \Return $\emptyset$
	\EndFunction
\end{algorithmic}

%% file: alg/heu_immed.tex
\begin{algorithmic}[1]
	\Function{computeFinalDiscriminator}{$tempDiscr$}
	\Loop{: outer}
	\Try
	\Let{$targets$}{$\Call{collectLeaves}{tempDiscr}$}
	\Let{$\langle i, o \rangle$}{$trace_{adt}(tempDiscr.parent)$}
	\Let{$mapping[S]$}{\Nil} \Comment{initialize empty array}
	\ForAll{$l \in targets$}
	\Let{$mapping[l.reference]$}{$l.reference$}
	\EndFor
	\For{$k = 1$ \textbf{to} $\vert i \vert$}
	\Let{$nextMapping[S]$}{\Nil} \Comment{initialize empty array}
	\ForAll{$s \in \{ t \vert t \in S, mapping[t] \neq \Nil\}$}
	\If{$\langle s, i_k \rangle \in openTransitions$}
	\State{$\Call{closeTransition}{\langle s, i_k \rangle}$}
	\EndIf
	\If{$\lambda(s, i_k) \neq o_k$} \Comment{inconsistency}
	\Let{$as$}{$accessSequences[mapping[s]]$}
	\State{$openCounterexamples.add(\langle as \cdot i_{i:k}, \lambda(as) \cdot o_{1:k} \rangle)$}
	\State \Return{$tempDiscr$}
	\EndIf
	\Let{$succ$}{$\delta(s, i_k)$}
	\If{$nextMapping[succ] \neq \Nil$} \Comment{converging states}
	\State \Return{$tempDiscr$}
	\EndIf
	\Let{$nextMapping[succ]$}{$mapping[succ]$}
	\EndFor
	\Let{$mapping$}{$nextMapping$}
	\EndFor
	\Let{$result$}{$\Call{computeDefensiveADS}{hypothesis, \{ s \vert s \in S, mapping[s] \neq \Nil\}}$}
	\If{$result \neq \Nil$}
	\ForAll{$l \in \Call{collectLeaves}{result}$}
	\Let{$l.reference$}{$mapping[l.reference]$}
	\EndFor
	\Else
	\State \Return{$result$}
	\EndIf
	\State \Return{$tempDiscr$}
	\EndTry
	\Catch{ModificationException}
	\State \textbf{continue} outer
	\EndCatch
	\EndLoop
	\EndFunction
\end{algorithmic}

%% file: figures/heu_immed.tex
\begin{subfigure}[t]{0.5\textwidth}
	\centering
	\subcaptionbox
	{Tentative hypothesis with partially closed transitions}[0.9\textwidth]
	{
		\begin{tikzpicture}[thick, ->, >=stealth']
			\tikzstyle{edge} = [dashed, decorate, decoration={snake, segment length=8mm}]
			\node[initial, initial text={}, state] (s0) at(0,0) {$s_0$};
			\node[state] (s1) at(4,0) {$s_1$};
			\node[state] (s2) at(0,4) {$s_2$};

			\draw (s0) -- node[anchor=east] {$pod$ / \coffeeok} (s2);
			\draw (s0) -- node[anchor=south] {$button$ / \coffeeerror} (s1);
			\draw (s0) edge[loop below] node[anchor=north] {$water$ / \coffeeok} (s0);

			\draw[edge] (s0) -- node[anchor=west] {$clean$ / \coffeeok} ++(-40:2cm);

			\draw (s1) edge[loop above] node[anchor=south] {$I$ / \coffeeerror} (s1);

			\draw (s2) edge[loop above] node[anchor=south] {$water$ / \coffeeok} (s2);

			\draw[edge] (s2) -- node[anchor=west] {$pod$ / ?} ++(50:2cm);
			\draw[edge] (s2) -- node[anchor=south, xshift=1em] {$button$ / ?} ++(-10:2cm);
			\draw[edge] (s2) -- node[anchor=west] {$clean$ / ?} ++(-40:2cm);

		\end{tikzpicture}
	}
\end{subfigure}%
\begin{subfigure}[t]{0.5\textwidth}
	\centering
	\subcaptionbox
	{Adaptive discrimination tree with temporary discriminator in dashed outline}[0.9\textwidth]
	{
		\begin{tikzpicture}[thick, ->, >=stealth']

			\node[symbol] (w1) at(0,0) {$water$};

			\node[final] (s1) at(2, -2) {$s_1$};
			\node[reset, dashed] (r) at(-2, -2) {};

			\node[symbol, dashed] (w2) at(-2,-4) {$water$};

			\node[symbol, dashed] (b) at (-2,-6) {$button$};

			\node[final] (s2) at (-1, -8) {$s_2$};
			\node[final] (s0) at (-3, -8) {$s_0$};

			\draw (w1) -- node[anchor=south west] {\coffeeerror} (s1);
			\draw (w1) -- node[anchor=south east] {\coffeeok} (r);

			\draw[dashed] (r) -- (w2);

			\draw[dashed] (w2) -- node[anchor=east] {\coffeeok} (b);

			\draw[dashed] (b) -- node[anchor=south west] {\coffeecup} (s2);
			\draw[dashed] (b) -- node[anchor=south east] {\coffeeerror} (s0);
		\end{tikzpicture}
	}
\end{subfigure}

%% file: sections/ads.tex
\chapter{On the Computation of Adaptive Distinguishing Sequences}
\label{cha:ads}

The previous chapters presented techniques to successfully integrate adaptive distinguishing sequences in the learning process and heuristics that actively employ adaptive distinguishing sequences.
However, it was always abstracted from their actual computation.
This chapter briefly discusses what different kinds of adaptive distinguishing sequences are elaborated and how the defensive computation of adaptive distinguishing sequences in the case of immediate replacements (cf. \Cref{sec:immediateRepl}) is realized.

\section{Adaptive Distinguishing Sequences}

As stated in \Cref{sec:ads}, Lee and Yannakakis proposed an algorithm (henceforth LY-al\allowbreak-gorithm) that computes -- if existent -- a quadratically bound\footnote{Bound in its length.} adaptive distinguishing sequence in polynomial time.
However, many heuristics compute adaptive distinguishing sequences only for a subset of states of the hypothesis, which is (unless $P = PSPACE$) a much harder problem and for which the LY-algorithm is generally not applicable.
Additionally, the adaptive distinguishing sequences returned by the LY-algorithm are not optimal, which is a property that again increases the complexity of the computation.

In order to allow the elaboration of different settings, a second approach to compute adaptive distinguishing sequences is utilized, that is based on the analysis of the successor tree \cite{gill1961state}.
However, the above mentioned complexity measures indicate, that certain computation strategies may highly impact the learning process runtime-wise.
Therefore, three \enquote{profiles} are considered for the evaluation:

\begin{description}

\item[Best Effort] describes the approach where the quality of the computed result is traded for its computational costs.
The best effort strategy utilizes three different (sub-) algorithms to compute an adaptive distinguishing sequence depending on the size $m$ of the target states:

\begin{compactitem}

\item for $m = n$, where $n$ denotes the size of the hypothesis (i.e. an adaptive distinguishing sequence for the complete automaton) the LY-algorithm is used.

\item for $m = 2$, an adaption of the Hopcroft-Karp algorithm \cite{hoka71} for equivalence checks of automata is used, which allows to compute a separating word -- and therefore an adaptive distinguishing sequence -- in near linear time.

\item for $2 < m < n$, a leveled breadth-first search on the successor tree is used.

\end{compactitem}

The leveled BFS is realized by iterating over the nodes of the successor tree in a breadth-first manner.
Whenever a node $k$ (and therefore a corresponding input sequence) is found that splits the current set of states e.g. in partitions $p_1$ and $p_2$, the current search is paused and new computations of adaptive distinguishing sequences for target sets $p_1$ and $p_2$ are started.
If these recursive calls return successfully (i.e. with an adaptive distinguishing sequence for the partitions), their corresponding ADS is appended to the input trace leading into node $k$.
Otherwise the search is continued at node $k+1$.
Note, that the recursion steps end with singleton partitions, which simply return the current state.

\item[Minimal Length] describes the approach of performing a breadth-first minimal cost search on the (adaptive) successor tree.
In contrast to the (classic) successor tree, does the adaptive extension allow to investigate resulting partitions independently from each other.
The costs $c(k)$ for a node $k$ are computed as follows:

$$c(k) = 1 + \max_{1 \leq l \leq m}{c(l)}$$

where $c(1), ..., c(m)$ denote the minimal costs for the $m$ child partitions of node $k$.
If a node represents a singleton state, its costs are defined as $0$.

The minimal length adaptive distinguishing sequence can then be extracted by following the path in the (adaptive) successor tree with minimal costs.

\item[Minimal Size] describes the approach of finding an adaptive distinguishing sequence of minimal size, i.e. with the minimum amount of symbol nodes.
Similar to the minimal length approach, this computation is realized by performing a breadth-first minimum cost search of the (adaptive) successor tree.
However, instead of using the maximum costs of all subtrees to compute the costs for the current node in the successor tree, this approach uses the sum of all minimal costs of child nodes.

\end{description}

\section{Defensive Adaptive Distinguishing Sequences}

The major difference between the traditional scenario for computing adaptive distinguishing sequences and the scenario encountered during the immediate replacement heuristic, is the potential undefinedness of transitions.
To be able handle undefined transitions, the previously discussed approaches can easily be extended to check for the existence of a transition and discard further analysis of an input symbol (or sequence) if necessary.
However, by skipping certain investigations, it may also happen that the potential finding of an adaptive distinguishing sequence is missed.
In general, the undefinedness of a transition is not a property of the target system -- since it is assumed to be complete -- but rather the result of a transition not being closed yet.
Thus, by using the available temporary discriminator to close a transition if necessary, the previously discussed approaches may after all find an adaptive distinguishing sequence.
This approach is formalized in \Cref{alg:defensiveads}.

\begin{algorithm}[t]
	\caption{Adaptive Distinguishing Sequences: Defensive ADS Computation \label{alg:defensiveads}}
	\input{alg/ads_defensive.tex}
\end{algorithm}

Note, that defensive computations only occur for the immediate replacement heuristic and therefore the parameter $targets$ is always a true subset of all hypothesis states.
Hence the computation of a defensive adaptive distinguishing sequence is always based on the traversal of the (adaptive) successor tree.

As indicated before, the approaches to compute a regular adaptive distinguishing sequence (i.e. during the call to \textsc{computeADS}) are extended to handle undefined transitions.
When a node (and therefore a set of current states and an input symbol) in the successor tree is encountered, for which an undefined transition exists in the hypothesis, the procedure interrupts for a special exception handling:
The symbol associated with the current successor tree node is stored in the global variable $openSymbol$.
Furthermore, for every state $s$ of the associated current-set it is checked, if the transition $\langle s, openSymbol \rangle$ is defined (there has to exists at least one state, for which this check fails).
All states, for which this check fails, are stored in the global $openStates$ variable.
If these variables were already defined, because the current traversal of the successor tree already visited a node with undefined transitions, the variables are overridden only if the current set of open states is smaller than the existent global one.
Afterwards, the current node (and its subtree) is discarded and the traversal of the successor tree is continued.

If, after the termination of the \textsc{computeADS} call, no adaptive distinguishing sequence is found, it is checked if there exist open transitions, which may have prevented the successful finding of an ADS.
Note, that by construction, the minimal amount of transitions is closed, to ensure progress for the next traversal of the successor tree.
If all encountered transitions are closed (i.e. the successor tree traversal has not defined any open states or symbol) the absence of a result corresponds to the absence of an adaptive distinguishing sequence for the complete hypothesis.

%% file: alg/ads_defensive.tex
\begin{algorithmic}[1]
	\Function{computeDefensiveADS}{$hypothesis, targets$}
	\Let{$result$}{$\Call{computeADS}{hypothesis, targets}$}
	\While{$result = \Nil$}
	\If{$openStates \neq \emptyset \And openSymbol \neq \Nil$}
	\ForAll{$s \in openStates$}
	\State{$\Call{closeTransition}{\langle s, openSymbol \rangle}$}
	\EndFor
	\Let{$openStates$}{$\emptyset$}
	\Let{$openSymbol$}{$\Nil$}
	\Let{$result$}{$\Call{computeADS}{hypothesis, targets}$}
	\Else
	\State \Return \Nil
	\EndIf
	\EndWhile
	\State \Return $result$
	\EndFunction
\end{algorithmic}

%% file: sections/evaluation.tex
\chapter{Evaluation}
\label{cha:eval}

This chapter presents the evaluation of the developed approaches of this thesis.
It will analyze key characteristics of the base algorithm, the impact of the proposed heuristics and compare their performance to other state-of-the-art learning algorithms.
\Cref{sec:theoanal} focuses on the theoretical analysis, presenting worst-case boundaries for certain properties of the algorithm.
However, due to the nature of the heuristics, a fine-grained analysis is cumbersome and requires a certain set of assumptions.
In order to give a more practical view on the performance, \Cref{sec:empanal} additionally presents the results of several empirical analyses.
A set of synthetic benchmarks is used to point out certain effects and characteristics of the developed approaches, whereas two real-life systems are used to show the performance in realistic environments.

\section{Theoretical Analysis}
\label{sec:theoanal}

In computer science, algorithms are often analyzed with regard to certain complexity measures to give an indication about their performance.
In many cases, the property of interest is time complexity:
Given the size of the input of an algorithm, it provides an estimate -- in most cases an upper bound -- on the number of steps the algorithm executes before terminating.
However, as already stated in \Cref{sec:problem}, for active learning algorithms this complexity measure has the tendency to be meaningless, as different execution steps may require a highly varying amount of time.
In fact, in many real-life applications, the performance of the target system is the dominating factor for the runtime performance.
As a result, rather than analyzing the time complexity, the active learning community often analyzes the query complexity of learning algorithms, giving estimates of the maximum number of posed membership queries and their maximum length.

Adjusting to the adaptive scenario, the following sections will provide asymptotic upper boundaries for the number of equivalence, reset and symbol queries for the base algorithm and discuss the impact of the presented heuristics.

\subsection{Base Algorithm}

For determining upper bounds on the various types of queries, one should recall the properties of a worst-case scenario:
Each refinement step only leads to the discovery of single new equivalence class, reducing the impact of a counterexample to its minimum.
Furthermore does the adaptive distinguishing tree yield its worst performance, when it degenerates to a linear list and each sifting operation requires the traversal of the complete adaptive discrimination tree.

This allows to give the following bound on the different types of queries:

\begin{theorem}[Boundaries for the number of symbol, reset and equivalence queries of the base algorithm]\label{thm:baseanal}
Let $n$ denote the size of the target system, $k$ the size of the input alphabet and $m$ the size of the longest counterexample.
The base algorithm (cf. \Cref{sec:base}) requires at most

\begin{compactitem}

\item $\mathcal{O}(n)$ equivalence queries,
\item $\mathcal{O}(kn^2 + n \log_2 m)$ reset queries and
\item $\mathcal{O}(kn^2m + nm \log_2 m)$ symbol queries.

\end{compactitem}
\end{theorem}

\begin{proof}

\textit{Equivalence queries}:
In analogy to \Cref{thm:termination}, a maximum of $n-1$ equivalence queries can be posed, after which the learning algorithm has detected all distinct equivalence classes.
An additional equivalence query -- indicating equivalence -- is posed to detect the termination, which results in a total of $n$ equivalence queries.

For determining an upper bound for the amount of reset and symbol queries, it is reasonable to split the analysis into two parts:
The costs for analyzing a counterexample and the costs for refining the hypothesis.

\textit{Reset queries}:
The decomposition of a counterexample is done by performing a binary search on the counterexample and computing the output for the extracted discriminator, which results in a maximum of $2 + \log_2 m$ queries per counterexample.
This is done in every single of the possible $n-1$ possible refinement steps, which accumulates the impact of the counterexample decomposition to a total of $\mathcal{O}(n \log_2 m)$ reset queries.

The costs for refining the hypothesis can be analyzed for each refinement step independently:
In refinement step $j$ -- after the $j+1$\textsuperscript{st} state is added to the hypothesis -- the adaptive distinguishing tree has at most $j-1$ reset nodes.
All $k$ outgoing transitions ($N$) of the new ($j+1$\textsuperscript{st}) state may lead into one of the two states referenced in the lowest subtree, resulting in a sift operation that traverses the complete discrimination tree.
Additionally, all $k \cdot j$ existing transitions ($O$) may require a sift operation through the lowest subtree, as they could have led into the split state.
The amount of reset queries during the initialization ($I$) is bounded by the number of input symbols.

Across the possible $n-1$ refinement steps, the number of reset queries then computes as follows:

\begin{align*}
\#rq_{ref} &\leq \underbrace{k}_{I} + \sum_{i = 1}^{n-1} \underbrace{k \cdot i}_{N} + \underbrace{k \cdot i}_{O}\\
&= k + 2k \cdot \sum_{i = 1}^{n-1} i\\
&= k + 2k \cdot \dfrac{n(n-1)}{2}\\
&= kn^2 - kn + k\\
&\in \mathcal{O}(kn^2)
\end{align*}

\textit{Symbol queries}:
While the number of queries posed during the counterexample decomposition is limited by $\log_2 m$, the queries themselves consist partly of the stored access sequences of hypothesis states (i.e. representatives of equivalence classes) and subsequences of the actual counterexample.
However, no stored access sequence can be longer than the input sequence leading into the respective state.
As a result, the maximum query length is bounded by $m$.
Therefore the amount of symbol queries caused by the counterexample decomposition computes as:

\begin{align*}
\#sq_{cd} &\leq \sum_{i = 1}^{n-1} m \log_2 m\\
&= m \log_2 m + \sum_{i = 1}^{n-1} 1\\
&= m \log_2 m \cdot (n-1)\\
&= nm \log_2 m - m \log_2 m\\
&\in \mathcal{O}(nm \log_2 m)
\end{align*}

Similar to the case of reset queries, in a worst-case scenario, the outgoing transitions of the most recently added hypothesis state ($N$) sift through the complete adaptive discrimination tree and every other transition ($O$) needs to be updated using the latest obtained discriminator.
During the sift operation, the queried input sequences have the form \enquote{access sequence $\cdot$ transition symbol $\cdot$ discriminator}.
The length of these sequences is bounded by $2m$, because every access sequence is a prefix of one counterexample and every discriminator is a suffix of one counterexample.
The initialization costs ($I$) are bounded by $k$, because for every transition only its output is determined.
The number of symbol queries then computes as follows:

\begin{align*}
\#sq_{ref} &\leq \underbrace{k}_{I} + \sum_{i = 1}^{n-1} \underbrace{2m \cdot i \cdot k}_{N} + \underbrace{2m \cdot i \cdot k}_{O}\\
&= k + 4mk \sum_{i = 1}^{n-1} i\\
&= k + 4mk \dfrac{n(n-1)}{2}\\
&= k + 2mk \cdot (n^2 -n)\\
&= 2kn^2m - 2knm + k\\
&\in \mathcal{O}(kn^2m) \qedhere
\end{align*}

\end{proof}

The bounds of the base algorithm coincide with the boundaries of the classic discrimination tree algorithm, which can be attributed to the worst-case evaluation.
However, the base algorithm does not utilize any replacement heuristics whose impact is analyzed in the subsequent sections.
As the costs of the counterexample decomposition will remain the same among all heuristics, their explicit mentioning will be omitted in the following proofs.

\subsection{Subtree Extensions}

For analyzing the impact of the presented heuristics, one faces the inherent problem of them: their potential non-applicability.
An accurate analysis is therefore highly problem dependent, as not only the target system is responsible for the structure of intermediate hypotheses, but also which counterexamples -- as they provide discriminators -- are encountered.
This makes a general analysis hard, if not impossible.

It is however possible to sketch the impact of the presented heuristics, by assuming certain conditions.
For example, if it is assumed, that the subtree extension is not applicable in any refinement step, it is easy to see, that the resulting query complexity coincides with the complexity of the base algorithm.
For the subtree extension heuristic, it is interesting to see its impact, if applicable in every refinement step.
This scenario is formulated in \Cref{thm:extanal}:

\begin{theorem}[Boundaries for the number of symbol, reset and equivalence queries with a successful subtree extension heuristic]\label{thm:extanal}
Let $n$ denote the size of the target system, $k$ the size of the input alphabet and $m$ the size of the longest counterexample.
Under the assumption of a successful application in each refinement step, the subtree extension heuristic (cf. \Cref{sec:subext}) requires at most

\begin{compactitem}

\item $\mathcal{O}(n)$ equivalence queries,
\item $\mathcal{O}(kn^2 + n \log_2 m)$ reset queries and
\item $\mathcal{O}(kn^2m + nm \log_2 m)$ symbol queries.

\end{compactitem}
\end{theorem}

\begin{proof}

\textit{Equivalence queries}:
See \Cref{thm:baseanal}.

\textit{Reset queries}:
Since the heuristic only reorganizes the current adaptive discrimination tree, no additional costs are introduced by the heuristic.
In the $j$-th refinement step, solely the $k$ outgoing transitions of the new state and the (up to) $jk$ incoming transitions of the node to split need to be updated.
The adaptive discrimination tree does not contain reset nodes at any time during the execution.
The number of reset queries therefore computes as follows:

\begin{align*}
\#rq_{ref} &\leq \underbrace{k}_{I} + \sum_{i = 1}^{n-1} \underbrace{k}_{N} + \underbrace{k \cdot i}_{O}\\
&= k + k(n-1) + k\sum_{i = 1}^{n-1} i\\
&= k + k(n-1) + k \dfrac{n(n-1)}{2}\\
&= \dfrac{kn^2}{2} + \dfrac{kn}{2}\\
&\in \mathcal{O}(kn^2)
\end{align*}

\textit{Symbol queries}:
Similar to the base algorithm, an upper bound for the length of each sifted word is given by $2m$.
However, the outgoing transitions of the new state only require a single sequence of input symbols as they do not encounter any reset nodes.

\begin{align*}
\#sq_{ref} &\leq \underbrace{k}_{I} + \sum_{i = 1}^{n-1} \underbrace{2m \cdot k}_{N} + \underbrace{2m \cdot i \cdot k}_{O}\\
&= k + 2mk(n-1) + 2mk \sum_{i = 1}^{n-1} i\\
&= k + 2mk(n-1) + 2mk \dfrac{n(n-1)}{2}\\
&= kn^2m + knm - 2km + k\\
&\in \mathcal{O}(kn^2m) \qedhere
\end{align*}

\end{proof}

While the exact bounds show a slight improvement, the asymptotic bounds remain equal.
Although, the possible improvements should be taken with a grain of salt:
The assumptions essentially enforce the existence of a single (and potentially long) discriminator, that is able to distinguish every state of the target system.
Additionally, each counterexample has to decompose in a way that allows to gradually construct the adaptive discrimination tree.
Encountering this scenario in a real system seems unlikely.

\subsection{Subtree Replacements}

Similar to the previous section, the applicability of subtree replacements is highly problem-dependent and therefore hard to analyze for the general case.
However, contrary to the previous heuristic, subtree replacements introduce costs for verifying replacements and updating the hypothesis.
To sketch the impact of these costs, the following theorem analyzes the query complexity in case of a successful replacement of the complete adaptive discrimination tree in every refinement step, which corresponds to the behavior of the leveled (cf. \Cref{sec:levelrepl}) and exhaustive (cf. \Cref{sec:exhausrepl}) subtree replacement heuristic.

\begin{theorem}[Boundaries for the number of symbol, reset and equivalence queries with a successful subtree replacement heuristic]\label{thm:levelanal}
Let $n$ denote the size of the target system, $k$ the size of the input alphabet and $m$ the size of the longest counterexample.
Under the assumption of a successful replacement of the complete adaptive discrimination tree with a single adaptive distinguishing sequence in each refinement step, the learning process requires at most

\begin{compactitem}

\item $\mathcal{O}(n)$ equivalence queries,
\item $\mathcal{O}(kn^2 + n \log_2 m)$ reset queries and
\item $\mathcal{O}((n^2 + m) kn^2 + nm \log_2 m)$ symbol queries.

\end{compactitem}
\end{theorem}

\begin{proof}

\textit{Equivalence queries}:
See \Cref{thm:baseanal}.

\textit{Reset queries}:
The heuristic does not change the initialization step and the decomposition of the counterexample, so the costs remain identical to the ones described in \Cref{thm:baseanal}.
However, the refinement step introduces a complete and successful replacement of the discrimination tree.
This means prior to the $j$-th refinement step, a verification ($V$) of $j$ traces and a re-sifting of (up to) $jk$ transitions is required.
After splitting the state of the hypothesis the adaptive discrimination tree has at most one reset node and only the $k$ outgoing transitions ($N$) of the most recently added hypothesis state may be affected by this reset node, while all other transitions ($O$) only need to sift through the new, lowest subtree.
The number of reset queries then computes as follows:

\begin{align*}
\#rq_{ref} &\leq \underbrace{k}_{I} + \sum_{i = 1}^{n-1} \underbrace{i + i \cdot k}_{V} + \underbrace{2k}_{N} + \underbrace{k \cdot i}_{O}\\
&= k + 2k(n-1) + \sum_{i = 1}^{n-1} i + 2k \cdot \sum_{i = 1}^{n-1} i\\
&= k + 2k(n-1) + \dfrac{n(n-1)}{2} + 2k \cdot \dfrac{n(n-1)}{2}\\
&= kn^2 + kn - k +\dfrac{n^2 - n}{2}\\
&\in \mathcal{O}(kn^2)
\end{align*}

\textit{Symbol queries}:
The initialization ($I$) and decomposition costs remain identical to the base scenario.
The depth of an adaptive distinguishing sequence for an automaton of size $n$ is bounded by $(n^2 - n)/2$ \cite{rystsov76}.
Hence an upper bound for the number of symbol queries during the verification ($V$) of a trace in the $j$-th refinement step is given by $j + j^2$, since the length of an access sequence is bounded by $j$.
A similar bound holds for the costs of updating ($U$) the (up to) $jk$ transitions after the adaptive discrimination tree has been replaced.
After decomposing the counterexample, the succeeding update of the new transitions ($N$) may cause at most a total of $k \cdot (j + j^2 + 2m)$ symbol queries, whereas the old transitions ($O$) have costs similar to the previous scenarios.
In total, the number of symbol queries computes as follows:

\begin{align*}
\#sq_{ref} &\leq \underbrace{k}_{I} + \sum_{i = 1}^{n-1} \underbrace{(i + i^2) \cdot i}_{V} + \underbrace{(i + i^2) \cdot i \cdot k}_{U} + \underbrace{(i + i^2 + 2m) \cdot k}_{N} + \underbrace{2m \cdot i \cdot k}_{O}\\
&= k + \sum_{i = 1}^{n-1} i^2 + i^3 + i^2k + i^3k + ik + i^2k + 2mk + 2mik\\
&= k + (k+1) \sum_{i = 1}^{n-1} i^3 + (2k + 1) \sum_{i = 1}^{n-1} i^2 + (k + 2mk)\sum_{i = 1}^{n-1} i + (n-1)\cdot 2mk\\
&= k + (k+1) \frac{(n-1)^4 + 2(n-1)^3 + (n-1)^2}{4} + (2k+1) \dfrac{2(n-1)^3 + 3(n-1)^2 + n - 1}{6}\\
& \quad + (k + 2mk) \cdot \dfrac{n(n-1)}{2} + (n-1)\cdot 2mk\\
&\in \mathcal{O}(kn^4 + mkn^2)
\end{align*}%

For the transformation of the sum of the p-th power of the first $n-1$ integers in this and the following proofs, see Faulhaber's formula \cite{citeulike:2879324}.
\end{proof}

The bounds show, that the additional costs for validation outweigh the potential benefits gained by them.
In case of symbol queries, this even affects the asymptotic bound.

\subsection{Immediate Replacements}

Similar to previous heuristics, the applicability of the immediate replacement heuristic highly depends on the encountered situation.
To sketch the potential impact, the following theorem gives a bound on the query complexity in case of a repeatedly successful application of the heuristic.

\begin{theorem}[Boundaries for the number of symbol, reset and equivalence queries with a successful immediate replacement heuristic]\label{thm:immediateanal}
Let $n$ denote the size of the target size, $k$ the size of the input alphabet and $m$ the size of the longest counterexample.
Under the assumption of a successful application in each refinement step, the immediate replacement heuristic for two nodes (cf. \Cref{sec:immediateRepl}) requires at most

\begin{compactitem}

\item $\mathcal{O}(n)$ equivalence queries,
\item $\mathcal{O}(kn^2 + n \log_2 m)$ reset queries and
\item $\mathcal{O}((n^2 + m)kn^2 + nm \log_2 m)$ reset queries and

\end{compactitem}
\end{theorem}

\begin{proof}

\textit{Equivalence queries}:
See \Cref{thm:baseanal}.

\textit{Reset queries}:
Each defensive computation may require to close all transitions ($C$) using the temporary discriminator.
Since all immediate replacements are assumed to be successful, the adaptive discrimination tree with the temporary discriminator only contains one reset node.
The replacement only contains two traces to validate ($V$).
After the successful replacement, all transitions may need to be updated ($U$), because they were closed using the old (now obsolete) temporary discriminator.

\begin{align*}
\#rq_{ref} &\leq \underbrace{k}_{I} + \sum_{i = 1}^{n-1} \underbrace{2k + i \cdot k}_{C} + \underbrace{2}_{V} + \underbrace{(i+1) \cdot k}_{U}\\
&= k + 2(n-1) + 3k(n-1) + 2k\sum_{i = 1}^{n-1} i\\
&= k + 2(n-1) + 3k(n-1) + 2k \dfrac{n(n-1)}{2}\\
&= kn^2 + 2kn -2k + 2(n-1)\\
&\in \mathcal{O}(kn^2)
\end{align*}

\textit{Symbol Queries}:
Closing ($C$) the transitions of the new state, requires a sifting operation through the existing adaptive discrimination tree and the temporary discriminator.
After the $j$-th refinement step finishes, the ADT has a maximum depth of $\sum_{i=1}^{j} j \leq j^2$:
In every previous refinement step $l$ ($l \leq j$) the immediate replacement extends the ADT by a sequence of at most $l$ symbols, because the separating word between two states in an automaton of size $l+1$ is at most $l$ \cite{Moore56}.
During the $j$-th refinement step, the complete costs for closing a transition of the newly added state therefore consist of $j + j^2$ (\enquote{old} ADT) + $2m$ (temporary discriminator).
The existing transitions only sift through the temporary discriminator.
The length of the intermediate replacement in the $j$-th refinement step is bounded by $j$, with the same reasoning about the length of a separating word.
Updating ($U$) potentially all transitions requires a sift operation through the updated ADT of maximum depth $(j+1)^2$.
Each access sequence sequence is bounded by the size of the current hypothesis.

\begin{align*}
\#sq_{ref} &\leq \underbrace{k}_{I} + \sum_{i = 1}^{n-1} \underbrace{k \cdot (i + i^2) + 2m \cdot (i+1) \cdot k}_{C} + \underbrace{2i}_{V} + \underbrace{(i + (i+1)^2) \cdot (i+1) \cdot k}_{U}\\
&= k + \sum_{i = 1}^{n-1} ki + ki^2 + 2mik + 2mk + 2i + i^3k + 3i^2k + ik + i^2k + 3ik +k\\
&= k + k\sum_{i = 1}^{n-1} i^3 + 5k\sum_{i = 1}^{n-1} i^2 + (5k + 2mk + 2) \sum_{i = 1}^{n-1} i + (n-1)(2mk + k)\\
&= k + k \dfrac{(n-1)^4 + 2(n-1)^3 + (n-1)^2}{4} + 5k \dfrac{2(n-1)^3 + 3(n-1)^2 + (n-1)}{6}\\
&\quad + (5k+ 2mk + 2) \frac{n(n-1)}{2} + (n-1)(2mk + k)\\
&\in \mathcal{O}(kn^4 + mkn^2) \qedhere
\end{align*}

\end{proof}

Even though, the immediate replacement heuristic proposes replacements on a smaller scale, the (asymptotic) worst-case performance coincides with the one of the subtree replacement heuristic.

\subsection{Summary}

A summary of the obtained worst-case bounds is given in \Cref{tab:compl}:

\begin{table}[h]
	\caption{Asymptotic worst-case query complexity of certain heuristics}
	\label{tab:compl}
	\centering
	\begin{tabular}{lll}
		\textbf{Heuristic} & \textbf{Reset Complexity} & \textbf{Symbol Complexity} \\  
		\midrule 
		Base Algorithm & $\mathcal{O}(kn^2 + n \log_2 m)$ & $\mathcal{O}(kn^2m + nm \log_2 m)$\\ 
		Subtree Extensions & $\mathcal{O}(kn^2 + n \log_2 m)$ & $\mathcal{O}(kn^2m + nm \log_2 m)$\\ 
		Leveled Subtree Replacements & $\mathcal{O}(kn^2 + n \log_2 m)$ & $\mathcal{O}((n^2 + m)kn^2 + nm \log_2 m)$\\ 
		Immediate Replacements & $\mathcal{O}(kn^2 + n \log_2 m)$ & $\mathcal{O}((n^2 + m)kn^2 + nm \log_2 m)$\\ 
	\end{tabular}
\end{table}

The theoretical analysis has shown, that the developed approaches in general do not improve the worst-case asymptotic query performance compared to the base algorithm.
In fact, for the number of reset queries, the explicit bounds are considerably higher when applying the presented replacement heuristics.
For the number of symbol queries, this even affects the asymptotic bound.
Discarding validated replacements or using (exponentially bound) partial adaptive distinguishing sequences may worsen this situation even further.

These results may discourage the use of the presented heuristic at first.
However, as stated in the previous section, a comprehensive analysis is hard due to the inherent problem dependent behavior of the heuristics.
Hence their actual performance, when applied to certain systems, may differ from the theoretic results.

\section{Empirical Analysis}
\label{sec:empanal}

To give an impression of the practical impact of the presented heuristics, this section presents the empirical evaluation.
A series of benchmarks for artificial and realistic systems were run to inspect characteristics of certain algorithms and compare their overall performance.
The following algorithms were evaluated:

\vspace{1ex}
\begin{compactitem}

\item \textbf{ADT}: the presented ADTLearner and its heuristics,
\item \textbf{KV}: the original discrimination tree algorithm by Kearns \& Vazirani (Mealy version) \cite{Kearns:1994:ICL:200548},
\item \textbf{DT}: the \enquote{discrimination tree} algorithm\footnote{Also known as \enquote{Observation Pack} algorithm.} by Howar \cite{howar}, 
\item \textbf{LStarM}: the $L^*$ algorithm by Angluin (Mealy version) \cite{DBLP:phd/de/Niese2003},
\item \textbf{TTT}: the TTT algorithm by Isberner et al. (Mealy version) \cite{ttt}.

\end{compactitem}
\vspace{1ex}

\noindent
Each benchmark has been executed on an Intel\textregistered~Core\textsuperscript{TM} i7-4790 system and was assigned $8$ gigabyte of memory ($24$ gigabyte for the ESM benchmark).

The subsequent sections introduce the different test setups and present an excerpt of the measured data.
For the full set of data, see \autoref{cha:appendix}.
Recall, that for the measurement of the reset and symbol performance, all algorithms used a query cache.
The ADTLearner (including its heuristics) used the integrated observation tree, while the membership oracles of the competing algorithms were wrapped in a tree cache.
Hence the collected number of reset and symbol queries represent the number of \emph{unique} reset and symbol queries.

\subsection{Synthetic Benchmarks}

For the synthetic benchmarks, each target system $\mathcal{T}$ was obtained by creating a random Mealy machine:
At first, for a given size, the set of states $S_{\mathcal{T}}$ was constructed.
Then, for all tuples $\langle s, i \rangle \in S_{\mathcal{T}} \times I_{\mathcal{T}}$, the successor (output) for each transition was determined locally by uniformly sampling an element from the set of states (outputs).

In each benchmark run, the dimensions of target system were chosen as follows: $\vert S \vert = 1000$, $\vert I \vert = 25$, $\vert O \vert = 10$.
A total of two test series, each containing 100 runs, were benchmarked.
The series differed in the way counterexamples were obtained:

\begin{itemize}

\item For the first series, the counterexamples were generated by finding separating words \cite{hoka71} between the current hypothesis and the true target system.
Therefore the counterexamples had near perfect length.

\item For the second series, each randomly generated target system was altered in the following way:
Let $s_0, ..., s_x, ..., s_{\vert S \vert - 1}$ denote the (ordered) sequence of states and $i_0, ..., i_y, ..., y_{\vert I \vert - 1}$ denote the (ordered) sequence of input symbols.
For each state $s_x$, the transition $\langle s_x, i_{x \mod \vert I \vert} \rangle$ was turned into a self-loop, i.e. $\delta(s_x, i_{x \mod \vert I \vert}) = s_x ~\forall 0 \leq x < \vert S \vert$.

Similar to the first series, counterexamples were initially determined by finding separating words between the current hypothesis and the altered target automaton.
Each obtained separating word $sep$ was then expanded to the counterexample 

$$\hat c = sep_{1:\vert sep \vert - 1} \cdot i_{y}^k \cdot sep_{\vert sep \vert}$$

\noindent
where $y = x \mod \vert I \vert$ for $s_x = \delta(sep_{1: \vert sep \vert - 1})$ and $k = 500 - \vert sep \vert$.
This means, after applying the first $\vert sep \vert - 1$ symbols of the separating word to the target system, the input symbol corresponding to the looped transition was applied until the counterexample reached a length of $499$.
Afterwards the last input symbol of the separating word was appended.
If the expanded counterexample $\hat c$ still posed a valid counterexample to the hypothesis (this is e.g. not the case if $i_y = sep_{\vert sep \vert}$), it was used for the hypothesis refinement step.
Otherwise, the original separating word $sep$ was passed to the refinement function.

\end{itemize}

\paragraph*{Separating Word}

\Cref{fig:bench_rand_heu} shows the reset and symbol query performance of the ADT-Learner for a selected set of heuristics.
Displayed are the averaged values with the standard deviation as error bars.

\begin{figure}[!t]
	\centering
	\subcaptionbox{Reset query performance}[\linewidth]{
		\includegraphics[width=\linewidth]{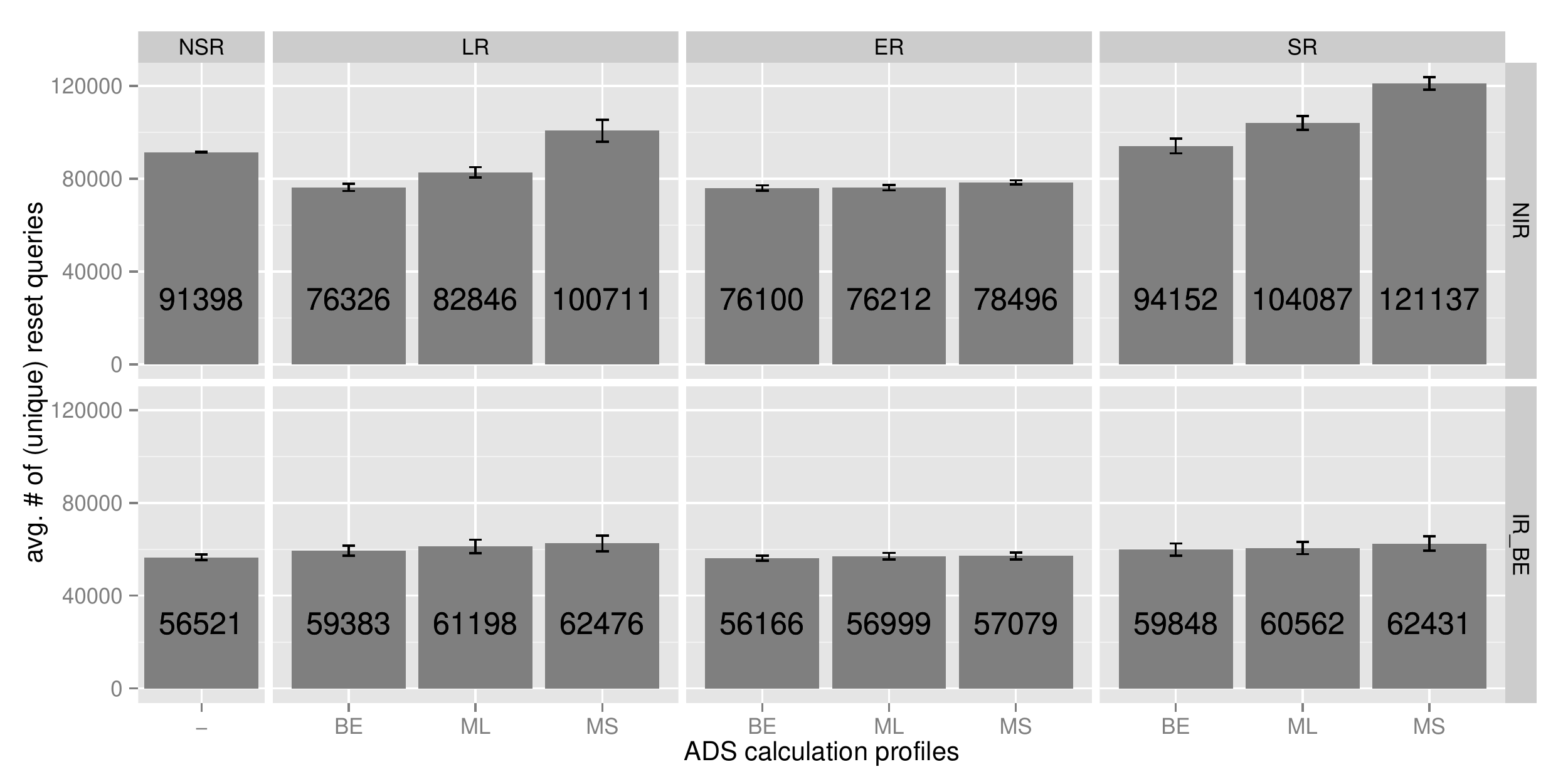}
	}
	\subcaptionbox{Symbol query performance}[\linewidth]{
		\includegraphics[width=\linewidth]{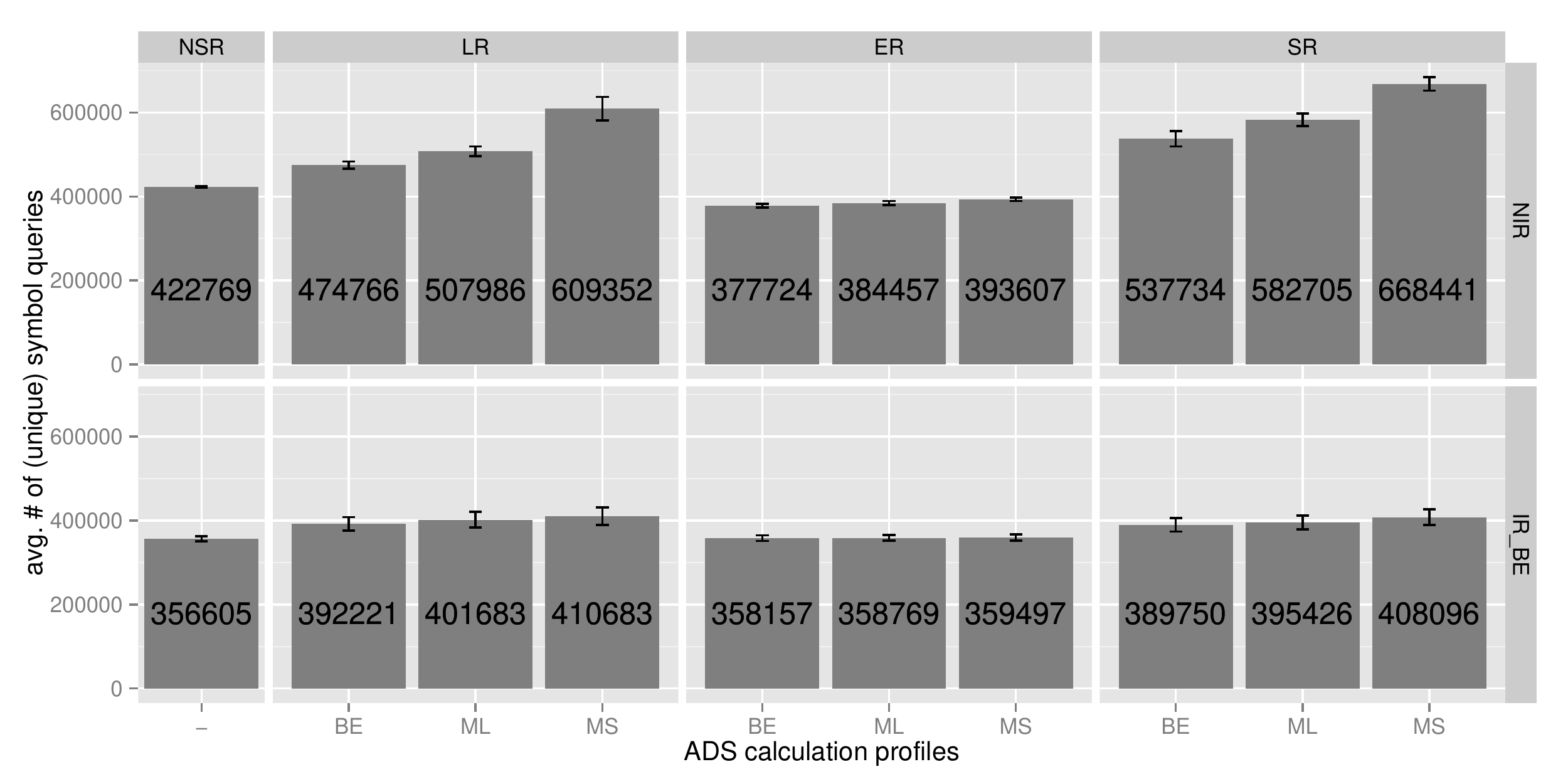}
	}
	\caption{Benchmark: Query performance of selected heuristics (random Mealy, separating word)}
	\label{fig:bench_rand_heu}
\end{figure}

The plots are organized as follows:
Each of the four vertical blocks represents a subtree replacement heuristic, with

\begin{compactitem}

\item \textbf{NSR} not subtree replacement heuristic,
\item \textbf{LR} leveled replacement heuristic,
\item \textbf{ER} exhaustive replacement heuristic and
\item \textbf{SR} single replacement heuristic.

\end{compactitem}

\noindent
Each block that covers a replacement heuristic, contains measurements for the presented ADS calculation profiles

\begin{compactitem}

\item \textbf{BE} best effort,
\item \textbf{ML} minimum length and
\item \textbf{MS} minimum size.

\end{compactitem}

\noindent
The two vertical blocks differentiate between the additional usage of the immediate replacement heuristic, with

\begin{compactitem}

\item \textbf{NIR} no immediate replacement heuristic and
\item \textbf{IR\_BE} immediate replacement heuristic using the best effort profile.

\end{compactitem}

\noindent
In all displayed configurations of the synthetic benchmarks, no subtree extension heuristic was used, since it posed no significant improvement.

Regarding reset complexity, one can see, that different heuristics have different impact on the total amount of (unique) reset queries:
Similar for all subtree replacement heuristics is, that especially the minimum size profile for computing adaptive distinguishing sequences increases the amount of executed reset queries.
The initial motivation behind computing a minimal (e.g. in size) ADS, was to improve the validation process, since fewer symbols have to be validated.
While the measurements show (cf. \autoref{cha:appendix}) that the minimal size profile indeed proposes fewer symbols to validate, the validation process encounters more errors and validated (and accepted) replacements have more reset nodes compared to the other profiles.
The condensed structure of minimal ADSs seems to work against the error-tolerating validation mechanism.

The exhaustive replacement heuristic seems to be least affected by this effect.
This may however be explained by the fact, that throughout the learning process this heuristic often only proposed a single replacement.

The single replacement heuristic shows the worst performance, even resulting in a higher amount of executed resets than the base algorithm.
The selection strategy of this heuristic seems to propose replacements that affect the most hypothesis states (cf. \autoref{cha:appendix}) among the other replacement heuristics.
Thus, accepting such a replacement results in high costs for updating the hypothesis afterwards.

For the sole employment of a subtree replacement heuristic, the \enquote{best effort} ADS profile and the leveled and exhaustive subtree replacement heuristic seem to yield the biggest benefit.
However, the immediate replacement heuristic allows even more improvements.
As shown, every configuration involving the immediate replacement heuristic outperforms the base algorithm and the best subtree replacement heuristics.
Combining the immediate replacement heuristic with the regular subtree replacement heuristics shows, that in this scenario the additional validation and update costs introduced by the subtree replacements generally worsen the overall performance.
Only the rather defensive exhaustive subtree replacement heuristic allows a slight improvement. 
When used in combination with the immediate replacement heuristic, the subtree replacement heuristics seem to propose fewer replacements, which may explain the reduced variability among the difference replacement heuristics and ADS calculation profiles.

Additional combinations of other ADS calculation profiles or the subtree extension heuristic show similar results to the combinations shown above.
As a result, they are not presented in detail.

Regarding symbol complexity, slightly different results can be observed:
While the leveled subtree replacements heuristic improves the reset complexity compared to the base algorithm, it worsens the symbol complexity.
An indication of this effect was already given by the higher theoretical worst-case query complexity (cf. \Cref{sec:theoanal}).
However, the exhaustive subtree replacement heuristic -- which often only proposed a single replacement -- shows, that a conservative utilization of subtree replacements may still improve the query performance.

Similar to the reset performance, the immediate replacement heuristic clearly poses an improvement to the query performance overall.
Again, combining subtree replacement heuristics with an immediate replacement heuristic reduces the amount of subtree replacements and therefore retains their variability.
Yet the subtree replacement heuristics seem to, again, decrease the performance when used in combination with the immediate replacement heuristic.

After all, the presented heuristic have to compete against other active learning algorithms.
A comparison of the performance of selected heuristics and other algorithms is shown in \Cref{fig:bench_rand_comp}:

\begin{figure}[ht]
	\centering
	\subcaptionbox{Reset query performance}[\linewidth]{
		\includegraphics[width=\linewidth]{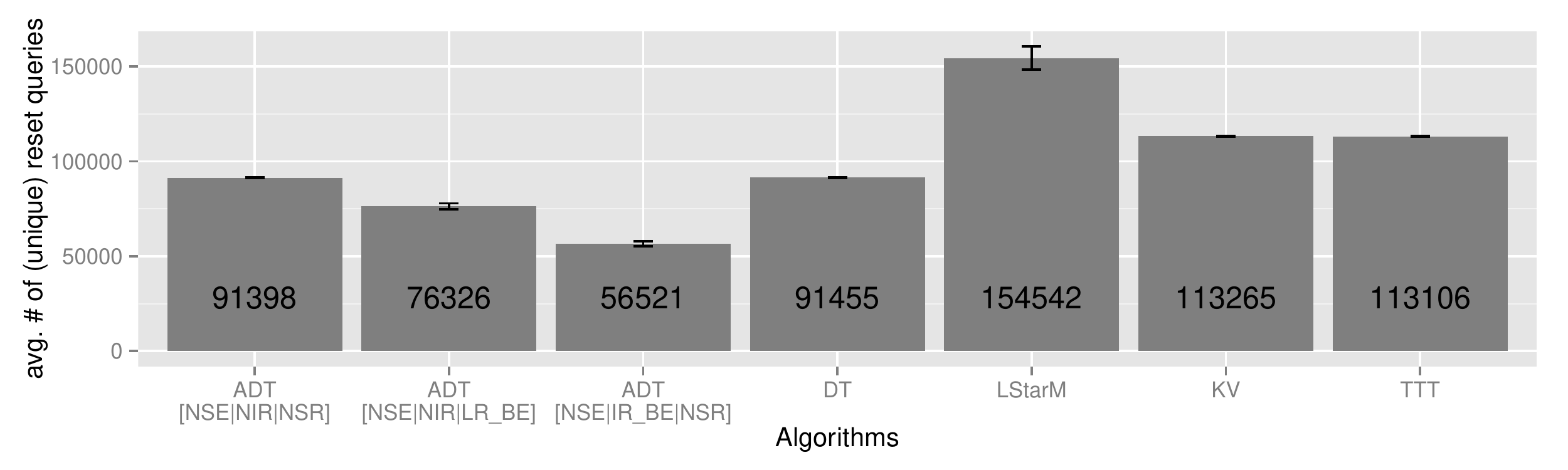}
	}
	\subcaptionbox{Symbol query performance}[\linewidth]{
		\includegraphics[width=\linewidth]{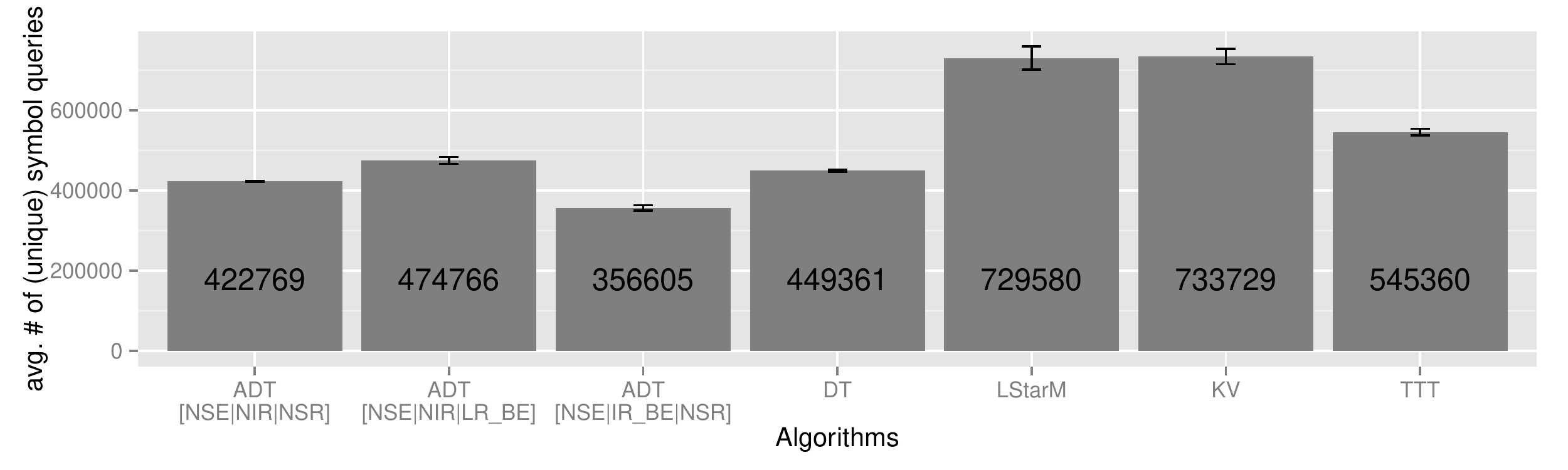}
	}
	\caption{Benchmark: Query performance of selected heuristics compared to other algorithms (random Mealy, separating word)}
	\label{fig:bench_rand_comp}
\end{figure}

The direct comparison with competing learning algorithm shows, that for the case of random Mealy machines, the developed approaches are able to achieve the set goal of reducing the number of required reset queries during the learning process.
Additionally, certain heuristics are also able to provide a better symbol query performance, compared to other algorithms.
However, some of the results may be attributed to the fact, that the provided counterexamples had near minimal length.
The next paragraph compares the heuristics and algorithms, for a scenario where redundancy is encountered.

\paragraph*{Expanded Separating Word}

Similar to the previous paragraph, \Cref{fig:bench_redundancy_heu}  shows the reset and symbol query performance of the ADTLearner for the same set of heuristics.
Displayed are the averaged values with the standard deviation as error bars.

\begin{figure}[!t]
	\centering
	\subcaptionbox{Reset query performance}[\linewidth]{
		\includegraphics[width=\linewidth]{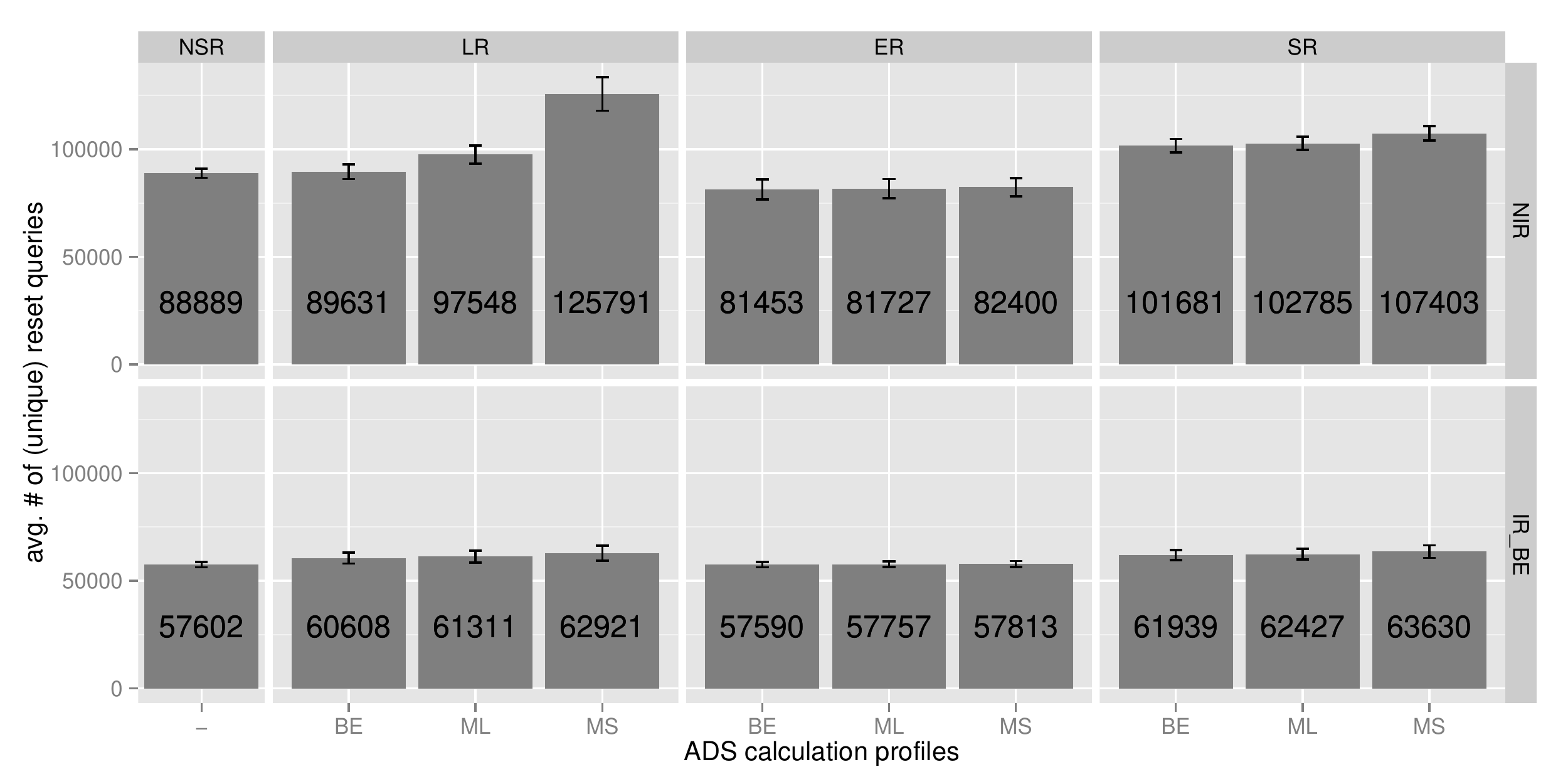}
	}
	\subcaptionbox{Symbol query performance}[\linewidth]{
		\includegraphics[width=\linewidth]{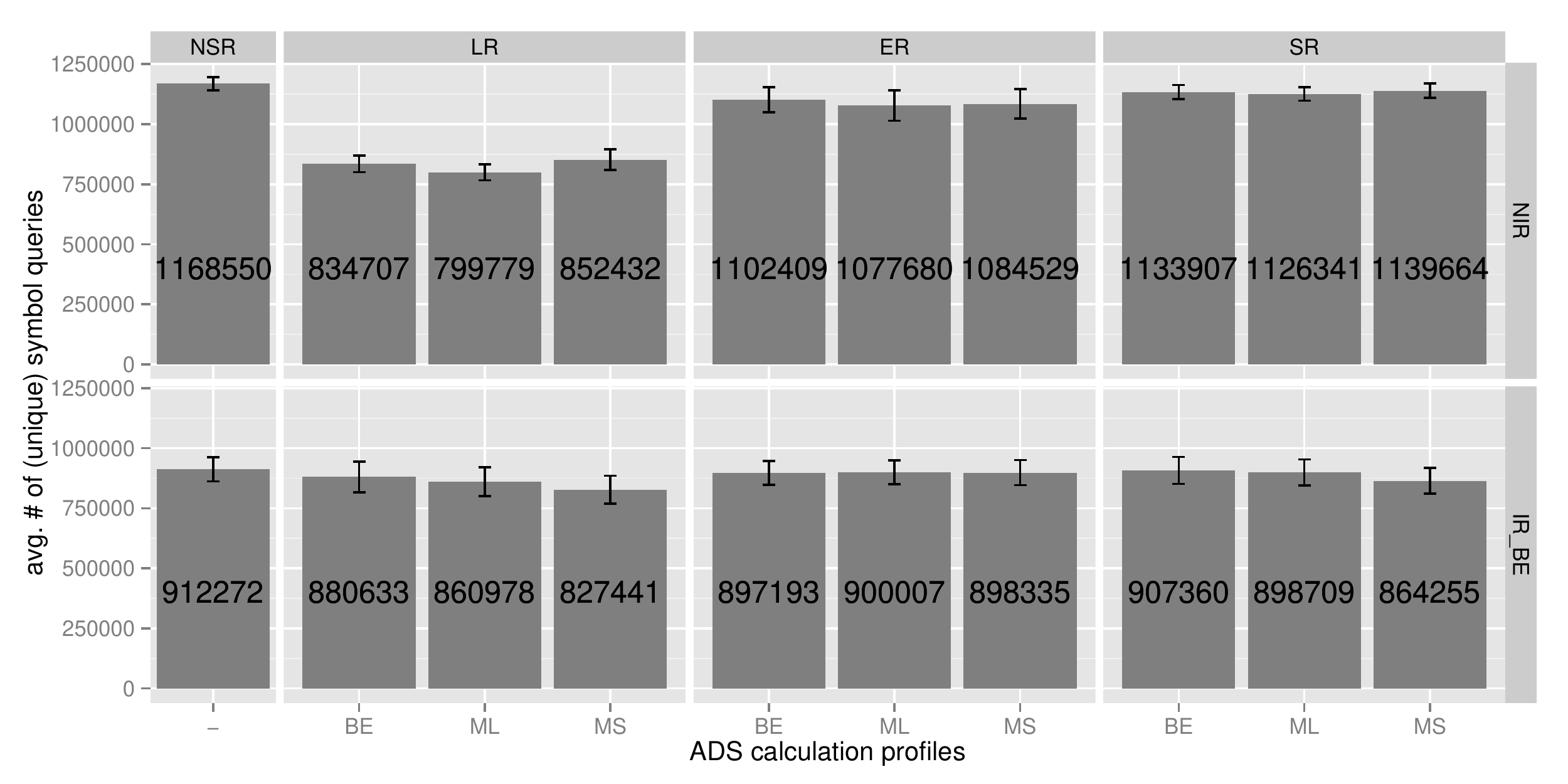}
	}
	\caption{Benchmark: Query performance of selected heuristics (random Mealy, expanded separating word)}
	\label{fig:bench_redundancy_heu}
\end{figure}

While the reset query performance remains somewhat similar to the scenario of regular separating words, the symbol query performance shows interesting new results:
In the previous case, using a subtree replacement heuristic generally increased the amount of (unique) symbol queries posed during the learning process, both with and without the combination of the immediate replacement heuristic.
In case of redundant counterexamples, the exact opposite can be observed.

Every (of the presented) subtree replacement heuristic is able to improve the symbol query performance.
The leveled subtree replacement heuristic, which by observation replaces subtrees the most aggressive way, is thereby the most successful one.
Additionally, the \enquote{minimal size} ADS calculation profile seems to better counteract the redundancy of the counterexamples, while previously -- when exposed to near optimal counterexamples -- decreasing performance.
While again, the sole utilization of the immediate replacement heuristic improves the symbol query performance, this time, additionally utilizing a subtree replacement heuristic further benefits the performance.
For comparing the heuristics with other algorithms, \Cref{fig:bench_redundancy_comp} displays the measured data:

\begin{figure}[ht]
	\centering
	\subcaptionbox{Reset query performance}[\linewidth]{
		\includegraphics[width=\linewidth]{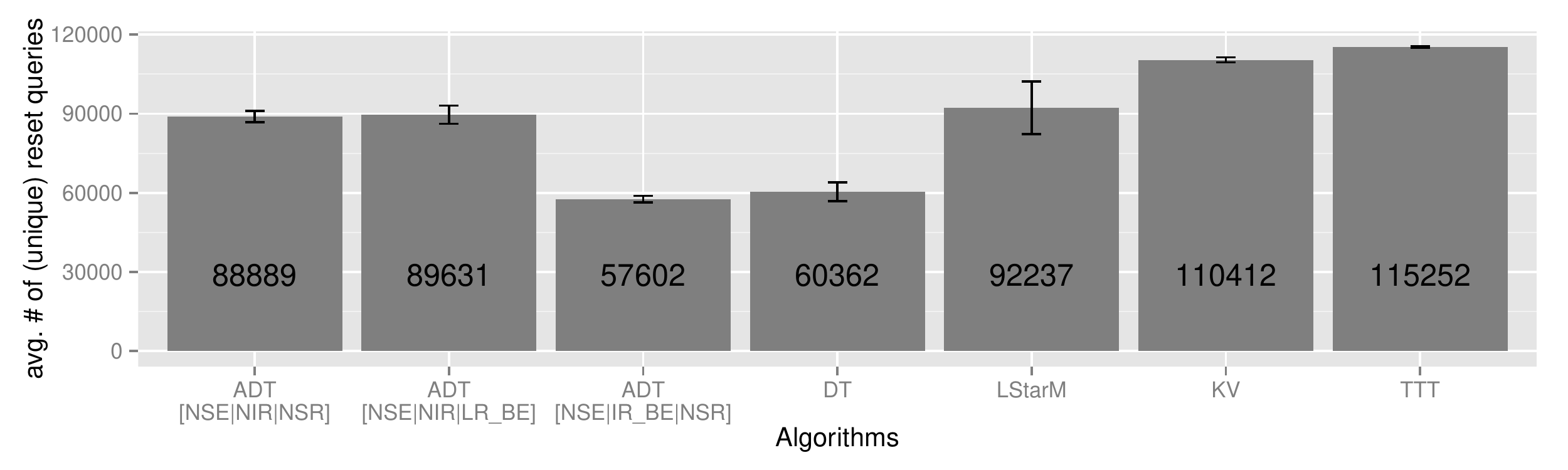}
	}
	\subcaptionbox{Symbol query performance}[\linewidth]{
		\includegraphics[width=\linewidth]{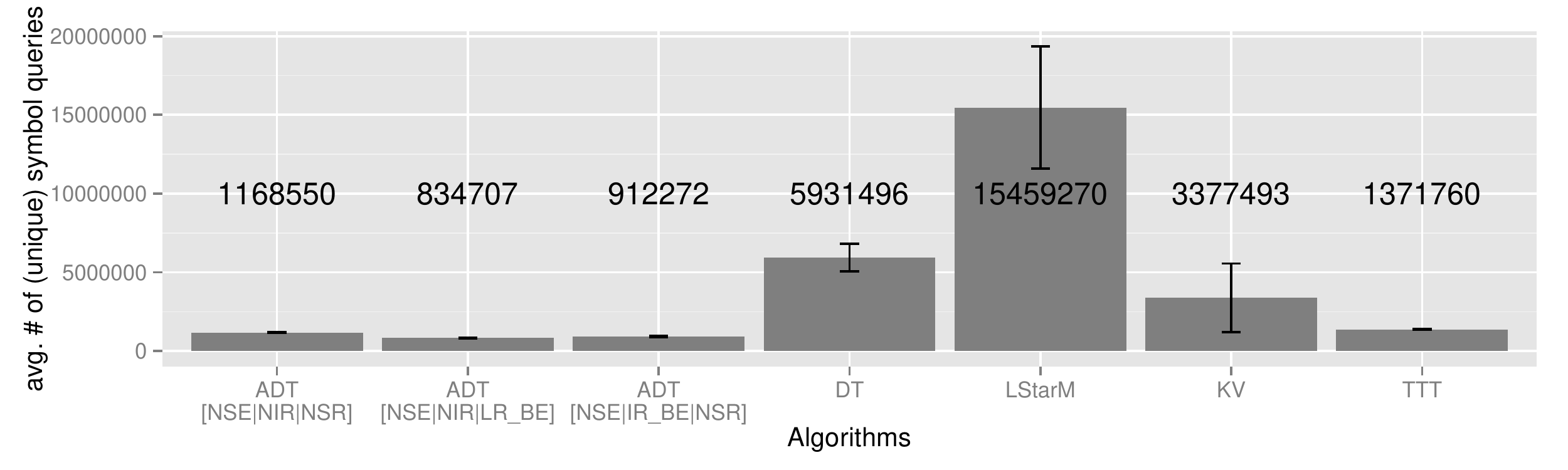}
	}
	\caption{Benchmark: Query performance of selected heuristics compared to other algorithms (random Mealy, expanded separating word)}
	\label{fig:bench_redundancy_comp}
\end{figure}

Regarding the reset performance, the sole application of the immediate replacement heuristic yields the best performance.
Though, in comparison to the results of the first series, the performance of the DT algorithm is much closer to the best heuristic.
This may be due to fact, that long discriminators (see the query performance of the DT algorithm) allow to distinguish between more states.
However, regarding reset- and symbol query performance, the elaborated approaches are still able to pose an improvement to the learning performance.

\subsection{Real-life Benchmarks}

Although synthetic benchmarks easily allow to test a plethora of configurations of target systems, they often lack characteristics of real-world applications.
Especially uniformly sampled random Mealy machines do not follow any specific structure (contrary to real-world applications), which may bias the measured data in a certain way. 
Therefore, to give an additional view on the performance of the algorithms, this section will present benchmark results for two real-world use cases.

\begin{itemize}

\item
For the first use case, the target system is a simulated version of the Online Conference Service (OCS), a web-based conference management service currently developed at the Chair of Programming Systems in Dortmund and used in production by Springer Verlag \cite{ocs}.
Although the simulator (as presented in \cite{windmuel}) adds a certain level of abstraction, such as a discretized input alphabet, it still resembles the core workflow and structure of the original system.
The realistic nature is supported by the fact, that the simulator itself is a piece of executable code for which no formal specification exists.

The simulator provides an interface that allows to input one of $17$ predefined input symbols.
For each input symbol the simulator emits a binary output symbol either indicating success or failure of the input action.
Counterexamples were obtained using an equivalence oracle chain that (in order) consults: a cache consistency oracle; a random word oracle posing $200$ queries of random length $l \in [20, 400]$; and a conformance check using the partial W-Method \cite{Khendek91testselection} with search depth $1$.
A total of $10$ runs were measured.

For the ADTLearner, no configuration for the exhaustive subtree replacement heuristic was tested, since a single run had not finished after $40$ hours.

\item
The target system of the second benchmark is an Engine Status Manager (ESM) \cite{Smeenk_applyingautomata}, a software component used in industrial printers and copiers.
In this benchmark, access to the formal model was available.
The model holds $3410$ states, an input alphabet with $77$ elements and an output alphabet with $151$ elements.

While initially, the usage of an random word equivalence oracle was intended, this approach ran into out-of-memory errors.
As a result, the counterexamples were obtained using separating words and a single run was measured.

No data has been collected for the LStarM algorithm, because even with $32$ gigabyte of memory, the algorithm ran into out-of-memory errors.
Furthermore, regarding the ADTLearner, only the \enquote{best effort} ADS calculation profile and the leveled subtree replacement heuristic were able to terminate in a reasonable ($\leq 60$ hours) amount of time.
\end{itemize}

\paragraph*{OCS}

The first notable observation is already given in the presentation of the benchmark setup:
The computational impact of adaptive distinguishing sequences.
To a slight extend, this effect was already visible for the synthetic benchmarks.
However, the random structure of the target systems seems to improve the process of computing an adaptive distinguishing sequence, because distinguishing behavior can be observed early and often.
In the case of \enquote{structured} target systems, the complexity of computing adaptive distinguishing sequences has a bigger impact on the learning process.

The differences between the individual replacement heuristics are similar to the differences of the synthetic benchmark using expanded separating words, which is why a detailed visualization is omitted.
The sole utilization of the leveled subtree replacement heuristic poses a slight improvement to the reset- and symbol query performance.
However, the immediate replacement heuristic -- again -- allows the biggest improvement with regard to the reset performance.
While the combination of a subtree replacement heuristic and an immediate replacement heuristic worsened the performance reset-wise, it showed an improvement for the query-performance.

A comparison of the competing algorithms is shown in \Cref{fig:bench_ocs_comp}:

\begin{figure}[ht]
	\centering
	\subcaptionbox{Reset query performance}[\linewidth]{
		\includegraphics[width=\linewidth]{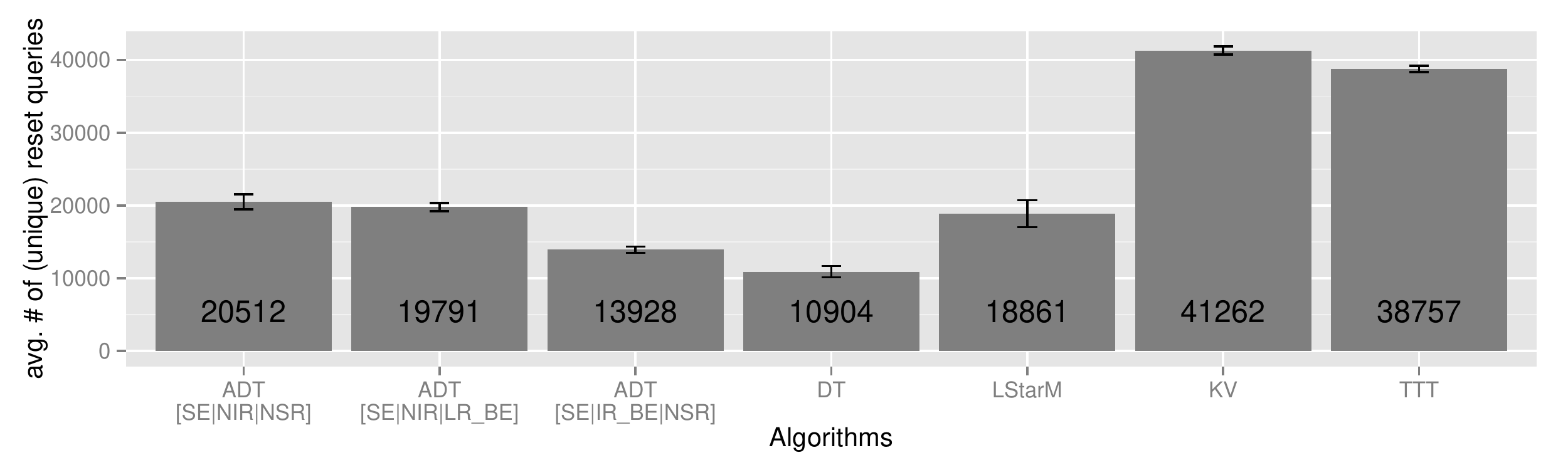}
	}
	\subcaptionbox{Symbol query performance}[\linewidth]{
		\includegraphics[width=\linewidth]{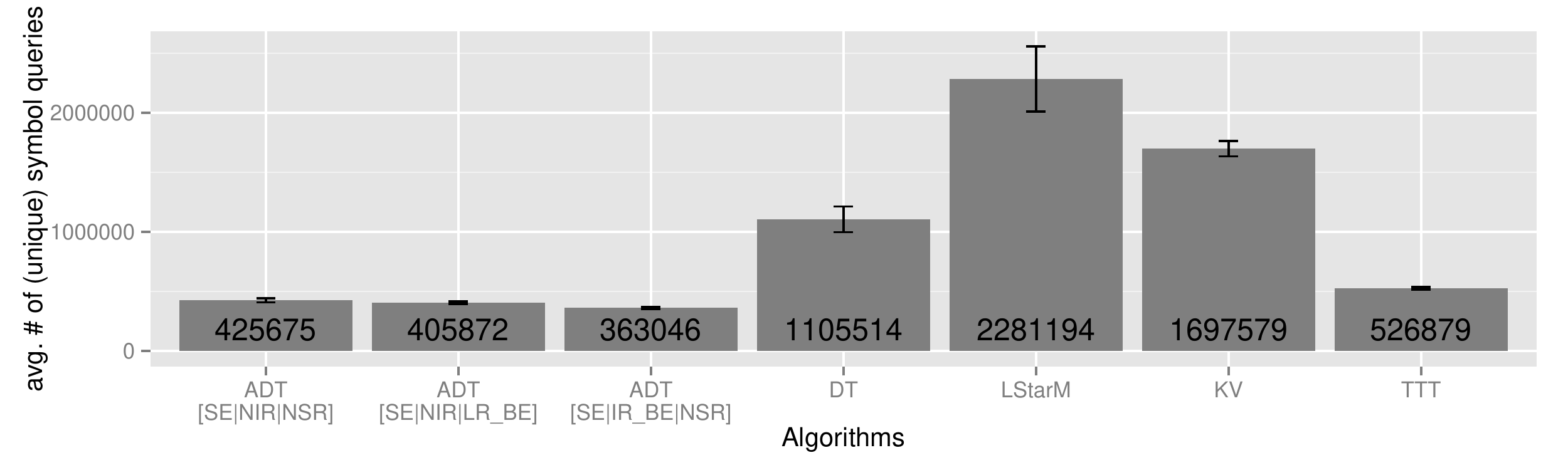}
	}
	\caption{Benchmark: Query performance of selected heuristics compared to other algorithms (OCS)}
	\label{fig:bench_ocs_comp}
\end{figure}

The most notable result is the DT algorithm which showed a better reset-performance than any of the elaborated heuristics.
This may be due to two reasons:

First, the structure of the OCS system does not seem to be suitable for the computation of adaptive distinguishing sequences.
Having only two output symbols, a lot of ADS computations potentially encounter converging states, which does not allow to utilize the full potential of reset-free adaptive distinguishing sequences.
An indication for this is given by the fact, that the ratio between the number of counterexamples (i.e. points in time, where a subtree replacement can be issued) and the number of proposed replacements (and their affected nodes) is lower than e.g. in the random Mealy benchmark (cf. \autoref{cha:appendix}).

Second, the utilization of a query cache may favor the DT algorithm.
The query performance shows that the DT algorithm posed significantly more symbol queries.
Once a long sequence of symbols is queried, every sequence that is a prefix of a previously posed query can be answered by the cache.
However, if only short sequences are queried (as indicated by the query performance of the elaborated approaches), every extension of a cached sequence needs to consult the target system and therefore causes an additional reset.

With respect to the sole reset performance, the ADTLearner was beaten by the DT algorithm.
However, regarding the overall query performance, especially in comparison to the other state-of-the-art algorithms, the elaborated approaches remain competitive.

\paragraph*{ESM}

The increased complexity of the target system outlines the problems of \enquote{ambitioned} heuristics that seek to compute \textit{optimal} replacements.
Even if these heuristics potentially propose optimal replacements, they introduce such high computational (and therefore runtime) costs, that only a very few use cases may profit from it.

Since only the evaluation of the leveled subtree replacement heuristic and the best effort ADS calculation profile was possible, this paragraph directly presents the comparison with other algorithms, which is  shown in \Cref{fig:bench_esm_comp}

\begin{figure}[ht]
	\centering
	\subcaptionbox{Reset query performance}[\linewidth]{
		\includegraphics[width=\linewidth]{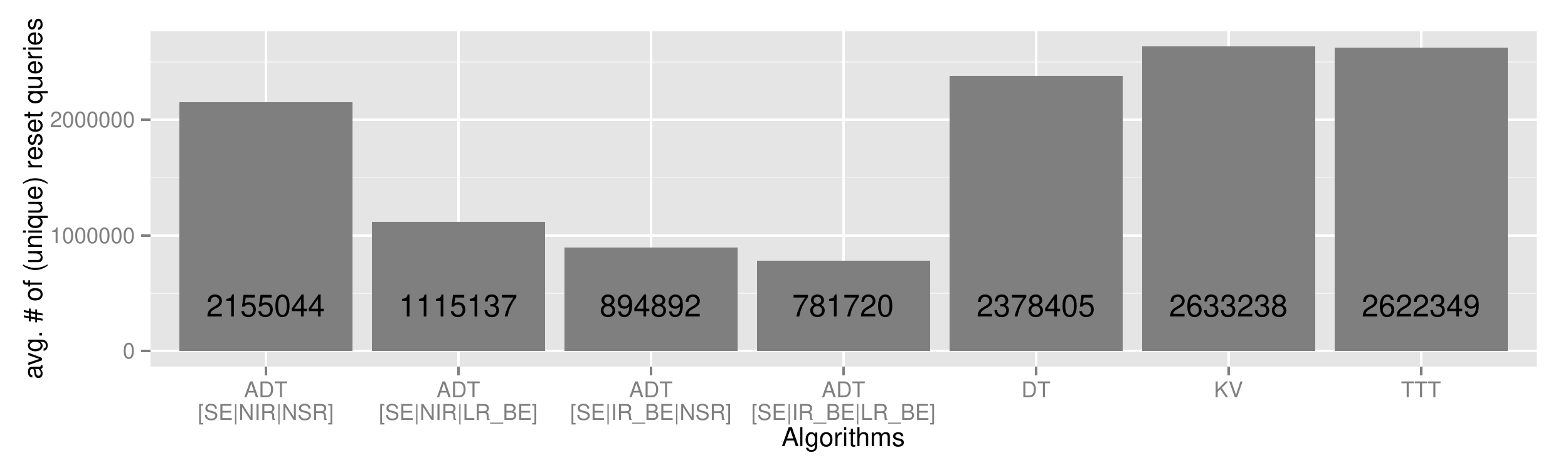}
	}
	\subcaptionbox{Symbol query performance}[\linewidth]{
		\includegraphics[width=\linewidth]{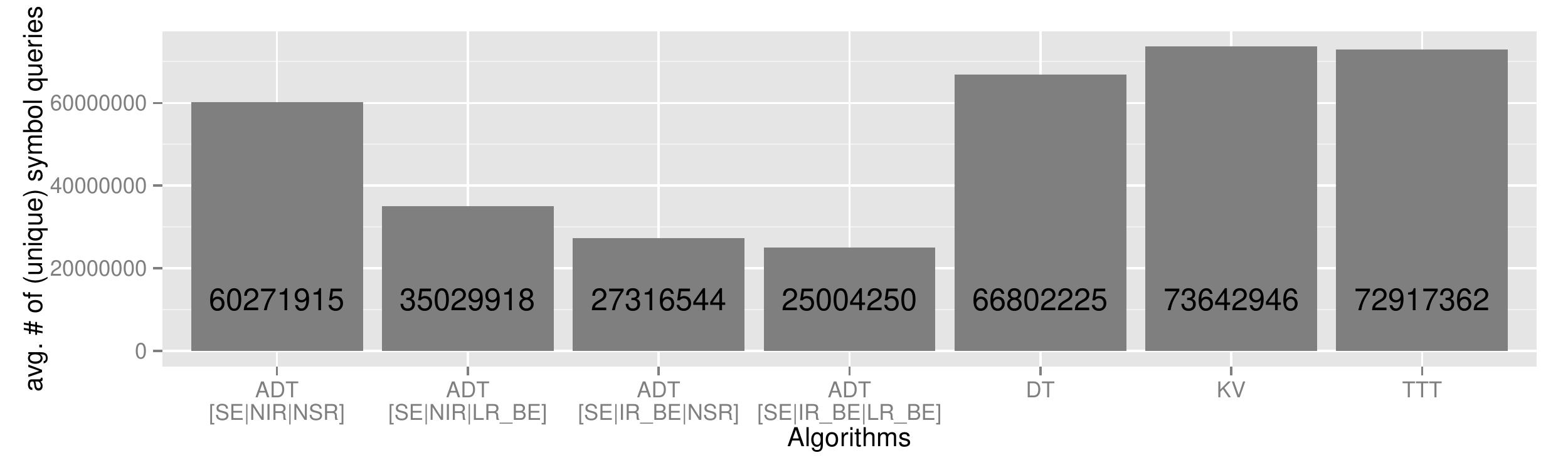}
	}
	\caption{Benchmark: Query performance of selected heuristics compared to other algorithms (ESM)}
	\label{fig:bench_esm_comp}
\end{figure}

As shown by the benchmark results (cf. \autoref{cha:appendix}), the ESM system allowed a much better utilization of adaptive distinguishing sequences.
With respect to the size of the target system, the achieved improvements therefore turn out much more drastic.
While the base algorithm somewhat resembles the performance of the other algorithms, the different heuristics allow a significant improvement regarding both reset and query performance.
Additionally, unlike previous results, the combination of the immediate replacement heuristics and the leveled subtree replacement heuristic performs better than the respective heuristics alone.

However, one has to attribute the fact, that separating words were used as counterexamples.
The overall query performance may vary if exposed to counterexamples containing redundancy, as indicated by the OCS benchmark.

\subsection{Summary}

The empirical analysis has shown that -- albeit the results of the theoretical analysis -- the elaborated approaches often affect the performance of the learning process in a positive way by reducing the amount of (unique) symbol and reset queries.

However, the case-studies further showed that the computational overhead of computing adaptive distinguishing sequences and subtree replacements clearly impacts the learning process.
Promising results were given by the \emph{best effort} ADS calculation profile, the \emph{immediate replacement} heuristic and the \emph{leveled subtree replacement} heuristic.
While in certain situations being the only realistically applicable heuristics, they often yielded the best (query complexity) results.

With regard to the potential improvement of the duration of the active learning process, the data of \autoref{cha:appendix} has to be taken with a grain of salt:
For all benchmarks, the reset mechanism of the target system was a fast operation.
Therefore, the measured duration of the learning algorithm mainly represents the time required for computing subtree replacements and adaptive distinguishing sequences.
The data therefore does not represent potential time savings of algorithms when exposed to a system with an expensive reset mechanism.

As for generality, the potential benefit of the elaborated approaches depends on the structure of the target system.
The OCS use-case showed, that there exist certain situations, in which the different heuristics are not able to outperform existing learning algorithms with regard to the reset performance.
Nevertheless, even in these situations the developed approaches remain competitive.
And, while the structure of random Mealy machines may not necessarily be representative for real-life applications, the ESM case-study showed, that for certain configurations the developed concepts of this thesis allow a significant improvement in query performance.

%% file: sections/future.tex
\chapter{Summary and Future Work}
\label{cha:future}

This chapter concludes the thesis.
It gives a summary about the goals of this thesis, the developed approaches for achieving these goals and the obtained results.
Moreover, it presents an outlook on possible further research that may be based on the work presented in this thesis.

\section{Summary}

The motivation for this thesis was to improve the active automata learning experience for a set of real-life applications with certain specifics:
Many extensions have been proposed to the original active automata learning framework developed by Angluin, to improve its applicability to real-life applications.
However, for guaranteeing on of its core requirements -- the requirement of independent communication (i.e. independent membership queries) -- it is often resorted to a somewhat artificial reset mechanism.
Since many applications not necessarily include a (reliable) reset mechanism in their original design, it may be realized by an expensive external operation (e.g. restarting a simulator), which may drastically reduce the performance of active automata learning and therefore reduce the will to employ active automate learning.

For tackling the above problem and reducing the negative impact of potentially time-consuming resets, this thesis elaborated the integration of adaptive distinguishing sequences -- a well-studied concept from the field of model-based testing -- in the active learning process.
For achieving this goal, this thesis has first presented a fully self-contained learning algorithm that lifts the active learning process to the adaptive environment required for embedding the proposed concepts.
On the foundation of this base algorithm, a generic framework for \enquote{improving} the performance of the learning algorithm by means of subtree replacements was presented.
For utilizing this elaborated framework and actively employing adaptive distinguishing sequences in the active learning process, a set of replacement heuristics and a ADS calculation profiles was presented.

While the theoretical analysis of the developed approaches has shown, that they -- in a worst-case scenario -- may not necessarily reduce the number of executed resets and may even increase the number of executed symbol queries, the empirical evaluation has shown that in many situations they outperform other active learning algorithms with regard to executed reset queries.
Even if not being able to beat other learning algorithms, the developed approaches remain on a competitive level.

\section{Future Work}

However, the possibilities of this field of research do not end with this thesis.
As stated in \Cref{cha:preliminaries}, this thesis focuses on finite, deterministic Mealy machines as the level of abstraction for the target system.
While there exist many success stories, where this level of abstraction yields good formal specifications for (deterministic) reactive input/output systems, one can easily find scenarios where this model only poorly covers the essential behavior of the target system or may not be practicably applicable at all.

With regard to determinism, references towards inferring non-deterministic automata and computing adaptive distinguishing sequences for non-deterministic systems were already given.
A potential question for future research could be, if non-determinism adds any additional side-effects to the learning process or if the concepts presented in this thesis can be directly carried over and similar improvements are achievable.

As for the general suitability of finite Mealy machines abstracting the target system's behavior, the development of more complex automata kinds -- such as register automata \cite{registerautomata,Cassel2016} -- was mentioned.
There are no algorithms known to the author, that directly target the state identification problem for register Mealy machines by means of an adaptive distinguishing sequence, which raises the question, if the structural extension of register Mealy machines can be exploited to compute \enquote{better} distinguishing sequences. 
Similar to the scenario above, a comparison of the potential benefits is of interest.

The exploration of further improvements is however not limited to the tentative hypothesis.
The usage of adaptive distinguishing sequences was initially motivated by reducing the total amount of executed system resets as these operations were considered most expensive.
Real-world applications may however exhibit a far more complex runtime behavior:
Certain input stimuli may result in irregular high costs if the target system is in a specific state.

The presented approach may be extended to computing minimal-cost adaptive distinguishing sequences.
In certain instances it may even be beneficial to compute multiple partial adaptive distinguishing sequences, i.e. minimal-cost adaptive distinguishing trees, if it helps to omit very expensive input symbols.
While it is generally not possible to compute a true optimal adaptive distinguishing sequence (-tree), since the true states of the target system are not known until termination of the learning process, the costs may be approximated by averaging the costs over all states, i.e. compute the average costs for each input symbol.
The effective runtime may be improved in certain scenarios.

Besides semantical additions, the effective performance of the learning algorithm may also be subject to further research.
As stated several times, the active learning community usually compares algorithms the on basis of their query performance.
While several benchmarks have shown that the developed approaches pose an improvement to the situation, they have also shown that the computational extra work introduced to the learning process impacts its overall runtime.
While this thesis has not directly targeted this issue, the problem is generally of good nature.

Most of the expensive computations (e.g. the computation of replacements) take place in an \emph{offline} scenario, meaning they only require access to local data such as the tentative hypothesis and do not interact with the target system.
This means the critical code paths are not directly related to the active learning process and therefore generic techniques for improving runtime -- such as parallelization -- can be applied:
The computation of an adaptive distinguishing sequence essentially reduces to a search problem.
Since the data of the tentative hypothesis is only accessed in a read-only manner, it can be easily distributed to multiple threads or even clusters and the search can run in parallel.

Even the single-threaded case may allow runtime improvements.
In its submitted state, the computation of adaptive distinguishing sequences, except for the case where the algorithm of Lee and Yannakakis is applicable, is based on constructing state-splitting input sequences by traversing the successor tree.
There may exist more efficient data structures and algorithmic approaches to compute similar results.

Eventually, the positive results presented in this thesis may motivate further investigations on this field of research.

%% file: sections/appendix.tex
\begin{appendices}

\chapter{Benchmark Results}
\label{cha:appendix}

The parameterizations of the ADTLearner are described by the following scheme:\\\textbf{ADT[a|b\_x|c\_x]}, where

\begin{itemize}

\item \textbf{a $\in$ \{NSE, SE\}} with

\begin{description}

\item[NSE] no subtree extension heuristic and
\item[SE] applied subtree extension heuristic.

\end{description}

\item \textbf{b $\in$ \{NIR, IR\}} with

\begin{description}

\item[NIR] no immediate replacement heuristic and
\item[IR] applied immediate replacement heuristic.

\end{description}

\item \textbf{c $\in$ \{NSR, LR, ER, SR\}} with

\begin{description}

\item[NSR] no subtree replacement heuristic, 
\item[LR] applied leveled replacement heuristic,
\item[ER] applied exhaustive replacement heuristic and
\item[SR] applied single replacement heuristic.

\end{description}

\item and (if applicable)\footnote{Deactivated heuristics do not compute adaptive distinguishing sequences.} \textbf{x $\in$ \{BE, ML, MS\}} with

\begin{description}

\item[BE] best effort: use LY-algorithm or leveled BFS-search to compute ADSs,
\item[ML] minimum length: use successor-tree to compute minimum-length ADSs and
\item[MS] minimum size: use successor-tree to compute minimum-size ADSs.

\end{description}

\end{itemize}

\noindent
The measured values are abbreviated as follows:

\begin{description}

\item[R] Total amount of (unique)\footnote{All algorithms used a symbol query/membership oracle backed by a tree cache.\label{fn:cache}} reset queries posed by the learner.
\item[SQ] Total amount of (unique)\footref{fn:cache} symbol queries posed by the learner.
\item[CE] Total amount of posed equivalence queries.
\item[ADT\_RN] Total amount of reset nodes in the final ADT.
\item[ADT\_RR] Averaged number of reset nodes encountered by a leaf of the final ADT (effective reset costs).
\item[ADT\_PR] Total amount of proposed ADT-subtree replacements.
\item[ADT\_PRAN] Total amount of hypothesis states referenced in proposed replacements.
\item[ADT\_PRS] Total amount of symbol nodes in the proposed replacements.
\item[ADT\_ARS] Total amount of symbol nodes in the accepted replacements.
\item[ADT\_ARR] Total amount of reset nodes in the accepted replacements.
\item[ADT\_ARP] Total amount of perfect (i.e. reset-free) accepted replacements.
\item[ADT\_ARA] Total amount of accepted replacements.
\item[OT\_E] Total amount of successful findings of an extending discriminator in the observation tree.
\item[OT\_S] Total amount of successful finding of a shorter discriminator than provided by the current counterexample.
\item[SIZ] The size of the final hypothesis.
\item[DUR] The time (in milliseconds) taken by the learning algorithm (this excludes the time taken for searching counterexamples).
\end{description}

\newcommand{\plotResult}[3]{
	\begin{longtabu}{l*{#2}{r}}
		\caption{#3}\\
		\toprule
		\rowfont[c]{\bfseries}
		\csvloop{
			file=#1,
			no head,
			no table,
			late after line=\\,
			late after first line=\\\midrule\endhead\bottomrule\endfoot,
			column count=#2,
			command={%
				\csviffirstrow{%
					\csvlinetotablerow\\%
					\midrule\endfirsthead%
					\toprule%
					\rowfont[c]{\bfseries} \csvlinetotablerow%
				}{%
					\csvlinetotablerow
				}%
			},
		}
	\end{longtabu}
}

\begin{landscape}
	\tiny
	\plotResult{data/random.csv}{16}{Averaged results for the random Mealy benchmark with separating words}
\end{landscape}

\clearpage

\begin{landscape}
	\tiny
	\plotResult{data/redundancy.csv}{16}{Averaged results for the random Mealy benchmark with expanded separating words}
\end{landscape}

\begin{landscape}
	\tiny
	\plotResult{data/ocs.csv}{17}{Averaged results for the OCS benchmark}
\end{landscape}

\begin{landscape}
	\tiny
	\plotResult{data/esm.csv}{16}{Results for the ESM benchmark}
\end{landscape}

\end{appendices}